\documentclass[11.5pt]{article}
\usepackage{appendix}

\usepackage{jmlr2e}
\usepackage{lastpage}
\jmlrheading{23}{2022}{1-\pageref{LastPage}}{6/21; Revised
7/22}{7/22}{21-0716}{Anastasis Kratsios and L\'{e}onie Papon}
\ShortHeadings{title}{Kratsios and Papon}

\usepackage{mathtools}
\usepackage{bbm}
\usepackage{booktabs}
\usepackage{subcaption}
\usepackage{xparse}
\usepackage{sidecap}
\usepackage{float}
\usepackage{enumerate,mdwlist}
\usepackage{epsfig}
\usepackage[bbgreekl]{mathbbol}
\usepackage{amsfonts,amsmath,enumerate}
\usepackage{amscd}
\usepackage{amsmath,amsfonts,mathrsfs}\usepackage{bm}
\usepackage{latexsym}
\usepackage{mleftright}
\usepackage[dvipsnames]{color}
\usepackage{xifthen}
\usepackage{makeidx}
\usepackage{amsmath}
\usepackage{amsfonts}
\usepackage{amssymb}
\usepackage{mathrsfs}
\usepackage{tikz,tikz-cd}
\usepackage{tkz-euclide}
\usetikzlibrary{matrix}
\usepackage{smartdiagram}

\usepackage[linesnumbered,ruled,vlined]{algorithm2e}
\SetKwComment{Comment}{$\triangleright$\ }{}
\usepackage{etoolbox}
\usepackage[export]{adjustbox}

\newcounter{algoline}

\AtBeginEnvironment{algorithm}{\setcounter{algoline}{0}}

\usepackage{hyperref} 
\hypersetup{
	colorlinks = true,
	linkcolor = blue,
	anchorcolor = blue,
	citecolor = blue,
	filecolor = blue,
	urlcolor = blue
}

\newcommand{\rr}{{\mathbb{R}}}

\newcommand{\rrflex}[1]{{\ensuremath{\rr^{#1}
}}}

\newcommand{\rrd}{{\rrflex{d}}}
\newcommand{\rrn}{{\rrflex{n}}}
\newcommand{\rrm}{{\rrflex{m}}}
\newcommand{\rrp}{{\rrflex{p}}}

\newcommand{\xxx}{\mathcal{X}}
\newcommand{\yyy}{\mathcal{Y}}

\newcommand{\nn}{{\mathbb{N}}}
\newcommand{\zz}{{\mathbb{Z}}}

\NewDocumentCommand\im{m}{\operatorname{Im}\left({#1}\right)}

\newcommand{\zzz}{{\mathcal{F}}}

\newcommand{\kkk}{{\mathscr{K}}}

\newcommand{\fff}{{\mathscr{F}}}

\renewcommand{\ggg}{{\mathscr{G}}}
\newcommand{\XX}{\mathbb{X}}

\NewDocumentCommand{\NN}{oo}{
    \ensuremath{
        \mathcal{NN}
        \IfValueT{#1}{_{#1}}\IfValueF{#1}{_{p,m,k}}
        \IfValueT{#2}{^{#2}}\IfValueF{#2}{^{\sigma}}
    }
}

\usepackage{wasysym}

\definecolor{darkcerulean}{rgb}{0.03, 0.27, 0.49}
\definecolor{darkmidnightblue}{rgb}{0.0, 0.2, 0.4}
\definecolor{darkcyan}{rgb}{0.0, 0.55, 0.55}
\definecolor{darkgreen}{rgb}{0.0, 0.2, 0.13}
\definecolor{deepjunglegreen}{rgb}{0.0, 0.29, 0.29}
\definecolor{darkcandyapplered}{rgb}{0.64, 0.0, 0.0}
\definecolor{darkred}{rgb}{0.55, 0.0, 0.0}
\definecolor{darkscarlet}{rgb}{0.34, 0.01, 0.1}
\definecolor{jasper}{rgb}{0.84, 0.23, 0.24}
    \definecolor{deepcarrotorange}{rgb}{0.91, 0.41, 0.17}
\definecolor{darkpurple}{rgb}{0.58, 0.0, 0.83}
\definecolor{electricpurple}{rgb}{0.75, 0.0, 1.0}
\definecolor{darkjazzberryjam}{rgb}{0.45, 0.04, 0.37}
\definecolor{MidnightBlue}{RGB}{25,25,112}
\definecolor{MidnightBlueComplementingGreen}{RGB}{25,112,25}
\definecolor{MidnightBlueComplementingPurple}{RGB}{112,25,112}
\definecolor{MidnightBlueComplementingRed}{RGB}{112,25,69}

\ShortHeadings{Universal Approximation Theorems for Differentiable Geometric Deep Learning}{A. Kratsios and L. Papon}

\newtheorem{assumption}{Assumption}

\NewDocumentCommand{\Anastasis}{mo}{
    \IfValueF{#2}{
                        {{
                            \textcolor{violet}{ 
                            \textbf{A:}
                            \textit{{#1}}
                            }
                        }}
        }
    \IfValueT{#2}{
                        \marginnote{{\scriptsize
                            \textcolor{violet}{ 
                            \textbf{A:}
                            \textit{{#1}}
                            }
                        }}
        }
                    }
\NewDocumentCommand{\Leonie}{mo}{
    \IfValueF{#2}{
                        {{
                            \textcolor{jade}{ 
                            \textbf{L:}
                            \textit{{#1}}
                            }
                        }}
        }
    \IfValueT{#2}{
                        \marginnote{{\scriptsize
                            \textcolor{jade}{ 
                            \textbf{L:}
                            \textit{{#1}}
                            }
                        }}
        }
                    }

\begin{document}

\title{Universal Approximation Theorems\\ for Differentiable Geometric Deep Learning}
	
	\author{\name Anastasis Kratsios \email 
	kratsioa@mcmaster.ca
	\\
		\addr Department of Mathematics\\
		McMaster University%
        \\
        1280 Main Street West, Hamilton, Ontario, L8S 4K1, Canada
		\AND
		\name L\'{e}onie Papon 
		\email 
        leonie.b.papon@durham.ac.uk
        \\
		\addr Department of Mathematical Sciences\\
		Durham University
		\\
		Upper Mountjoy Campus, Stockton Rd, Durham DH1 3LE, United Kingdom
		}
	
	\editor{Sayan Mukherjee}

\maketitle

\begin{abstract}
This paper addresses the growing need to process non-Euclidean data, by introducing a geometric deep learning (GDL) framework for building universal feedforward-type models compatible with differentiable manifold geometries.  We show that our GDL models can approximate any continuous target function uniformly on compact sets of a controlled maximum diameter.  We obtain curvature dependant lower-bounds on this maximum diameter and upper-bounds on the depth of our approximating GDL models.   Conversely, we find that there is always a continuous function between any two non-degenerate compact manifolds that any ``locally-defined" GDL model cannot uniformly approximate.   Our last main result identifies data-dependent conditions guaranteeing that the GDL model implementing our approximation breaks ``the curse of dimensionality."   We find that any ``real-world" (i.e. finite) dataset always satisfies our condition and, conversely, any dataset satisfies our requirement if the target function is smooth.  As applications, we confirm the universal approximation capabilities of the following GDL models: \cite{ganea2018hyperbolic}'s hyperbolic feedforward networks, the architecture implementing \cite{krishnan2015deepDeepKalman}'s deep Kalman-Filter, and deep softmax classifiers.  We build universal extensions/variants of: the SPD-matrix regressor of \cite{meyer2011regression}, and \cite{fletcher2003statistics}'s Procrustean regressor.  In the Euclidean setting, our results imply a quantitative version of \cite{kidger2019universal}'s approximation theorem and a data-dependent version of \cite{yarotsky2019phase}'s uncursed approximation rates.  
\end{abstract}
\hfill\\
\section{Introduction}\label{s_Intro}
Since their introduction in \cite{mcculloch1943logical}, the approximation capabilities of neural networks and their superior efficiency over many classical modelling approaches have led them to permeate many applied sciences, many areas of engineering, computer science, and applied mathematics.  Nevertheless, the complex geometric relationships and interactions between data from many of these areas are best handled with machine learning models designed for processing and predicting from/to such non-Euclidean structures.  

This need is reflected by the emerging machine learning area known as \textit{geometric deep learning}.  
A brief (non-exhaustive) list of situations where geometric (deep) learning is a key tool includes: shape analysis in neuroimaging (e.g. \cite{Fletcher2013} and \cite{shapereco}), human motion patterns learning (e.g. \cite{humanmotion1}, \cite{humanmotion2}, and \cite{humanmotion3}), molecular fingerprint learning in biochemistry (e.g. \cite{biochemistry}), predicting covariance matrices (e.g. \cite{BonnabelRegression}), robust matrix factorization (e.g. \cite{baes2019lowrank} and \cite{herrera2020denise}), learning directions of motion for robotics using spherical data as in (e.g. \cite{dai2018principal}, \cite{pmlr-v38-straub15}, and \cite{pmlr-v119-dutordoir20a}), and many other learning problems.

There are many structures of interest in geometric deep learning such as differentiable manifolds (overviewed in \cite{bronstein2017geometric} and in \cite{bronstein2021geometric} as well as in our applications Sections~\ref{ss_applications_Pt1} and~\ref{ss_Applications_II_sophisiticated_models}), graphs (\cite{zhou2020graph}), deep neural networks with group invariances or equivariances (respectively, see \cite{Yarotsky2021InvariantNets} and \cite{cohen2016group}).  This paper concentrates on the first of these cases and develops a general theory compatible with any differentiable manifold input and output space.  To differentiate the manifold-valued setting from the other geometric deep learning problems just described, we fittingly refer to it as \textit{``differentiable geometric deep learning"}.  

\subsection*{Contributions}\label{s_Intro_Contributions}
This paper adds to this rapidly growing research area by developing a self-contained geometric deep learning framework for building universal deep neural models between differentiable manifolds.  The models in our framework are all universal and are constructed explicitly from simpler building blocks.  Additionally, each of our universal (resp. efficient) approximation theorems is quantitative.  

After introducing our framework and presenting our main results, we demonstrate the scope and flexibility of our proposed framework by using it to validate the approximation capabilities of various commonly used geometric deep learning models.  These include: \textit{the Hyperbolic Feedforward Networks} of \cite{ganea2018hyperbolic} for learning from efficient embeddings of large undirected graphs (see \cite{munzner1997h3}), trees (see \cite{SalaHyperbolicTradeoffs}), complex social-networks (see \cite{krioukov2010hyperbolic}), and hierarchical datasets (see \cite{NIPS2017_7213}), the architecture implemented in the\textit{ Deep Kalman filter} of \cite{krishnan2015deepDeepKalman} for approximating update rules between spaces of non-degenerate Gaussian measures, and deep softmax classifiers which are omnipresent in contemporary multiclass classification.  
We then show how our framework yields universal extensions to the following popular non-Euclidean regression models: the Procrustean pre-shape space regression models of \cite{Fletcher2013} used in biomedical imaging, and computer-vision and its analogue in the projective shape space of \cite{RealProjectiveShapeSpaces2005AnnalsOfStats}, and the symmetric positive-definite matrix regressor of \cite{meyer2011regression} (see \cite{pennec2006riemannian} for uses of this geometry in tensor computing).  Illustrations of our results in the context of spherical and toral input and output spaces are also provided; we note that spherical data is prevalent in astronomy applications (see \cite{SphereicalDatabook}), and toral geometries have found recent applications in data visualization (e.g. \cite{FlatTorusCitation1} and \cite{FlatTorusCitation2}).  

In the Euclidean context, our results imply a \textit{quantitative version} of the qualitative universal approximation theorem for deep and narrow feedforward networks (recently obtained in \cite{kidger2019universal}) as well as an extension of the dimension-free approximation rates of \cite{yarotsky2019phase} for non-smooth functions defined on ``efficient datasets" (introduced in Section~\ref{sss_efficient_datasets}).  We also find that all datasets are efficient when the target function is smooth and, conversely, any ``real-world" dataset (i.e., non-empty and finite) is efficient for every target function.  In particular, this last result offers a new ``datacentric" perspective explaining the well-documented effectiveness of deep learning.

\subsection*{Our Approach}\label{s_Intro_ss_GDL}
Often, geometric machine learning models work by first linearizing the non-Euclidean data in the input space via a {\color{darkcerulean}{\textit{feature map}} $\phi:\xxx\rightarrow \rrp$}, then processing the linearized data via a classical ``Euclidean" learning model {\color{darkjazzberryjam}{$g$}}, and finally using an inverted linearization step to recover non-Euclidean predictions in the output space $\yyy$ via some {\color{deepjunglegreen}{\textit{``readout map"} $\rho:\rrm\rightarrow \yyy$}}.  
\begin{figure}[!ht]
    \centering
    \begin{tikzcd}
    \xxx \arrow[blue]{r}{f} \arrow[darkcerulean]{d}{\phi} & \yyy\\
    \rrp \arrow[darkjazzberryjam]{r}{g} & \rrm\arrow[deepjunglegreen]{u}{\rho}\\
    \end{tikzcd}
    \caption{Lifting Euclidean learning models to non-Euclidean input/output spaces.}
    \label{fig_Non_Euc_tikz}
\end{figure}
This schema, illustrated in Figure~\ref{fig_Non_Euc_tikz}, has been successfully employed (either explicitly or implicitly) in various areas of machine learning; examples include the principal geodesic analysis method of \cite{fletcher2011geodesic}, the Log-Euclidean Kernel regressors of \cite{Li2013ICCVKernelExpLog}, the unscented Kalman filters of \cite{Hauberg2013KalmanUnscentedManifold}, the deep Kalman filter of \cite{krishnan2015deepDeepKalman}, the hyperbolic neural network models of \cite{ganea2018hyperbolic} and of \cite{shimizu2021hyperbolic}, the architecture implemented in the feature-map learning meta-algorithm of \cite{kratsios2021neu}, and others.  

More generally, in \cite{kratsios2020non}, it was shown that when {\color{darkjazzberryjam}{$g$}} in Figure~\ref{fig_Non_Euc_tikz} is allowed to be any universal learning model class from $\rrp$ to $\rrm$ then, under certain conditions on {\color{darkcerulean}{ $\phi:\xxx\rightarrow \rrp$}} and on {\color{deepjunglegreen}{$\rho:\rrm\rightarrow \yyy$}}, any continuous function {\color{blue}{$f:\xxx\rightarrow \yyy$}} could be approximated uniformly on compact subsets of $\xxx$.
However, these conditions on {\color{deepjunglegreen}{$\rho:\rrm\rightarrow \yyy$}} require the ``global geometry of $\yyy$" to be ``approximately Euclidean".  Unfortunately, this leaves many interesting input/output spaces arising naturally in geometric deep learning applications out of the scope of this type of approach.  

We take this observation as our starting point. This paper offers a complete solution to the problem of (explicitly) developing universal deep neural models between arbitrary differentiable manifolds $\xxx$ and $\yyy$.   
\subsection*{Our Differentiable Geometric Deep-Learning Framework}\label{ss_section}
Our analysis begins with the observation that universal approximation between differentiable manifolds is \textit{necessarily a local problem} (in Theorems~\ref{thrm_homotopic_necessary_condition} and~\ref{thrm_negative_motiation}).  Thus, we reinterpret Figure~\ref{fig_Non_Euc_tikz} as only holding ``locally", instead of ``globally".  This is achieved by letting the maps {\color{darkcerulean}{$\phi$}} and {\color{deepjunglegreen}{$\rho$}} vary and be fully-specified by the local geometries of $\xxx$ and $\yyy$, respectively.  Our main geometric deep learning model, the \textit{geometric deep networks} (GDNs), are succinctly summarized by Figure~\ref{fig_visualization_GDNs}.  
\begin{figure}[H]
    \centering
    \includegraphics[width=15em]{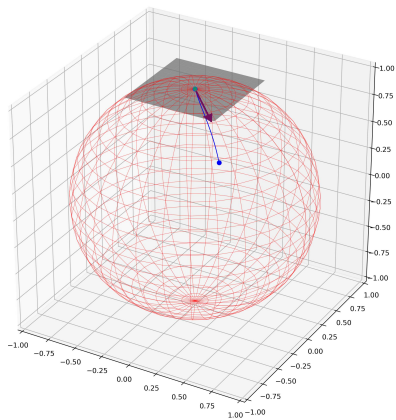}
    \caption{Visualization representation of GDNs}
    \label{fig_visualization_GDNs}
\end{figure}

Our GDN models are build from three distinct types of layers.  Fix a {\color{darkjazzberryjam}{feedforward network $g$}} defined in the {\color{gray}{tangent space}} of a fixed {\color{darkgreen}{reference point $x$}} on $\xxx$.  First, the local non-Euclidean data near {\color{darkgreen}{$x$}} is fed into this ``tangential" {\color{darkjazzberryjam}{feedforward network $g$}} by a local-linearization procedure which sends points near {\color{darkgreen}{$x$}} to the {\color{blue}{velocity vector}} of an optimally efficient curve emanating from {\color{darkgreen}{$x$}} and terminating at the data-point.  Next, this velocity vector is processed by the {\color{darkjazzberryjam}{feedforward network $g$}}.  Lastly, the output of {\color{darkjazzberryjam}{the network $g$}} is mapped onto the output space, about some {\color{darkgreen}{reference point $y$}} $\in\yyy$, by an analogous but inverted linearization layer.  

\begin{remark}\label{remark_Riemannian_Metric}
Our description of the GDN model, summarized by Figure~\ref{fig_visualization_GDNs}, requires fixing a complete Riemannian metric.  These tools will be overviewed shortly.  
\end{remark}

The GDN framework is flexible and general enough to approximate any function between differentiable manifolds.  Nevertheless, at times, it may be more convenient/natural to use GDNs as an integral component of geometric deep learning models with more sophisticated computational graphs.  In the second part of this paper, we build a general class of such models whose additional layers implement fundamental geometric operations such as quotients (useful for easily encoding symmetries and invariances), products (useful for parallelizing GDNs defined on potentially different input and output spaces), and feature as well as readout maps which can be used to combine these models as well as process more general ``non-smooth/singular" geometries.  Examples of the latter are manifolds with boundaries/corners, e.g. the standard simplex.

\subsection*{Organization of Paper}
This paper is organized as follows.  Section~\ref{s_Background} overviews the relevant topological and Riemannian geometric background required for the presentation of our main results.  
Section~\ref{s_main_results} contains our main results surrounding the approximation capabilities of GDNs.  These begin with a necessary condition for universal approximation between complete Riemannian manifolds and the impossibility of models defined from local data to approximate arbitrary continuous functions between compact Riemannian manifolds of positive dimension uniformly on \textit{arbitrary compact subsets of the input space}.  This motivates us to look for a ``controlled universal approximation theorem" which focuses on approximation of continuous functions between Riemannian manifold uniformly on compact subsets of a given maximum diameter.  

Next, we confirm that this is an appropriate notion of universal approximation for geometric deep learning by deriving our ``controlled universal approximation theorem" for GDNs, which is \textit{quantitative} on two fronts: it provides a (non-trivial) curvature-dependant lower-bound on the maximum diameter of a compact subset of $\xxx$ on which a function of a given regularity can be uniformly approximated, and second, it provides detailed depth-order estimates of the deep and narrow GDN implementing the approximation.  Next, we obtain dimension-free approximation rates by GDNs on efficient datasets (introduced in Section~\ref{sss_efficient_datasets}).  We show that any ``real-world dataset" (i.e., non-empty and finite subset of $\xxx$) is efficient for any function (not necessarily continuous) and, conversely, we find that any dataset is efficient for every smooth function with Lipschitz higher-order partial derivatives.
In Section~\ref{ss_applications_Pt1}, quantitative universal approximation theorems for many of the above geometric deep learning architectures are derived.

Section~\ref{s_Results_pt2} develops the approximation capabilities, and the calculus for the geometric deep neural models build from the GDNs using the ``geometric processing layers" described above.  Section~\ref{ss_Applications_II_sophisiticated_models} derives the universal approximation theorems for the remaining geometric deep learning architectures described in the introduction (as well as some others for illustrative purposes).  

\subsection*{Notation and Standing Assumptions}
We denote by $\mathcal{NN}_{p,m,k}^{\sigma}$ the class of functions described by feedforward neural networks with $p$ neurons in the input layer, $m$ neurons in the output layer, and an arbitrary number of hidden layers, each with at-most $k$ neurons and with activation function $\sigma$.  Thus, $f \in \NN[p,m:k]$ if there is a $J\in \nn_+$ and there are composable affine maps $W_{J+1},\dots,W_1$ such that:
$$
f = W_{J+1}\circ \sigma \bullet \dots \circ \sigma \bullet W_1 ,
$$
where, $\bullet$ denotes component-wise composition.  We use $\NN[p,m]$ to denote the set of DNNs of arbitrary width and depth.  
We always assume all activation functions to satisfy the following.
\begin{assumption}[{\cite{kidger2019universal} Condition}]\label{ass_Kidger_Lyons_Condition}
The activation function $\sigma\in C(\rr)$ is non-affine and there is a $x\in \rr$ at which $\sigma$ is differentiable.  Moreover, $\sigma'(x)\neq 0$.  
\end{assumption}

\vspace{-1.5em}
\section{Background}\label{s_Background}
This section contains the metric-theoretic, topological, and Riemannian geometric background for the formulation of our results.  Further background required only for proofs is relegated to the appendix.  
\subsection{Uniform Continuity}\label{ss_Uniform_Continuity}
Since our approximation results are quantitative, focus on functions whose ``metric distortion" is quantifiable; meaning that $f:\xxx\rightarrow \yyy$ is continuous and its \textit{optimal modulus of continuity}:
$$
\omega(f,t)\triangleq \sup_{
\underset{
    x_1,x_2\in \xxx
    }{
        d_{\xxx}(x_1,x_2)\leq t
    }
}\,
d_{\yyy}(f(x_1),f(x_2))
,
$$ 
is finite for all $t \in [0,\infty)$.  Such $f$ are called \textit{uniformly continuous} and the set of all uniformly continuous $f$ mapping $\xxx$ to $\yyy$ is denoted by $C(\xxx,\yyy)$.  
{\color{black}{Many of our estimates require the inverse of $f$'s optimal modulus of continuity.  However, even if $\omega$ is monotone increasing it need not be continuous and therefore in order to invert it we appeal to the \textit{generalized inverse}, in the sense of \cite{EmbrechtsHofert}.  This generalize inverse of $\omega(f,\cdot)$ is defined for $\epsilon>0$ as follows:
\vspace{-.5em}
\begin{equation*}
    \omega^{-1}(f,\epsilon) := \sup \{ t: \omega(f,t) \leq \epsilon \}.
\end{equation*}
}}

The notion of convergence on $C(\xxx,\yyy)$ is {\color{black}{that of \textit{uniform convergence on compact sets }}}
inherited from the larger space of \textit{continuous functions}
from $\xxx$ to $\yyy$, denoted by $\bar{C}(\xxx,\yyy)$%
\footnote{
Many approximation results (e.g. \cite{Yarotski} or \cite{kidger2019universal}) consider continuous functions on compact subsets of $\rrp$.  In this situations, the Heine-Cantor Theorem (\citep[Theorem 27.6]{munkres2014topology}) states every continuous function is uniformly continuous; i.e. $C(\xxx,\yyy)=\bar{C}(\xxx,\yyy)$.  
}%
{\color{black}{and this notion of convergence in $\bar{C}(\xxx,\yyy)$ is defined via convergent sequences as follows.  A sequence $\{f_n\}_{n=1}^{\infty}$ in $\bar{C}(\xxx,\yyy)$ converges \textit{uniformly on compact sets} to some $f\in\bar{C}(\xxx,\yyy)$ if for every compact subset $K\subseteq \xxx$ and every $\epsilon>0$, there is some positive integer $N$ such that for any $n\geq N$:
\[
\max_{x \in K}
    \, 
d_{\yyy}\left(
    f_n(x)
        ,
    f(x)
\right)<\epsilon
.
\]
Next, we discuss a few distinguished types of continuous functions relevant to our analysis.  
}}
\subsection{Homeomorphisms and Homotopies}\label{ss_Background_Homotop}
Topology studies geometric properties which are invariant up to continuous deformation.  The strongest such notion is that of a \textit{homeomorphism} $\phi$ from a metric space $\xxx$ to a metric space $\yyy$, which is a continuous bijection with continuous inverse.  Effectively, since most topological properties are preserved either by continuous functions or by their inverses, then the existence of a homeomorphism $\phi:\xxx\rightarrow \yyy$ implies that $\xxx$ and $\yyy$ are topologically identical.

The existence of a homeomorphism between two metric spaces $\xxx$ and $\yyy$ is a very strict condition implying that both spaces have many identical topological properties.  Rather, both spaces can be considered as topologically similar if one can be ``progressively deformed''.  To formalize this idea, we need to define a \textit{homotopy} between any two continuous functions $f,g:\xxx\rightarrow \yyy$, which is a continuous function $F:[0,1]\times \xxx\rightarrow \yyy$ satisfying $F(0,x)=f(x)$ and $F(1,x)=g(x)$; here, $[0,1]\times \xxx$ has the \textit{product metric}, defined by
\vspace{-.5em}
$$
d_{[0,1]\times \xxx}\left(
(t_1,x_1),(t_2,x_2)
\right)\triangleq 
\sqrt{
|t_1-t_2|^2
+
d_{\xxx}(x_1,x_2)^2
}.
\vspace{-.5em}
$$
{\color{black}{
We think of our two spaces as being topologically similar if there are continuous functions $f:\xxx\rightarrow \yyy$ and $g:\yyy\rightarrow \xxx$ for which there is a homotopy between $g\circ f$ and the identity on $\xxx$, as well as as homotopy between $f\circ g$ and the identity on $\yyy$.  NB, if a homeomorphisms $f$ between $\xxx$ and $\yyy$ exists then we may take $g:=f^{-1}$, $F(t,x)=x$ for all $(t,x)\in [0,1]\times \xxx$ and $G(t,y)=y$ for all $(t,y)\in [0,1]\times \yyy$ to be the relevant homotopies.  

Not all spaces which are homeomorphic are homotopic.  In particular, the most relevant instance of this for this paper is the existence of a homotopy between a space and a point; or more generally, the existence of a}} 
homotopy between a continuous function $f:\xxx\rightarrow \yyy$ and a constant function from $\xxx$ to $\yyy$.  If $f$ is homotopic to a constant function then we will say that $f$ is said to be \textit{null-homotopic}.  In general, any two $f,g \in C(\xxx,\yyy)$ need not be homotopic; however, the situation simplifies drastically when the output space is Euclidean.  {\color{black}{NB, Euclidean spaces are precisely those relevant for most classical (uniform) universal approximation theorems \citep{hornik1989multilayer,Yarotski,kidger2019universal}.}}
\begin{example}
\label{ex_no_normed_linear}
If $\xxx$ is a normed-linear space then every $f \in C(\xxx,
\rrm
)$ is null-homotopic via the homotopy $(t,x)\mapsto
(1-t)f(x)$.  In particular, every $f \in C(\rrflex{p},\rrm)$ is null-homotopic.  
\end{example}
%
{\color{black}{In contrast with Example~\ref{ex_no_normed_linear}, not all continuous functions are null-homotopic; for instance, it can be shown that the identity map of the circle is not null-homotopic.  Thus, we can interpret null-homotopy as a formalization the idea that a function is ``globally topologically simple''.  

Homotopies allows us to define a key topological property, relevant to our analysis, called \textit{simply connectedness}.  A metric space $\xxx$ is said to be simply connected if any pair of paths $\gamma_0,\gamma_1:[0,1]\rightarrow \xxx$ with the same endpoints, i.e: $\gamma_0(i)=\gamma_1(i)$ where $i=0,1$, there is a homotopy $H$ from $\gamma_1$ to $\gamma_2$ which fixes the endpoints; i.e. $H(t,0)=\gamma_0(t)$ and $H(t,1)=\gamma_1(t)$ for all $t\in [0,1]$.  As in Example~\ref{ex_no_normed_linear}, all Euclidean spaces are simply connected, the circle is simply connected, but one can show that the Torus%
\footnote{The Torus is defined as $\rr^2$ with the equivalence relation $(x,y)\sim (x+1,y+1)$ for each $(x,y)\in \rr^2$ (see \citep[page 46 and Proposition 1.6]{HatcherAlgebraicTopology})}%
 is not.  
%
}} 
\subsection{Riemannian Geometry}\label{ss_Background_Riem_Geo_intro_version}
Fix $p \in \nn$.  A \textit{($p$-dimensional) manifold} is a space which locally topologically resembles $\rrp$.  More formally
, a manifold is a topological space $\xxx$ for which there is an \textit{atlas} to $\rrp$; i.e: a family $\{\phi_{\alpha},U_{\alpha}\}_{\alpha \in A}$ of open subset $U_{\alpha}\subseteq \xxx$ with $\cup_{\alpha}\, U_{\alpha}=\xxx$ and homeomorphisms $\phi_{\alpha}:U_{\alpha}\rightarrow \rrp$.  
{\color{black}{%
More broadly, a \textit{manifold with boundary} refers to a topological space $\xxx$ for which there is a collection $\{\phi_{\alpha},U_{\alpha}\}_{\alpha \in A}$ (also called an atlas when clear from the context) of open subsets $U_{\alpha}\subseteq \xxx$ with $\cup_{\alpha}\, U_{\alpha}=\xxx$ and homeomorphisms $\phi_{\alpha}$ from $U_{\alpha}$ to either $\rrp$ or the ``half-space'' $\{(x_1,\dots,x_p)\in \rr^p:\,x_p\geq 0\}$.  Unless otherwise specified the term ``manifold'' will always refer to a manifold without boundary.%
}}%
  We focus on manifolds whose geometry is locally comparable to $\rrp$, and not only their topology; i.e.: \textit{Riemannian manifolds.}  

Broadly speaking, a $p$-dimensional \textit{complete Riemannian manifold} (without boundary) is a complete metric space $\xxx$, with metric $d_{\xxx}$, for which there are meaningful local notions of length, volume, curvature and differentiation, all of which are locally comparable to Euclidean space.  

{\color{black}{We will always assume that our Riemannian manifolds are complete, since this is a standard assumption made both when designing learning models of Riemannian manifolds \citep{Hauberg2013KalmanUnscentedManifold,Fletcher2013,MR3763767_BayesianMixed_Learning} and when optimizing those models \citep{lezcano2019trivializations,MR4267082_Optimization_Riemannian}.  We impose geodesic completeness of our Riemannian manifolds, since amongst other things, it rules out pathological geometries such as $\rr^2-\{0\}$ with the Riemannian metric inherited from the Euclidean space $\rr^2$, wherein for example one cannot realize the distance between the points $(0,1)$ and $(0,-1)$ with a distance-minimizing geodesic.  }}

The local comparability happens on two fronts.  The $0^{th}$-order comparability requires that every $x \in \xxx$ be contained in some sufficiently small open ball $B_{\xxx}(x,\delta)\triangleq \left\{
z \in \xxx:\, d_{\xxx}(z,x)<\delta
\right\}$, for some $\delta>0$, which can be mapped{\color{black}{, via a smooth homeomorphism with smooth inverse,}} onto a sufficiently small Euclidean ball centered at $0$ and of radius $\epsilon>0$; we denote the latter by $B_{\rrp}(0,\epsilon)\triangleq \left\{
z \in \rrp:\, \|z-0\|<\epsilon
\right\}$.  

The first-order compatibility happens on the infinitesimal level by a set of copies of $\rrm$ lying tangential to each $x \in \xxx$ called \textit{tangent spaces}, each of which is denoted by $T_x(\xxx)$.  Each of these tangent spaces comes equipped with an inner product $g_x$, varying smoothly in $x$, which is used to formulate infinitesimal notions of angle and distance.  Naturally, the $0^{th}$ and first-order comparability must be consistent and this happens when the distance $d_{\xxx}(x_1,x_2)$ between any two points $x_1,x_2 \in \xxx$ is realized by the \textit{arc length} of an optimally efficient smooth path $\gamma:[0,1]\rightarrow \xxx$ beginning at $x_1$ and ending at $x_2$.  Analogously to $\rrm$, the arc length of any such path is measured by
$
\int_0^1 
\sqrt{
g_{\gamma(t)}\left(
    \dot{\gamma}(t)
,
    \dot{\gamma}(t)
\right)
}
dt;
$
where $\dot{\gamma(t)}$ denotes the velocity vector at $\gamma(t)$.  Any such path, called a \textit{geodesic}, exists and is locally characterized as the unique solution to a particular \textit{ordinary differential equation}, called the \textit{geodesic equations} whose initial conditions determine the location and initial velocity of the geodesic.  For any $x \in \xxx$, there corresponds a maximal Euclidean ball $B_{\rrp}(0,\operatorname{inj}_{\xxx}(x))$ whose elements are all possible initial velocities to geodesics emanating from $x$ and for which the map $\operatorname{Exp}_{\xxx,x}: v\to \gamma(1)$ sending any initial velocity $v \in B_{\rrp}(0,\operatorname{inj}_{\xxx}(x))$ to the point $\gamma(1)$, where $\gamma$ is the geodesic beginning at $x$ with initial velocity $v$ is well-defined on the entire tangent space
, and it is a homeomorphism near the origin.  The quantity $\text{inj}_{\xxx}(x)\in [0,\infty]$ is called the \textit{injectivity radius at $x$} and the map $\operatorname{Exp}_{\xxx,x}$ is the \textit{Riemannian exponential map at $x$}.   
%

Suppose that $\dim(\xxx)>1$.  {\color{black}{Given any $x\in \xxx$, consider an arbitrarily small triangle with vertex at $x$ and whose sides are formed by geodesics emanating from $x$ with initial velocities $v_1,v_2\in T_x(\xxx)$, and let $\pi_x(u,v)$ denote the $2$-dimensional linear subspace of $T_x(\xxx)$ spanned by $v_1$ and $v_2$}}.  The ratio of the gap between the sum of angles of that geodesic triangle with the sum of the angles of a triangle in Euclidean space $\pi_x(u,v)$, over the area of that geodesic triangle is a description of the curvature of $\xxx$ at $x$.  It is called the \textit{sectional curvature} and denoted by $K_{\xxx}(\pi_x(u,v))$.  
We denote the set of all such smoothly varying tangent planes by $G_{p,2}(\xxx)$.
Similar methods can be used to define the \textit{intrinsic volume} of any Borel subset $B\subseteq \xxx$, denoted by $\operatorname{Vol}_{\xxx}(B)$.  We say that a Riemannian manifold $\xxx$ is orientable if it is impossible to smoothly move a two-dimensional figure along $\xxx$ in such a way that the moving eventually results in the figure being flipped. 
Additional details surrounding Riemannian geometric can be found within the paper's appendix.


\section{Main Results on GDNs}\label{s_main_results}
In the remainder of this paper, we require that the geometries of the input and output spaces are ``non-singular"; by which we mean that their curvature does not become unbounded and that the volume of any metric ball (of positive radius) never vanishes.  Formally, we maintain the following.  
\begin{assumption}[Non-Degenerate Geometry: {\cite{CheegerGromovTaylorTheorem1982}}]\label{ass_non_degenerate_spaces}\hfill\\
There exist constants $v_{\xxx},k_{\xxx}>0$ satisfying:
\begin{enumerate}
    \item[(i)] $
    \sup_{\pi_x(u,v):x\in \xxx, \,\pi_x(u,v) \in G_{p,2}(\xxx)} 
\left|
K_{\xxx}(\pi_x(u,v))
\right|\leq 
k_{\xxx}
,
    $
    \item[(i)] For any $\operatorname{diam}(\xxx)>r>0$, $\inf_{x \in \xxx}\operatorname{Vol}_g\left(
    B_{\xxx}(x,r)
    \right)>0.
    $
\end{enumerate}
Moreover, mutatis mutandis, (i) and (ii) also hold for $(\yyy,h)$.  
\end{assumption}
\begin{remark}\label{remark_concerning_ass_non_degenerate_spaces}
Pathological input or output space failing Assumptions~\ref{ass_non_degenerate_spaces}, have been identified in the partner papers \cite{DegenerateExamplesI}, and \cite{DegenerateExamplesII}.  However, it is difficult to imagine these constructions arising in practice.
\end{remark}
\subsection{Differentiable Geometric Deep Learning is a Local Problem}\label{ss_Main_Obstructions}
Our first theoretical contribution is, to the best of our knowledge, the only known necessary condition for a function to be universally approximable (uniformly on compact sets).  The result states that \textit{any model class} $\fff\subseteq C(\xxx,\yyy)$ is universal in $C(\xxx,\yyy)$ only if every function in $C(\xxx,\yyy)$ can be continuously deformed into some model in $\fff$.  
\begin{lemma}[Deformability is Necessary for Universality]\label{lem_top_nec}
Let $\fff\subseteq C(\xxx,\yyy)$ and $f \in C(\xxx,\yyy)$.  Then, for every $\epsilon>0$ and every non-empty compact subset $K\subseteq \xxx$ there exists a $\hat{f}\in \fff$ satisfying:
$$
\sup_{x \in K}\, d_{\yyy}(f(x),\hat{f}(x))<\epsilon,
$$
\textit{only if}: there exists an $H\in C([0,1]\times \xxx,\yyy)$ and a model $\tilde{f}\in \fff$ such that, for every $x \in \xxx$:
\begin{align}
    H(0,x)=f(x)
    \mbox{ and }
    H(1,x)=\tilde{f}(x)
    .
    \label{eq_lem_top_nec_same_homotopy_type}
\end{align}
\end{lemma}
In the non-Euclidean setting, universal approximation is faced with topological obstructions which are never present in the Euclidean setting of \cite{pinkus1999approximation} or \cite{kidger2019universal}, or in the more general $\rr$-valued settings considered in \cite{chen2018neural} and in \cite{Yarotsky2021InvariantNets}.  
\begin{example}[{$\rrm$-Valued Maps are ``Simple"}]\label{ex_Euclidean_is_simple}
The necessary condition for universality of Lemma~\ref{lem_top_nec} is always satisfied when considering $\rrm$-valued functions approximated by DNNs in $\NN[p,m]$ for any activation function $\sigma \in C(\rr)$.  This is because, given any $\hat{f}\in \NN[p,m]$ and any target function $f\in C(K,\rrm)$, the following homotopy satisfies condition~\eqref{eq_lem_top_nec_same_homotopy_type}:
$$
H(t,x)\triangleq t f(x) + (1-t)\hat{f}(x)
.
$$  
\end{example}
In contrast to Example~\ref{ex_Euclidean_is_simple}, the behaviour of functions between even the simplest non-Euclidean geometries can be wildly complicated.  For instance, there are infinitely many functions from the sphere to the circle which fail condition~\eqref{eq_lem_top_nec_same_homotopy_type}.  Let $S^k\triangleq \{x\in \rrflex{k+1}:\, \|x\|=1\}$.
\begin{example}[Maps in Simple Non-Euclidean Manifolds are complicated]\label{ex_massive_difference}
There is a countably infinite family $\fff\subseteq C(S^3,S^2)$ whose members can only approximate themselves, in the sense that: if $f_1,f_2\in \fff$ and $f_1\neq f_2$ there does not exist an $H\in C([0,1]\times S^3,S^2)$ satisfying $H(0,\cdot)=f_1$ and $H(1,\cdot)=f_2$.  (The proof of this fact is in the paper's appendix).
\end{example}

Our first main result focuses on the observation that any model built from local data, in the sense of Figure~\ref{fig_visualization_GDNs}, can only ``globally approximate" if they are null-homotopic.

This necessary condition is summarized graphically in Figure~\ref{fig_visualitzation_of_Theorem2}, where we notice two ``types of functions" on the sphere.  The first is the identity function thereon (illustrated by the {\color{gray}{gray}} sphere itself), this is an example of a non-universally approximable target function.  The second ``type of function" is illustrated by each of the {\color{red}{coloured}} paths, these functions are topologically defined by any model constructed from a \textit{``local interpretation of Figure~\ref{fig_Non_Euc_tikz}"} (formalized below) and they illustrate functions which can be universally approximated.  Intuitively the difference between the {\color{gray}{gray}} function and the {\color{red}{coloured}} functions is that the gray function can never be asymptotically deformed into one of the coloured functions without puncturing the sphere.  I.e.: no such deformation as described by~\eqref{eq_lem_top_nec_same_homotopy_type} in Lemma~\ref{lem_top_nec} is possible.  

\begin{figure}[H]
    \centering
    \includegraphics[width=15em]{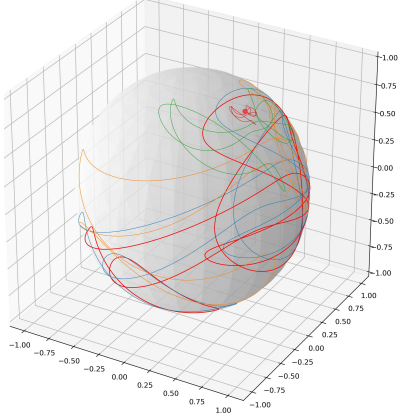}
    \caption{{Visualizing the null homotopy-type condition of Theorem~\ref{thrm_negative_motiation}}}
    \label{fig_visualitzation_of_Theorem2}
\end{figure}

We formalize the phrase ``local interpretation of Figure~\ref{fig_Non_Euc_tikz}".  
\begin{definition}[Locally-Defined GDL Model]\label{defn_locally_defined_models}
Fix atlases $\{\phi_{\alpha},U_{\alpha}\}_{\alpha \in A}$ and $\{\rho_{\zeta},V_{\zeta}\}_{\zeta \in Z}$ of $\xxx$ and of $\yyy$, respectively, and a ``Euclidean model class" $\emptyset\neq \fff\subseteq C(\rrp,\rrm)$.  A \textit{locally defined GDL model} is a family $\{\fff_{\alpha}\}_{\alpha\in A}$ of models
$\fff_{\alpha}
        \subseteq 
    C\left(
        \phi_{\alpha}(U_{\alpha})
            ,
        \yyy
    \right)
$, each of which is defined via:
\begin{equation}
    \fff_{\alpha}\triangleq \left\{
        \rho_{\zeta}^{-1}\circ \hat{f}\circ \phi_{\alpha}:\, \hat{f}\in \fff,\,\zeta \in Z
        \right\}
    \label{eq_general_locally_defined_model_class}
    .
\end{equation}
\end{definition}

\begin{theorem}[Only Null-homotopic Maps are Approximable by Locally-Defined GDLs]\label{thrm_homotopic_necessary_condition}
Let $\xxx$ and $\yyy$ be complete connected Riemannian manifolds of positive dimension, satisfying~\eqref{ass_non_degenerate_spaces}, and let $\{\fff_{\alpha}\}_{\alpha\in A}$ be a locally-defined GDL model.  For every $\alpha \in A$ and every compact $\emptyset\neq K\subseteq U_{\alpha}$:
\begin{enumerate}
    \item[(i)] For every $\zeta \in Z$ and every $g \in C(\rrp,\rrm)$, the map $\rho^{-1}_{\zeta}\circ g\circ\phi_{\alpha}$ is well-defined and in $C(K,\yyy)$,
    \item[(ii)] If there exists $f \in C(K,\yyy)$ which is not null-homotopic then, there is an $\epsilon>0$ satisfying:
    $$
    \underset{
\underset{
\zeta \in Z, \, \alpha \in A
}
{
g \in C(\rrp,\rrm)
}
}{\inf}\,
\underset{z \in K}
{\sup}\,
d_{\yyy}\left(
\rho_{\zeta}^{-1}
\circ g \circ 
    \phi_{\alpha}(z)
        ,
    f(z)
\right)\geq \epsilon
.
$$
\end{enumerate}
\end{theorem}
Theorem~\ref{thrm_homotopic_necessary_condition} is a simple necessary condition for a map with non-Euclidean outputs to be globally approximable by feedforward networks.  In particular, when the global geometry of $\yyy$ differs too greatly from Euclidean space, then functions which are not globally approximable necessarily \textit{exist}.  
\begin{theorem}[Locally-Defined GDL Models are not Uniform Universal Approximators]\label{thrm_negative_motiation}
Let $\xxx$ and $\yyy$ be complete connected Riemannian manifolds, with $p\geq m$, $\yyy$ compact and orientable, $\xxx$ satisfies Assumption~\ref{ass_non_degenerate_spaces}, and let $\{\fff_{\alpha}\}_{\alpha\in A}$ be a locally-defined GDL model.  Then, there exists $\alpha\in A$, $\zeta \in Z$, a compact $\emptyset \neq K\subseteq \xxx$, $x \in K\subseteq U_{\alpha}$, an $\epsilon>0$, and an $f\in C(K,\yyy)$ such that:
\begin{enumerate}
\item[(i)] Each $\hat{f}\in \fff_{\alpha}$ is a well-defined function in $C(K,\yyy)$,
\item[(ii)] $
\sup_{z \in K}\,\inf_{\hat{f}\in \fff_{\alpha}}\,
d_{\yyy}\left(
\hat{f}(z),f(z)
\right)\geq \epsilon
.
$
\end{enumerate}
\end{theorem}
\subsubsection{{Discussion: Why Uniform Approximation Poorly Suited to GDL Problems}}\label{sss_discussion_why_so_major_problem}
The necessary condition for universality identified in Lemma~\ref{lem_top_nec} causes a major obstruction to building universal approximators in $C(\xxx,\yyy)$.  This is because, the model class $\fff\subset C(\xxx,\yyy)$ needs to exhaust all the homotopy types therein.  However, verifying that this condition is met is at-least as difficult as computing the homotopy groups of the output space $\yyy$ (see \cite{FuchsFomenkoHomotopicalTopology2016Edition2} for details), which has recently been shown in \cite{MR3268623} and in \cite{matousek2013computing} to be an (at-least) \textit{$\text{NP}$-hard} problem.  Therefore, verifying the compatibility of any model class $\fff\subset C(\xxx,\yyy)$ with the geometry of $\yyy$ is \textit{computationally infeasible}.

In the simplified setting where one instead considers only locally defined GDL models, Theorem~\ref{thrm_negative_motiation} guarantees that when $\xxx$ and $\yyy$ are both compact and connected Riemannian manifolds of positive dimension then there are functions in $C(\xxx,\yyy)$ which cannot approximate all functions in $C(\xxx,\yyy)$ uniformly on \textit{arbitrarily large compact subsets of $\xxx$}.  Thus, any locally-defined GLD model is faced with the following problem: either the conditions for Theorem~\ref{thrm_negative_motiation} are met, and therefore, the model class is not universal, or it is computationally infeasible to verify if the model class is universal.

Therefore, uniform approximation on ``uncontrolled compact subsets" (i.e. of arbitrarily large maximum diameter) of a function between general Riemannian manifolds is not a well-suited notion of ``universal approximation" for geometric deep learning.  However, as we will now show, all these obstructions vanish when the models are only required to approximate the target function on compact subsets of $\xxx$ with a certain maximum diameter.  

We now introduce the notion of \textit{``controlled universal approximation"} (i.e.: universal approximation on compact subsets with a specific bounded maximum diameter). Moreover, we find that our GDN models are universal in this sense. We show that controlled universal approximation coincides with uniform approximation on compact sets when $\xxx$ and $\yyy$ are non-positively curved (e.g. Euclidean space).  Therefore, this notion of universality strictly extends the familiar notion of \cite{Hornik}, \cite{pinkus1999approximation}, and \cite{kidger2019universal} to the non-Euclidean setting without any of the topological obstructions of the ``naive" uniform approximation on compact sets notion of universality.  
\begin{remark}[Connection to Relative Forms of Uniform Convergence]\label{remark_relative_uniform}
In the general case, where $\xxx$ or $\yyy$ may have somewhere positive curvature (e.g. any compact Riemannian manifolds of positive dimension) our notion of ``controlled universality approximation" is most similar to density in the relative uniform convergence topologies introduced in \cite{ArensDugundjiIntroTopoFunctionSpaces1951} and studied in \cite{McCoyNtantu1988UniformSetOpen}, \cite{UniformTopologyNokhrinOsipov2009}, and in \cite{BouchairKelaiaia1024SetOpenUniformComparison}.  
\end{remark}
\subsection{Controlled Universal Approximation}\label{ss_Main_local}
We also make use of the function sending any $x \in \xxx$ and any $K \in (0,\infty]$ to:
$$
\delta(\xxx,x,K) \triangleq \sup_{0<r< K}
r \frac{
\operatorname{Vol}_{\xxx}\left(
B_{\xxx}(x,r)
\right)
}{
\operatorname{Vol}_{\xxx}\left(
B_{\xxx}(x,r)
\right)
+
\operatorname{Vol}_{T_x(\xxx)}\left(
B_{T_x(\xxx)}(0,2r)
\right)
};
$$
note, that $\delta(\yyy,y,K)$ is defined analogously.  
Our analysis relies on the following function, mapping any $K \in \rr$ to the extended-real number:
$$
K^{\star}\triangleq \begin{cases}
\frac{\pi}{4 \sqrt{K}} & :  K>0 \\
\infty & :  K\leq 0.
\end{cases}
$$  
Our approximation results concern the following locally-defined GDL model.   
\begin{definition}[Geometric Deep Feedforward Networks]\label{defn_GDN}
Fix $\sigma \in C(\rr)$.  A geometric deep feedforward network (GDN) from $\xxx$ to $\yyy$ at $x \in \xxx$ with activation function $\sigma$, is a function $\hat{f}\in C(B_{\xxx}(x,\operatorname{inj}_{\xxx}(x)),\yyy)$ with representation:
$$
\hat{f}\triangleq \operatorname{Exp}_{\yyy,y}\circ g\circ \operatorname{Exp}_{\xxx,x}^{-1},
$$
for some $g\in \NN[p,m]$ and some $y\in \yyy$.
\end{definition}
\begin{theorem}[Controlled Universal Approximation]\label{thrm_main_Local}
Let $\xxx$ and $\yyy$ be connected complete Riemannian manifolds satisfying Assumption~\ref{ass_non_degenerate_spaces}, of respective dimensions $p$ and $m$, suppose that $\xxx$ is compact, and let $\sigma$ be an activation function satisfying Assumption~\ref{ass_Kidger_Lyons_Condition}.  For any continuous function $f:\xxx\rightarrow \yyy$, any $\epsilon>0$, and any $x \in \xxx$, if:
\begin{equation}
\delta < \min\left\{
\operatorname{inj}_{\xxx}(x),
\omega^{-1}\left(f,
\operatorname{inj}_{\yyy}(f(x))
\right)
\right\},
    \label{eq_local_curvature_condition}
\end{equation}
then the following hold:
\begin{enumerate}
    \item[(i)] \textbf{Well-Definedness of GDN:} For every $g\in \NN[p,m,p+m+2]$ 
the map 
$
    \hat{f}\triangleq \operatorname{Exp}_{\yyy,f(x)}\circ g\circ \operatorname{Exp}_{\xxx,x}^{-1}
$
is well-defined on $\overline{B_{\xxx}(x,\delta)}$,
\item[(ii)] \textbf{Controlled Universal Approximation:} There is a GDN $\hat{f}$ as in (i) satisfying:
$$
\sup_{\tilde x \in \overline{
B_{\xxx}\left(
x,\delta
\right)
}} 
\,
d_{\yyy}\left(
f(\tilde x),\hat{f}(\tilde x)
\right)\leq \epsilon
.
$$  
\item[(iii)] \textbf{GDN Complexity Estimate:} 
The depth of $g$ is recorded in Table~\ref{tab_rates_and_depth}, and it depends on $\sigma$'s regularity.
\end{enumerate}
Furthermore, the right-hand side of~\eqref{eq_local_curvature_condition} is lower-bounded via:
\begin{equation}
\min\left\{
\delta(\xxx,x,k^{\star}_{\xxx}),
\omega^{-1}\left(f,
\delta(\yyy,f(x),k^{\star}_{\yyy})
\right)
\right\}
\leq 
\min\left\{
\operatorname{inj}_{\xxx}(x),
\omega^{-1}\left(f,
\operatorname{inj}_{\yyy}(f(x))
\right)
\right\}
\label{eq_generic_lower_bound_Gromov_Application}
.
\end{equation}
\end{theorem}
\begin{table}[!ht]
    \centering
    \begin{tabular}{l|l}
    \toprule
        Regularity of $\sigma$ & Order of Depth \\
    \midrule
        $C^{\infty}(\rr)$ + Non-polynomial
         & 
         $
         O\left(
        \frac{
            m (2\delta)^{2p} 
            }{
            \kappa_2^{2p}\left(\omega^{-1} \left(
                f 
        , \frac{\epsilon\kappa_1}{(1+\frac{p}{4})m} \right)\right)^{2p} }
        \right)
         $
         \\
         Non-affine polynomial\footnote{We must allow for one extra neuron per layer.}
         & 
         $
         O \left(
    \frac{
        m(m+p)(2\delta)^{4p+2}
        }{
        \kappa_2^{4p+2} \left(\omega^{-1} \left(
        f 
        , \frac{\epsilon\kappa_1}{(1+\frac{p}{4})m} \right)\right)^{4p+2}
        }
        \right)
         $
         \\
         $C(\rr)$ + Non-polynomial
         &
         $
          O\left( 
    \frac{
        m (2\delta)^{2p}
    }{
        \kappa_2^{2p} 
        \omega^{-1}\left(
            f 
                , 
            \frac{\epsilon\kappa_1}{2m(1+\frac{p}{4})}    
        \right)^{2p} 
        \left(\kappa_2 \omega^{-1} \left( \sigma, \frac{\epsilon}{2Bm(2^{(2\delta)^{2}[\omega^{-1}(
        f 
        , \frac{\epsilon\kappa_1}{2m(1+\frac{p}{4})})]^{-2}+1} -1)} \right) \right)
    }
        \right)
         $
         \\
    \bottomrule
    \end{tabular}
    Where $\kappa_1,\kappa_2>$ depend only on the curvature of $\xxx$ at $x$ and of $\yyy$ at $f(x)$, respectively.  
    \caption{Approximation rates for GDNs based on activation function}
    \label{tab_rates_and_depth}
\end{table}
%
Theorem~\ref{thrm_main_Local} guarantees that universal approximation by GDNs on compact subsets of general Riemannian manifolds whose size is ``controlled by the right-hand side of~\eqref{eq_local_curvature_condition}" is possible; even if Theorem~\ref{thrm_negative_motiation} mandates it typically fails ``globally"; i.e., for arbitrarily large compact subsets of $\xxx$.  Thus, the ``radius" in~\eqref{eq_local_curvature_condition} quantifies the \textit{gap between ``local" and ``global" universal approximation}.
\begin{definition}[Universality Radius]\label{defn_approx_ratius}
Let $\xxx$ and $\yyy$ be Riemannian manifolds, and let $f\in C(\xxx,\yyy)$.  The universality radius of $f$ at any $x \in \xxx$ is defined to be the quantity:
$$
\mathcal{U}_f(x)\triangleq 
\min\left\{
\operatorname{inj}_{\xxx}(x),
\omega^{-1}\left(f,
\operatorname{inj}_{\yyy}(f(x))
\right)
\right\}
.
$$
\end{definition}
\begin{remark}[Analogy: Taylor Expansions and Controlled Universal Approximation]\label{remark_Taylor_Local_Universal}
\hfill
The \textit{universality radius} of any $f \in C(\xxx,\yyy)$ plays a similar role to the radius and interval of convergence of a smooth function in classical calculus on $\rr$.  This is because, on the interval of convergence about any $x \in \xxx$ a smooth function $f\in C^{\infty}(\rr,\rr)$ can be locally approximated to arbitrary precision by its Taylor series.  Analogously, any $f \in C(\xxx,\yyy)$ can be universally approximated by a GDN on $B_{\xxx}(x,\mathcal{U}_f(x))$.  In both cases, the radius depends on the point of the input space about which the approximation is performed and on the regularity of the function.  
\end{remark}

One may ask if there is a broad class of input/output spaces for which the obstruction of Theorem~\ref{thrm_negative_motiation} vanishes.  In such cases, the GDN architecture can be developed about any point of the input space with the confidence that the the lower-bound~\eqref{eq_generic_lower_bound_Gromov_Application} is infinity.  
\subsubsection{Local-to-Global Universality for Cartan-Hadamard Manifolds}\label{sss_Cartan_Hadamard_Manifolds}
Our search for input or output spaces with generically infinite universality radii begins with the Cartan-Hadamard Theorem (see \citep[Corollary 6.9.1]{jost2008riemannian}) and \textit{Cartan-Hadamard} manifolds.  These are simply connected, complete Riemannian manifolds of everywhere non-positive sectional curvature.  Three important examples in geometric deep learning are the Hyperbolic spaces, the manifold of non-degenerate \textit{Gaussian probability measures} with the Fisher-Rao distance (from information geometry; see \citep[Equation 3.22]{jostInformationGeometry}), and the familiar Euclidean spaces.  

For Cartan-Hadamard manifolds, we have the following ``local-to-global" result.  That is, the next result describes a broad range of situations in which controlled universal approximation coincides with density in the uniform convergence on compact sets topology on $C(\xxx,\yyy)$.   
\begin{corollary}[From Local to Global Universal Approximation]\label{cor_Cartan_Hadamard}
If $\xxx$ and $\yyy$ are \hfill\\ Cartan-Hadamard manifolds, then the following estimate holds:
$$
\inf_{x \in \xxx}\inf_{f \in C(\xxx,\yyy)}\, \mathcal{U}_f(x) =\infty
.
$$
\end{corollary}

\subsection{Breaking the Curse of Dimensionality via Efficient Datasets}\label{subsection_data_dependant_dimenion_free_approx}
\subsubsection{Discussion: Overview of Our Approach}\label{subsection_data_dependant_dimenion_free_approx_explanation}
Thus far, as in most universal approximation papers, the objective has been to approximate $f\in C(\xxx,\yyy)$ on arbitrary compact subsets of $\xxx$ for which universal approximation is not obstructed by Theorem~\ref{thrm_negative_motiation}.  Indeed, classical constructive approximation results found in \cite{BookModCOnt} guarantee that cursed approximation rates (as in Theorem~\ref{thrm_main_Local}) are unavoidable.  This phenomenon can be equally seen in the simple Euclidean case where it is confirmed in \cite{Yarotski} that the best possible approximation rates for feedforward networks with ReLU activation function are unavoidably exponential in the involved spatial dimensions and the approximation error.  Thus, the universal approximation problem is ``cursed from the start" since we looked for a general rate which applies to any uniformly continuous function on any compact subset of the input space.

As pioneered in the quantitative approximation theorem of \cite{barron1993universal}, the author found that the curse of dimensionality can be avoided if restrictions are placed on the set of functions which are considered for approximated.  Since then, several other authors; e.g. \cite{barron1993universal}, \cite{yarotsky2019phase}, \cite{SIEGEL2020313}, \cite{QuantitativeDeepReLUSobolev}, \cite{suzuki2018adaptivity}, and \cite{Florian2_2021}, have identified sub classes of function which can be approximated by DNNs whose number of parameters does not depend adversely on the dimension of the input and output spaces\footnote{These approximation results are not all in $C(\rrp,\rrm)$ for the uniform distance; nevertheless, they all have the commonality of avoiding the curse of dimensionality by restricting the class of approximated functions within some larger function space in which the DNNs are dense.} (potentially in different function spaces).  
We highlight that, each of these results takes a ``\textit{functioncentric} perspective" in that they focus on the impact of $f$'s regularity on the neural network approximation rates and omit the impact of the dataset on these approximation rates.  

Here, we instead consider a ``\textit{datacentric} perspective" wherein we ask: \textit{given an $f:\xxx\rightarrow \yyy$, on which datasets}\footnote{Note a dataset need not be a training dataset but rather refers to the set on which we expect our approximation to hold.}\textit{ $\XX$ (i.e.: non-empty subsets of $\xxx$) can $f$ be uniformly approximated by a GDN whose number of parameters does not depend adversely on the dimension of $\xxx$ and of $\yyy$?}  
Such datasets will be called \textit{efficient for $f$}.  We will see that, if $f$ is sufficiently smooth then any dataset $\XX$ is efficient for $f$ and, conversely, every ``real-world dataset" (i.e. finite and non-empty) $\XX$ is efficient for any (potentially discontinuous) function $f:\xxx\rightarrow \yyy$.  

\begin{remark}[Implications in the Euclidean Case]\label{remark_foreshadowing_Euclidean_efficiency}
In particular, as developed further in Section~\ref{ss_Euc}, our result strictly extend the dimension-free rates for DNNs in $C([0,1]^p,\rr)$ obtained recently in \cite{yarotsky2019phase}.  Hence, even in the Euclidean case, our datacentric perspective is both novel and more general than the functioncentric perspective.  
\end{remark}

The idea of our approach is concisely summarized by Figure~\ref{fig_proof_sketch} wherein see that {\color{darkgreen}{the green function $F$}} coincides with the {\color{deepcarrotorange}{target function $f$}} on the {\color{darkpurple}{dataset $\XX$}} but {\color{darkgreen}{$F$}} is much more regular.  Thus, if we instead approximate the {\color{darkgreen}{more regular function $F$}} by a GDN on all of the input space, then we can do so with a GDN which avoids the curse of dimensionality and simultaneously obtain an equally accurate approximation of {\color{deepcarrotorange}{target function $f$}} on {\color{darkpurple}{dataset $\XX$}} since both {\color{deepcarrotorange}{target function $f$}} and {\color{darkgreen}{the green function $F$}} coincide thereon.  Note that, our next main results does not assume that {\color{darkpurple}{$\XX$}} is finite.  

\begin{figure}[!ht]
    \centering
    \includegraphics[width = 20em]{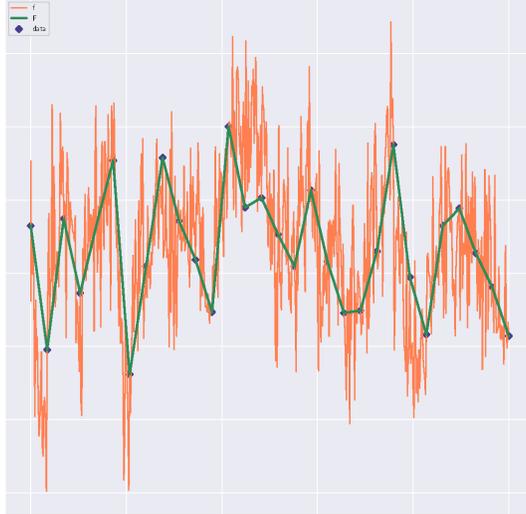}
    \caption{What are efficient datasets?}
    \label{fig_proof_sketch}
\end{figure}

To formalize our task, we need to define what we mean by $F$ being ``regular enough".  We consider an extension of the notion of regularity studied in \cite{yarotsky2019phase} but in the Riemannian context.  Following \cite{jost2008riemannian}, we make the following definition.
We say that a function $f \in \bar{C}(\xxx,\yyy)$ is regular if it has many higher-order partial derivatives, and if the last one of which locally distorts distance up to a linear scaling factor. 
\begin{definition}[{$C^{k,1}_{loc}(\xxx,\yyy)$}]\label{defn_Cklip_loc}
Fix a $k\in \nn_+$, and fix smooth atlases $(\phi_{\alpha},U_{\alpha})_{\alpha \in A}$ and $(\psi_{\zeta},V_{\zeta})_{\zeta \in Z}$ of $\xxx$ and of $\yyy$ respectively.  We say that $f \in \bar{C}(\xxx,\yyy)$ belongs to  $C^{k,1}_{\operatorname{loc}}(\xxx,\yyy)$ if, for every $\alpha\in A$ and every non-empty compact subset $K\subseteq \phi_{\alpha}(U_{\alpha})$ we have:
\begin{equation*}
    \max_{\beta: |\beta|\leq k}
    \max_{x \in K}\, \|D^{\beta} \pi_i\circ \phi_{\alpha}\circ f\circ \psi_{\zeta}^{-1}(x)\|
        +
    \sup_{x,y\in K,\, x\neq y}\, \frac{\|D^{\beta} \pi_i\circ \phi_{\alpha}\circ f\circ \psi_{\zeta}^{-1}(x)-D^{\beta}\pi_i\circ \phi_{\alpha}\circ f\circ \psi_{\zeta}^{-1}(y)\|}{d_Y(x,y)} 
        <
    \infty
;
\end{equation*}
whenever the composition $\phi_{\alpha}\circ f\circ \psi^{-1}_{\zeta}$ is well-defined and where $\pi_i:\rrm\ni x \mapsto x_i\in \rr$ is the canonical projection onto the $i^{th}$-coordinate, $D^{\beta}=\frac{\partial^{|\beta|}}{\partial x_{\beta_1}\dots \partial x_{\beta_{|\beta|}}}$, and $|\beta|$ is the length of the multi-index $\beta$.  
\end{definition}
Therefore, our approach will be the following: \textit{replace the target function $f$ in $\XX$ by a
\textit{sufficiently smooth} $F:\xxx\rightarrow \yyy$.  Here, sufficiently smooth means that $F$ admits all $k$ continuous partial derivatives for some integer $k$ divisible by the dimension $p$ of $\xxx$; that is \textit{$F\in C^{k,1}_{\operatorname{loc}}(\xxx,\yyy)$ and $k=np$ for some positive integer $n$}}.  
In this case, we may approximate $F$ by a GDN depending on few parameters over a ``regular'' subset of $\xxx$ containing $\XX$ and then restrict our approximation of $F$ to $\XX$ thereby efficiently approximating $f$.
In what follows, we will denote the cardinality of a dataset $\XX$ by $\#\XX$.  

\begin{remark}[The roles of $k$ and of $n$]
\label{remark_np_and_k}
Suppose that $\XX$ is finite.  
If $n$ equals to the number of points in $\XX$ (i.e.\ if $n=\#\XX$) then we would be seeking a function $F\in C_{\operatorname{loc}}^{np,1}(\xxx,\yyy)$ satisfying
\begin{align}
\label{eq_Whitney_extension}
    F(x)&=f(x) 
    &\mbox{(for all } x\in \XX)
    .
\end{align}
Under the conditions that $k=\#\XX p$, the ``smoothness'' of $f$'s extension $F$ on the dataset $\XX$ is effectively coupled with the dimension $p$ of $\xxx$ but it is also coupled to the number of datapoints in $\XX$.  In this case, our next result (Theorem~\ref{thrm_efficient_rates}) implies that for any $\epsilon>0$, $f$ can be approximated to $\epsilon$-precision on $\XX$ by a GDN determined by $\mathscr{O}\big(m(m^2-1) \epsilon^{-2p/3(np+1)}\big)$ trainable parameters.  

However, the requirement that $k= \#\XX p$ places a heavy restriction on the candidate smooth functions $F$ which could satisfying~\eqref{eq_Whitney_extension}, since the condition $k= \#\XX p$  necessitates that $F$ must be very smooth whenever $\#\XX$ is large.  
A fortiori, this formulation is meaningless for any infinite $\XX$.  

In fact, for our efficient approximation result (Theorem~\ref{thrm_efficient_rates}) to hold we do not need that $k=\#\XX p$; rather, we only that $F\in C_{\operatorname{loc}}^{k,1}(\xxx,\yyy)$ where $k$ is some positive integer divisible by $p$.  Therefore, by prespecifying some positive integer ($n$) such that $k=np$ and looking for an $F\in C_{\operatorname{loc}}^{k,1}(\xxx,\yyy)$ satisfying~\eqref{eq_Whitney_extension} we may still conclude that $f$ can be approximated on $\XX$ to $\epsilon$-precision by a GDN determined by $\mathscr{O}\big(m(m^2-1) \epsilon^{-2p/3(np+1)}\big)$ parameters.  Moreover, by decoupling $k$ from the cardinality of $\XX$ in this way, we no longer constrain the collection of ``candidate functions'' $F$ satisfying~\eqref{eq_Whitney_extension} for large (but finite) dataset $\XX$.  Furthermore, by decoupling $k$ from $\#\XX$ we can also meaningfully handle the case where $\XX$ is infinite.
\end{remark}

When this is possible, we show that the geometric arguments of Theorem~\ref{thrm_main_Local} may be combined with an extension of the recently efficient approximation results of \cite{yarotsky2019phase} (describing efficient approximation of functions in  $C^{np,1}_{\operatorname{loc}}(\rrp,\rr)$ by models in $\NN[p,1][ReLU]$; where, $ReLU:x\to \max\{0,x\}$) to obtain an efficient approximating of $f$ by GDNs.  The problem of replacing a function $f$ by a $C^{np,1}_{\operatorname{loc}}(\xxx,\yyy)$ function coinciding with it on $\XX$ is equivalent to the problem of extending $f$ on $\XX$ to such a function.  This latter problem is known as the \textit{Whitney Extension Problem}, and dates back to \cite{WhitneyExtensionTheorem_Part_1_Notenopart2existsuntlthe2000s}\footnote{
We require \cite{Fefferman1995ExtensionTheorem}'s Extension Theorem since we are interested in \textit{uniform-type approximation results}.  If one were interested in applying our approach to other notions of approximation, e.g. Sobolev, or Besov norms, then the recent development surrounding extension theorems; see \cite{FeffermanIstralLuliExtensionOperatorSobolev2014}, \cite{HeikkinenIhnatsyevaTuominen2016ExtensionBesovMMS2016}, \cite{AmrosioPuglisi2020ExtensionOperatorLipschitz}, or \cite{BrueDiMarinoSimoneStraExtension2021Lipschitz}, would likely be equally central to obtaining ``datacentric" uncursed rates by GDNs (or even classical DNNs) in those contexts.  
}
.  Fortunately, this long-standing open problem has recently been solved in a series of papers: \cite{BierstoneMilmanPawlucki_Part1_2003}, \cite{Fefferman1995ExtensionTheorem}, and \cite{BierstoneMilmanPawlucki_Part2_2006}.  We leverage these analytic results to solve our efficient universal approximation problem\footnote{
A proper treatment of the Whitney Extension Problem is not aligned with our paper's length target.  The interested reader is referred to: \cite{BrudnyiSquared_Part_1,BrudnyiSquared_Part_2}.}
.

\subsubsection{Efficient Datasets}\label{sss_efficient_datasets}
Our analysis begins by reformulating the conditions of \cite{Fefferman1995ExtensionTheorem}'s Whitney Extension Theorem to suit our controlled universal approximation context.  The best known conditions, to the authors' knowledge, are (in the language of our context) the following.  
\begin{definition}[{Efficient Datasets}]\label{ass_regularity_conditions}
Fix $n\in \nn_+$, an $f:\xxx\rightarrow \yyy$, a dataset $\XX$, and set:
$$
\ln(C^{\star})\triangleq \min\{\ln(\# \XX),{2^{C}} \ln(C+1)\} \mbox{ and }
C\triangleq {\binom{p+np}{p}}.
$$
Then, $\XX$ is \textit{$n$-efficient} for $f$ at $x\in \xxx$ if the following holds:
\textit{for each $\{x_c\}_{c=1}^{C^{\star}}\subseteq \XX$ there exists an $M>0$, independent of $\{x_c\}_{c=1}^{C^{\star}}$, and polynomials $p_c \in C(\rrp,\rrm)$ of degree $np$ satisfying:}
\begin{enumerate}
    \item[(i)] $p_c(\operatorname{Exp}_{\xxx,x}^{-1}(x_c))=\operatorname{Exp}_{\yyy,f(x)}^{-1}(f(x_c))$, for all $c=1,\dots,C^{\star}$
    \item[(ii)] $|\partial^{\beta} p_c(\operatorname{Exp}_{\xxx,x}^{-1}(x_c))_i|\leq M$, for all $c=1,\dots,K$, $|\beta|\leq np$, and $i=1,\dots,p$,
    \item[(iii)] $|\partial^{\beta} (p_c-p_j)(\operatorname{Exp}_{\xxx,x}^{-1}(x_c))_i|\leq M\left\|
    \operatorname{Exp}_{\xxx,x}^{-1}(x_c)) - \operatorname{Exp}_{\xxx,x}^{-1}(x_j))
    \right\|^{nd-|\beta|}
    $, for all $c,j=1,\dots,C^{\star}$, $|\beta|\leq np$, and $i=1,\dots,p$,
\end{enumerate}
where $p_{c}(z)_i$ denotes the projection of $p_c$ onto its $i^{th}$ coordinate evaluated at $z$.  
{\color{black}{We say that $\XX$ is $n$-efficient for $f$ if it is $n$-efficient for $f$ at each $x\in \xxx$.}}
\end{definition}
We begin by showing that functions for which all datasets are efficient extend the class of efficiently approximable functions of \cite{yarotsky2019phase} to the general Riemannian case.  
\begin{proposition}[Every Dataset is Efficient for $C^{np,1}_{\operatorname{loc}}(\xxx,\yyy)$-Functions]\label{ex_triviality_smooth_functions}
Fix $f \in C^{np,1}_{\operatorname{loc}}(\xxx,\yyy)$ {\color{black}{and let $\XX$ be a dataset satisfying the following: there is an $0\leq \eta <1$ and a $x\in \xxx$}} such that:
\begin{equation}
    \XX \subseteq B_{\xxx}(x,\eta\operatorname{inj}_{\xxx}(x)) 
    \qquad
    \mbox{and}
    \qquad
    f(\XX)\subseteq B_{\yyy}(f(x),\eta\operatorname{inj}_{\yyy}(f(x)))
    \label{eq_containement_condition}
    ;
\end{equation}
{\color{black}{then, $\XX\subseteq \xxx$ is $n$-efficient for $f$.}}
In particular, condition~\eqref{eq_containement_condition} always holds if both $\xxx$ and $\yyy$ are Cartan-Hadamard manifolds.  
\end{proposition}
Proposition~\ref{ex_triviality_smooth_functions} is doubly insightful since it implies that functions for which every dataset is efficient are typical, from the approximation-theoretic standpoint.    
\begin{corollary}[Functions for Which Every Dataset is Efficient are Generic]\label{prop_everything_is_nice}
Consider the setting of Proposition~\ref{ex_triviality_smooth_functions}.  {\color{black}{
Let $C^{eff}(\xxx,\yyy)$ denote the set of all $f\in  \bar{C}(\xxx,\yyy)$ with the following property: for every $x \in \xxx$, each $0\leq \eta<1$, and every finite $\XX\subseteq B_{\xxx}(x,\eta \mathcal{U}_f(x))$, there is some positive integer $n$ for which $\XX$ is $n$-efficient for $f$.  Then, the set $C^{eff}(\xxx,\yyy)$ is dense in $\bar{C}(\xxx,\yyy)$.
}}
\end{corollary}
Conversely, datasets for which every function is ``efficient" are also prevalent.  In fact, \textit{every ``real-world dataset"} (i.e. a non-empty finite dataset) has this property.  
\begin{proposition}[Real-World Datasets are Efficient for Any Function]\label{prop_every_finite_dataset_is_efficient}
Let $\XX$ be a finite set and $f:\xxx\rightarrow \yyy$.  Suppose that $\XX\subset B_{\xxx}(x,\eta\mathcal{U}_f(x))$, for some $x\in \xxx$ and some $0<\eta<1$.  Assume that $\#\XX <p$ and that $p$ divides $\#\XX-1$.  Then $\XX$ is $\frac{\# \XX-1}{p}$-efficient for $f$.  
\end{proposition}
Together, Propositions~\ref{ex_triviality_smooth_functions} and~\ref{prop_every_finite_dataset_is_efficient} show that efficient datasets describe a rich host of situations which are well beyond the scope of the classical perspective of assuming additional regularity of $f$.  
To ensure a consistent narrative with the recent developments in \cite{Yarotski}, \cite{YAROTSKYSobolev}, and \cite{yarotsky2019phase}, we focus on normalized datasets.
\begin{assumption}[Normalizable Dataset]\label{defn_normalized}
Let $\XX\subseteq \xxx$ be a dataset and $f:\xxx\rightarrow \yyy$.  Then, $\XX$ is \textit{$f$-normalizable} if there is some $x \in \xxx$ and some $0<\eta<1$ such that $\XX\subset B_{\xxx}(x,\eta \,\mathcal{U}_f(x))$ and 
\begin{equation}
    \text{Exp}_{\xxx,x}^{-1}(\XX)
        \subset
    [0,1]^p
    \label{eq_normalized_dataset_definition}
    ;
\end{equation}
\end{assumption}
\begin{example}[Normalizability in the Euclidean Setting]\label{ex_normalizability_Euclidean_setting}
If $\xxx=\rrp$ and $\yyy=\rrm$ then, $\mathcal{U}_f(x)=\infty$ and $\text{Exp}_{\rrp,x}(y)=y+x$; thus, condition~\eqref{eq_normalized_dataset_definition} reduces to 
$
    \{z-x\}_{z \in \XX} \subseteq [0,1]^p
.
$
\end{example}
%
%

\subsubsection{Breaking the Curse of Dimensionality on Efficient Datasets}\label{sss_pw_lin_active}
Our result focuses on \textit{piecewise linear} activation functions.  By a piecewise linear activation function, we mean a $\sigma \in C(\rr)$ for which there exists $B\in \nn$ and distinct $x_1,\dots,x_B\in \rr$ satisfying: every $x \in \rr-\{x_b\}_{b=1}^B$ is contained in an open interval in which $\sigma$ is linear and there is no such interval for each $x_b$ (for $b=1,\dots,B$).  Note, if $\sigma$ is piecewise linear and non-affine then $B\geq 1$.  These include the ReLU activation function of \cite{fukushima1969visualReLU}, the leaky-ReLU activation function of \cite{maas2013rectifierLeakyReLU}, the pReLU activation function of \cite{he2015delvingPReLU}, commonly implemented piecewise linear approximations to the Heavyside function (implemented for example in \cite{tensorflow2015whitepaper} and in \cite{Theano2016}), and many others.  

Let $g \in \NN[p,m]$ have representation
$g=W_J\circ \sigma \bullet \dots \circ \sigma\bullet W_1$ and $W_j(x) \triangleq A_jx +b_j$ (where $A_j$ is a $p_{j+1}\times p_j$ matrix, $b_j\in \mathbb{R}^{p_j}$, $p_1=p$, and $p_J=m$) for some $J\in \nn_+$ with $J>1$.  Following \cite{Florian2_2021}, the total number of \textit{trainable parameters} in this representation of $g$ is defined by:
$$
\sum_{j=1}^J p_j(p_{j-1}+1)
.
$$

\begin{theorem}[Polynomial Approximation Rates On Efficient Datasets]\label{thrm_efficient_rates}
Fix $n\in \nn_+$, let $f:\xxx\to \yyy$, $\sigma$ be a non-affine piecewise linear activation function, and let $\XX
$ be an $f$-normalizable and $n$-efficient dataset for $f$.
Then, for each $\epsilon>0$, there is a $W\in \nn_+$, a $g\in \NN[p,m:W]$, 
and a constant $\kappa>0$ (not depending on $\epsilon$, $p$, or on $m$), 
such that the GDN: 
    $\hat{f}\triangleq \operatorname{Exp}_{\yyy,f(x)}\circ g\circ \operatorname{Exp}_{\xxx,x}^{-1},
    \label{eq_non_eucl_NNs_def}
$
satisfies the uniform estimate:
\begin{equation}
    \sup_{x \in \XX}\, 
        d_{\yyy}\left(
                f(x)
                    ,
                \hat{f}(x)
                \right)
\leq 
    \kappa
    m^{\frac1{2}}
    \epsilon
    \label{eq_thrm_efficient_rates_rates}
    .
\end{equation}
Moreover, $g$ satisfies the following sub-exponential complexity estimates:
\begin{enumerate}
\item[(i)] \textbf{Width:} satisfies $m\leq W\leq m(4p+10)$, 
\item[(ii)] \textbf{Depth:} of order $
\mathscr{O}\left(m+m
\epsilon^{
    \frac{2p}{3(np+1)}
        -
    \frac{p}{np+1}
}
\right)
$,
\item[(iii)] \textbf{Number of trainable parameters:} is of order 
$
\mathscr{O}\left(
m(m^2-1)
\epsilon^{-\frac{2p}{3(np+1)}}
\right)
.
$
\end{enumerate}
\end{theorem}
\begin{remark}[Dimension-Free Rates]\label{remark_breaking_curse}
If $\XX$ is $1$-efficient for $f$, then the network $g$ of Theorem~\ref{thrm_efficient_rates} has depth roughly of the order
$
\mathscr{O}\left(m+m
\epsilon^{
    \frac{-1}{3}
}
\right)
$ and it depends on $\approx
\mathscr{O}\left(
m^2
\epsilon^{-\frac{2}{3}}
\right)
$ trainable parameters.  
\end{remark}

\begin{remark}[Discussion: Efficiency Datasets Vs. Target Functions Regularity]\label{remark_why_efficiency_makes_more_sense}
An advantage of our efficient dataset approach to ``non-cursed" approximation rates over the classical approach, which imposes regularity assumptions on the target function, is a practical one.  Namely, given any dataset $\XX$, the Definition~\ref{ass_regularity_conditions} can directly be verified.  However, any additionally assumed regularity of the target function typically \textit{cannot} be verified in practice.  
\end{remark}

\subsection{Applications}\label{ss_applications_Pt1}
We illustrate our theoretical framework developed thus far by establishing the universality of many commonly deployed geometric deep learning models.  
\subsubsection{Hyperbolic Feedforward Networks are Universal}\label{ss_hyperbolic_ff_nets}
Hyperbolic spaces have gained significant recent interest, in geometric deep learning, for their ability to represent complex tree-like structures much more efficiently and faithfully than Euclidean representations.  Examples of such state-of-the-art embeddings include low-dimension representations of complex hierarchical datasets used in \cite{NIPS2017_7213}, efficient representations of complex social networks in \cite{krioukov2010hyperbolic}, tractable representations of large undirected graphs in \cite{munzner1997h3}, and accurate representations of trees in \cite{SalaHyperbolicTradeoffs}.  Accordingly, the hyperbolic feedforward networks of \cite{ganea2018hyperbolic}, and \cite{shimizu2021hyperbolic} have gained significant recent interest due to their ability to process such representations since they have inputs and outputs in (generalized) hyperbolic spaces.  Let us briefly recall these notions before establishing the relevant quantitative universal approximation guarantees.  

The \textit{(generalized) hyperbolic spaces} $\mathbb{D}^n_c$ is the Cartan-Hadamard manifold $\{x \in \rrn:\, c\|x\|^2 <1\}$ whose Riemannian structure induces the distance function:
$$
d_c(x,y) \triangleq \frac{2}{\sqrt{c}}\tanh^{-1}\left(
\sqrt{c}
\left\|
\frac{
	(1-c\|x\|^2)y	-(1 - 2c x^{\top}y + c \|y\|^2)
}{
	1 - 2c x^{\top}y + c^2 \|x\|^2\|y\|^2
}
\right\|
\right)
.
$$
The \textit{hyperbolic feedforward networks} of \cite{ganea2018hyperbolic} are defined via a series of complicated operations; however, as the authors later note \citep[Equation (26)]{ganea2018hyperbolic} every hyperbolic feedforward network $\hat{f}:\mathbb{D}^p_c \rightarrow \mathbb{D}^m_c$ can equivalently be represented by:
\begin{equation}
\hat{f}
    =
\operatorname{Exp}_{\mathbb{D}^p_c,0}
\circ 
f 
\circ
\operatorname{Exp}_{\mathbb{D}^m_c,0}^{-1}
\label{eq_hyperbolic_NNs}
,
\end{equation}
where $f \in \NN[p,m]$.  We note that closed-form expressions for $\operatorname{Exp}_{\mathbb{D}^m_c,0}$ and $\operatorname{Exp}_{\mathbb{D}^m_c,0}^{-1}$ are known (see \citep[Lemma 2]{ganea2018hyperbolic}).  Our framework therefore implies the following universal approximation theorem for hyperbolic feedforward networks, which is a quantitative version of \citep[Corollary 3.16]{kratsios2020non}.  
\begin{corollary}[Hyperbolic Neural Networks are Universal Approximators]\label{ex_HNN_univ}
	Let $\sigma\in C(\rr)$ satisfy the Kidger-Lyons conditions.  Fix $c,\epsilon>0$, $f \in C(\mathbb{D}^p_c,\mathbb{D}^m_c)$, and a non-empty compact subset $K\subseteq \mathbb{D}^p_c$.  Then, there exists a hyperbolic feedforward network $\hat{f}$ satisfying:
	\begin{equation}
	    \sup_{x \in K}\, d_{c}(f(x),\hat{f}(x)) <\epsilon
	    \label{eq_ex_HNN_univ}
	    ,
	\end{equation}
	of width $m+p+2$ and whose depth is recorded in Table~\ref{tab_rates_and_depth}.  
\end{corollary}
We also substantially sharpened the variant of the above rates when the training and testing data belong to an $f$-normalizable and $n$-efficient dataset.  
\begin{corollary}[Hyperbolic Neural Networks are Efficient Universal Approximators]\label{ex_HNN_eff}
\hfill\\ 
Consider the setting of Corollary~\ref{ex_HNN_univ} and let $n\in \nn_+$, $\sigma$ be a non-affine piecewise linear activation function, and let $\XX$ be an $f$-normalizable and $n$-efficient dataset for $f$.  Then, there is a $W\in \nn_+$, a $g\in \NN[p,m:W]$, 
and a constant $\kappa>0$ not depending on $\epsilon$, $p$, or on $m$, such that the hyperbolic feedforward network: $
    \hat{f}\triangleq \operatorname{Exp}_{\mathbb{D}^m_c,f(0)}\circ g\circ \text{Exp}_{\mathbb{D}^p_c,0}^{-1},
$ satisfies the approximation bound:
$$
\sup_{x \in \XX}  d_c\left(
        f(x)
    ,
        \hat{f}(x)
\right) 
    <
\kappa \sqrt{m}\epsilon
.
$$
Furthermore, the DNN $g$ satisfies the sub-exponential complexity estimates:
\begin{enumerate}
\item[(i)] \textbf{Width:} satisfies $m\leq W\leq m(4p+10)$, 
\item[(ii)] \textbf{Depth:} of order $
\mathscr{O}\left(m+m
\epsilon^{
    \frac{2p}{3(np+1)}
        -
    \frac{p}{np+1}
}
\right)
$,
\item[(iii)] \textbf{Number of trainable parameters:} is of order 
$
\mathscr{O}\left(
m(m^2-1)
\epsilon^{-\frac{2p}{3(np+1)}}
\right)
.
$
\end{enumerate}
\end{corollary}
\subsubsection{Universal Symmetric Positive-Definite Matrix-Valued Networks}\label{ss_Applications_SPDplus}
Non-degenerate covariance matrices are fundamental tools for describing the non-trivial interdependence of various stochastic phenomena; with notable applications ranging from mathematical finance (\cite{markowitz1991foundations}) to computer vision (see \cite{haralick1996propagating}).  Briefly, any covariance matrix $A$ between $p$ different random variables $\xi_1,\dots,\xi_p$ can be identified (component-wise) with a $p\times p$ vector in the low-dimensional subset $P_p^+\subset \rrflex{p\times p}$ given by:
$$
P_p^+ \triangleq \left\{
A \in \rrflex{p\times p}:\, (\forall x \in \rrp-\{0\})
    \,
x^{\top} Ax >0
\right\}
;
$$
here, we have identified $p\times p$-matrices with vectors in $p^2$ via $(A_{i,j})_{i,j=1}^p \mapsto (A_{1,1},\dots,A_{p,1},\dots,A_{p,p})$.  In fact, $P_p^+$ is a (non-linear) differentiable submanifold of $\rrflex{p\times p}$ (see \cite{pennec2006riemannian}).  

The Euclidean metric $P_p^+$ is not well-suited to the description of covariance matrices.  For example, suppose that $\xi=(\xi_1,\dots,\xi_p)$ and $\zeta=(\zeta_1,\dots,\zeta_p)$ are vectors of features from some dataset of images.  One would expect that, since the content of any image does not change if the image is rotated or shifted, then the relation between the covariance matrices $\text{Cov}(\xi)$ and $\text{Cov}(\zeta)$ should be equal to the distance of $\text{Cov}(X\xi+b)=X^{\top}\text{Cov}(\xi)X$ and $\text{Cov}(X\zeta+b)=X^{\top}\text{Cov}(\zeta) X$; where $X$ is a $p\times p$-orthogonal matrix and $b\in \rrp$ (since $x\mapsto Xx+b$ is exactly a rotation and shift in $\rrp$).  However, this is not the case when comparing covariance matrices dissimilarity with the Euclidean distance.  

In \cite{pennec2006riemannian}, a solution to this problem was obtained via the so-called \textit{``affine-invariant" metric}.  This distance function was obtained by equipping $P_p^+$ with a specific Cartan-Hadamard structure designed to encode invariances under the aforementioned symmetry.  The distance function $d_+$ of this Riemannian metric satisfies $d_+(A,B)=d_+(X^{\top}AX,X^{\top}BX)$ for any $p\times p$-orthogonal matrix $X$ and any $A,B\in P_p^+$ and $d_+$ is computed via:
$$
d_+(A,B)\triangleq 
\left\|
\sqrt{A}
\log\left(
\sqrt{A}^{-1}
B
\sqrt{A}^{-1}
\right)
\sqrt{A}
\right\|_2
;
$$
where $\|\cdot\|_2$ is the Fr\"{o}benius norm on $\rrflex{p\times p}$, $\log(\cdot)$ is the inverse of the matrix exponential $\exp(\cdot)$ and $\sqrt{\cdot}$ is the matrix square-root (both of which are well-defined on $P_p^+$).  
The Riemannian exponential maps is obtained as follows.  Identify $\rrflex{p(p+1)/2}$ with the set of $\text{Sym}_p$ of $p\times p$-symmetric-matrices:
\begin{equation}
    Sym_p: \, \rrflex{p(p+1)/2}\ni (a_{1,1},\dots,a_{1,p},\dots,a_{p,p}) \mapsto 
\begin{pmatrix}
a_{1,1} & \dots & a_{1,p}\\
\vdots & \ddots &  \vdots \\
a_{1,p} & \dots & a_{p,p}
\end{pmatrix} \in \text{Sym}_p
\label{eq_matrix_symmetrization_map}
.
\end{equation}
Under this identification, the Riemannian exponential map and its inverse are computed to be:
\begin{equation}
\begin{aligned}
    \text{Exp}_{P_p^+,A}(B)& = \sqrt{A}\exp\left(\sqrt{A}^{-1}Sym_p(B)\sqrt{A}^{-1}\right)\sqrt{A}\\
    \text{Exp}_{P_p^+,A}(B)^{-1}& = Sym_p^{-1}[\sqrt{A}\log\left(\sqrt{A}^{-1}B\sqrt{A}^{-1}\right)\sqrt{A}]
    .
\end{aligned}
    \label{eq_Riemannian_Exp_Log_SPD}
\end{equation}

The suitability of this geometry to the problem of covariance-matrix feature description is well-studied, especially in the computer vision literature.  Most relevant to our program, in \cite{meyer2011regression} the authors introduce a class of non-Euclidean regression models on $P_p^+$ and in \cite{bonnabel2013stochastic} and \cite{becigneul2018riemannian} classes of optimization algorithms were introduced which leverage the Riemannian geometry of $P_p^+$.  Likewise, there have been numerous optimization software advances specifically designed to handle such situations \cite{JMLRv15boumal14a}, \cite{townsend2016PYMANOPT}, and \cite{JMLR_Geomstats_v21_19_027}.

Subsequently, \cite{baes2019lowrank} and \cite{herrera2020denise} extended some of these ideas by introducing geometric deep learning models with inputs and outputs from the set of $p\times p$-symmetric positive \textit{semi-}definite matrices to itself.  Thereafter, in \cite{kratsios2020non} the authors derived a universal extension of the regression model of \cite{meyer2011regression} which necessarily inputs and outputs matrices from $P_p^+$ and $P_m^+$, respectively.  The latter model class of ``affine-invariant" GDNs have the representation:
\begin{equation}
    \hat{f}\triangleq \text{Exp}_{P_m^+,I_m}\circ g \circ \text{Exp}_{P_p^+,I_p}^{-1} \in C(P_p^+,P_m^+)
    \label{eq_SPD_universal}
    ;
\end{equation}
where, $g \in \NN[\frac{p(p+1)}{2},\frac{m(m+1)}{2}]$.  
The following are, respectively, quantitative and efficient improvements of the universal approximation theorems for the model class~\eqref{eq_SPD_universal} derived in \citep[Section 3.2.1]{kratsios2020non}.  We denote the set of $p\times p$-orthogonal matrices by $O(p)$.  

\begin{corollary}[{Universality of the Affine-Invariant Networks of~\eqref{eq_SPD_universal}}]\label{ex_Aff_inv_GDNs_univ}
	Let $\sigma\in C(\rr)$ satisfy the Kidger-Lyons conditions.  Fix $\epsilon>0$, $f \in C(P_p^+,P_m^+)$, and a non-empty compact subset $K\subseteq P_m^+$.  Then, there exists an ``affine-invariant" GDN $\hat{f}$ satisfying:
	\begin{equation}
	    \max_{A \in K}\sup_{X\in O(p)}\,
	    d_{+}\left(
	        X^{\top}f(A)X
	            ,
	        X^{\top}\hat{f}(A)X
	    \right)
	        <
	    \epsilon
	    \label{eq_ex_Aff_inv_GDNs_univ}
	    .
	\end{equation}
	Moreover, $g$ in the representation~\eqref{eq_SPD_universal}, has width at-most $\frac{p(p+1)+m(m+1)+4}{2}$ and its depth is recorded in Table~\ref{tab_rates_and_depth} with $2^{-1}p(p-1)$ and $2^{-1}m(m-1)$ in place of $p$ and $m$, respectively.   
\end{corollary}
Using the concept of efficient dataset, we are able to refine the above theorem.  
\begin{corollary}[Affine-Invariant GDNs are Efficient Universal Approximators]\label{ex_Aff_inv_GDNs_eff}
\hfill\\
Consider the setting of Corollary~\ref{ex_Aff_inv_GDNs_univ} and let $n\in \nn_+$, $\sigma$ be a non-affine piecewise linear activation function, and let $\XX$ be an $f$-normalizable and $n$-efficient dataset for $f$.  Then, there is a $W\in \nn_+$, a $g\in \NN[2^{-1}p(p+1),2^{-1}m(m+1):W]$, 
and a constant $\kappa>0$ not depending on $\epsilon$, $p$, or on $m$, such that the GDN $
\hat{f}\triangleq \operatorname{Exp}_{P_m^+,f(0)}\circ g\circ \operatorname{Exp}_{P_p^+,0}^{-1},
$ satisfies the approximation bound:
$$
\sup_{x \in K} 
\sup_{X\in O_,}
d_+\left(
        X^{\top}
            f(x)
        X
    ,
        X^{\top}
            \hat{f}(x)
        X
\right) 
    <
\frac{\kappa \sqrt{m(m+1)}}{\sqrt{2}} 
    \epsilon
.
$$
Furthermore, $g$ satisfies the sub-exponential complexity estimates:
\begin{enumerate}
\item[(i)] \textbf{Width:} satisfies $m\leq W\leq 2^{-1}m(m+1)(2p(p+1)+10)$, 
\item[(ii)] \textbf{Depth:} of order $
\mathscr{O}\left(
m(m+1)(1+
\epsilon^{
    \frac{p(p+1)}{3(2^{-1}np(p+1)+1)}
        -
    \frac{p(p+1}{np(p+1)+2}
}
)
\right)
$,
\item[(iii)] \textbf{Number of trainable parameters:} of order 
$
\mathscr{O}\left(
m(m+1)((m(m+1))^2-1)
\epsilon^{-
\frac{2p(p+1)}{3(np^2 + np +2)}
}
\right)
.
$
\end{enumerate}
\end{corollary}
\subsubsection{Spherical Neural Networks and Approximation in Kendall's Pre-Shape Space}\label{ss_UAT_Sphereical}
Our last illustration focuses on Theorem~\ref{thrm_negative_motiation}.  
As described in \cite{pmlr-v38-straub15}, spherical data plays a central role in many computer vision applications as a natural medium for describing direction data.  This, and its connections to various other areas such as geo-statistics, has made learning from spherical data an active area of research both in the machine learning (see \cite{pmlr-v119-dutordoir20a}, and \cite{JMLR:v8:hamsici07a}), and in the statistics communities (see \cite{MR3852654}).  

Geodesics on the sphere are well-studied; for example, the distance on $S^p$ is $d_{S^p}(x,y)=\arccos{(y^{\top}x)}.$
Most importantly for our analysis, the Riemannian exponential map, and its inverse, at any $x \in S^p$ admits the following closed-form expressions 
\begin{equation}
\operatorname{Exp}_{S^p,x}(v) 
=
\cos(\|v\|)x + \sin(\|v\|)\frac{v}{\|v\|}
\mbox{ and }
\operatorname{Exp}_{S^p,x}(y)^{-1}
=
\frac{y - (y^{\top}x)x}{\|y - (y^{\top}x)x\|}\arccos{(y^{\top}x)}
,
\label{eq_sphere_Exp_Log}
\end{equation}
 (see \citep[page 3341]{MR3852654} for example).  Unlike the geometries in the two previous examples, the sphere is positively curved, with sectional curvature always equal to $1$.  Consequentially the Riemannian Exponential map's inverse, about any point $x \in S^p$, is not globally defined.  

\begin{corollary}[Local Quantitative Deep Universal Approximation for Spherical Data]\label{cor_NE_UAT_Sphereical}
Let $\sigma\in C(\rr)$ satisfy Assumption~\ref{ass_Kidger_Lyons_Condition}.  For any continuous function $f:S^p\rightarrow S^m$, any $\epsilon>0$, given any $B_{S^p}(x,\delta)\subseteq S^p$ for which $0<\delta < \pi$ then, for every $g\in \NN[p,m,p+m+2]$ the GDN
$
    \hat{f}\triangleq \operatorname{Exp}_{S^m,f(x)}\circ g\circ \operatorname{Exp}_{S^p,x}^{-1},
    \label{eq_non_eucl_NNs_def_sphere}
$ 
is well-defined on $\overline{B_{S^p}(x,\delta)}$ and there is one such $\hat{f}$ satisfying the approximation bound:
$$
{
\underset{x \in \overline{
B_{S^p}\left(
x,\delta
\right)
}}{\max}
\,
d_{S^m}\left(
f(x),\hat{f}(x)
\right)\leq \epsilon
.
}
$$  
Moreover, $g$'s depth is recorded in Table~\ref{tab_rates_and_depth}.  
\end{corollary}
\begin{remark}\label{remark_on_pi}
The quantity $\pi$ estimating the maximum radius of the ball $B_{S^p}(x,\delta)$ in Corollary~\ref{cor_NE_UAT_Sphereical} is a lower-bound for $\mathcal{U}_f(x)$.  This estimate is specific to the sphere's geometry, where we have tightened the generic estimate of~\eqref{eq_generic_lower_bound_Gromov_Application} via \cite{klingenberg1991simple}'s Quarter-Pinched Sphere Theorem.  
\end{remark}
The obstruction identified in Theorem~\ref{thrm_negative_motiation} has the following consequence for spherical spaces.  
\begin{corollary}[Universal Approximation on Spheres is Local]\label{cor_spheres_obstruction}
If $\xxx=\yyy=S^m$, then there exists a non-empty compact subset $\kkk\subseteq S^m$, $x \in \kkk$, $\epsilon>0$, such that for every
$k \in \nn$, and for every 
$g \in \NN[m,m;k]$, the map $
\operatorname{Exp}_{S^m,f(x)}\circ g \circ \operatorname{Exp}_{S^m,x}^{-1}
$ is a well-defined function in $C(\kkk,\yyy)$, but
$$
{
    \underset{
\underset{
y \in S^m,\, k \in \nn_+
}
{
g \in \NN[m,m,k]
}
}{\inf}\,
\sup_{z \in \kkk}\,
d_{\yyy}\left(
\operatorname{Exp}_{S^m,f(x)}\circ g \circ \operatorname{Exp}_{S^m,x}^{-1}(z)
    ,
1_{S^m}(z)
\right)
    \geq 
\epsilon
.
}
\vspace{.75em}
$$
\end{corollary}

\begin{remark}[Universal Approximation Theorem in Kendall's Pre-Shape Space]\label{remark_pre_shape_space}
As \hfill\\
shown in \cite{KendallOG1984} and \cite{KendallLe1993AnnalsofStats}, high-dimensional spheres coincide with Kendall's pre-shape space.  Therefore, Corollary~\ref{cor_spheres_obstruction} guarantees that GDNs between Kendall's pre-shape spaces are universal.  These GDNs are a direct ``deep learning" extension of the Procrustean (pre-shape space) regressors of \cite{Fletcher2013}.  
\end{remark}


\section{Main Results on Building Universal GDL Models using GDNs}\label{s_Results_pt2}
Next, we treat GDNs as elementary building blocks, and we derive several results which describe how to combine GDNs to build universal \textit{geometric deep learning models} compatible with complicated geometries; summarized in Figure~\ref{fig_full_computational_graph}.  This additional flexibility is gained by combining multiple GDNs using \textit{``geometric processing layers"} (symbolized arrows in Figure~\ref{fig_full_computational_graph} other than {\color{MidnightBlueComplementingRed}{$\hat{f}_i$}}).  These are non-trainable layers that encode specific geometric ``features" into our geometric deep learning models, such as products, quotients, parameterization, or boundary-like regions.  

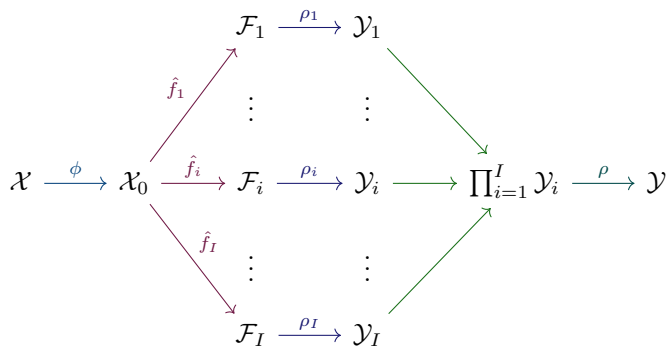
\begin{figure}[H]
    \centering
    \begin{tikzcd}[row sep=small]
            & &\mathcal{F}_1  \arrow[r,MidnightBlue,"\rho_1"] &\yyy_1 \arrow[MidnightBlueComplementingGreen,ddr] &\\
            & & \vdots & \vdots & \\
    \xxx \arrow[r,darkcerulean,"\phi"] 
        &\xxx_0
        \arrow[uur,MidnightBlueComplementingRed,"\hat{f}_1"]  
        \arrow[r,MidnightBlueComplementingRed,"\hat{f}_i"]
        \arrow[ddr,MidnightBlueComplementingRed,"\hat{f}_I"]
    &\mathcal{F}_i  \arrow[r,MidnightBlue,"\rho_i"] &\yyy_i \arrow[MidnightBlueComplementingGreen,r] &\prod_{i=1}^I \yyy_i \arrow[r,deepjunglegreen,"\rho"] &\yyy\\
    & & \vdots & \vdots & \\
    & &\mathcal{F}_I  \arrow[r,MidnightBlue,"\rho_I"] &\yyy_I \arrow[MidnightBlueComplementingGreen,uur] & \\
    \end{tikzcd}
    \caption{The geometric deep learning model's full computational graph.}
    \label{fig_full_computational_graph}
\end{figure}

We briefly outline Figure~\ref{fig_full_computational_graph}: $\phi$ is a \textit{feature map} with the UAP-invariance property of \cite{kratsios2021neu}, each of the {\color{MidnightBlueComplementingRed}{$\hat{f}_i$}} are GDNs mapping into ``deep feature spaces $\mathcal{F}_i$", the {\color{MidnightBlue}{$\rho_i$}} are ``good" quotient maps which impose symmetries on the deep features in $\mathcal{F}_i$, the {\color{MidnightBlueComplementingGreen}{green arrows}} are a parallelization of the architectures thus far via a ``skip connection" (analogously to \cite{he2016ResNets} and \cite{srivastava2015HighwayNets}), and {\color{deepjunglegreen}{$\rho$}} parameterizes the output space $\yyy$ up to a \textit{``negligible subset of $\yyy$"} (where our notion of negotiability is similar to that of \cite{TorunczykZsets} and to \cite{FavInfiniteDimensionalTopologyBookVanMills2001}).  

We progressively introduce each geometric processing layer in Figure~\ref{fig_full_computational_graph} and incrementally derive its universal approximation theorem.  Each step of our derivative will correspond to a geometric deep learning model defined by a computational sub-graph of Figure~\ref{fig_full_computational_graph}.  
\subsection{Feature Spaces and Quotient Layers: For Quotient Geometries}\label{ss_NSE_sss_Quotient}
Often, a metric space $\yyy$'s geometry is extremely complicated, but its description can substantially be simplified by understanding its points as equivalence classes of symmetries defined on a ``simpler" $m$-dimensional Riemannian manifold $\mathcal{F}$.  We consider the situation of Figure~\ref{fig_factorization}.  
\begin{figure}[H]
    \centering
        \[
    \begin{tikzcd}[column sep=small, row sep=small] 
    & \mathcal{F}  \arrow[MidnightBlue]{ddr}{\rho} &\\
    &  &\\[-2ex]
    \mathcal{X} 
    \arrow[MidnightBlueComplementingRed]{uur}{\tilde{f}}
    \arrow{rr}[swap]{f} & & \mathcal{Y} \\
    \end{tikzcd}
        \]
    \caption{Factorizing the Target Function.}
    \label{fig_factorization}
\end{figure}
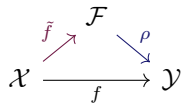

Our setting is formalized as follows.  Let $G\subset C(\zzz,\zzz)$ be a set of surjective \textit{isometries}; that is, each $g \in G$ does not distort the relative distance between any two points $z_1,z_2\in \zzz$ since $d_{\zzz}(z_1,z_2)=d_{\zzz}(g(z_1),g(z_2))$.  We require the isometries in $G$ to be ``compatible" in the sense that:
\begin{assumption}[{\color{black}{Symmetric Space}}]
\label{ass_group}
$1_{\zzz}\in G$ and if $g_1,g_2\in G$ then\footnote{We note that $g_2^{-1}$ is always well-defined since every isometry is injective; thus, $g_2$ is a bijection and therefore it has a unique two-sided inverse $g_2^{-1}$.  } $g_1\circ g_2^{-1}\in G$.  
\end{assumption}
We {\color{black}{consider output spaces $\yyy$ which are}} ``invariant/symmetric to the isometries in $G$".  {\color{black}{Such $\yyy$ are called \textit{symmetric spaces} and are widely studied both in the context of density estimation when data and/or parameters lie in a low-dimensional manifold in \cite{LiLuChevallierDunson_2020_DensityEstimation_SymmetricSpaces}, non-linear dimension reduction \cite{FletcherConglinPizerSarang_2004_PGASymmetricSpaces}, learning faithful graph representations in \cite{LopezPozzettiTrettelStrubeWiengard_2021_SymmetricSpacesGraph_Embeddings}, stochastic filtering in \cite{PontierSzpirglas_1986_FilteringRiemannianSymmetricSpaces}, as well as several other instances in machine learning literature and its adjacent research areas.}}

Following \citep[Section 3.3]{Burago2CourseonMetricGeometry2001}, we do this by setting (resp. identifying) the points in $\yyy$ to be (resp. with) the equivalence classes:
 $
[z]\triangleq \left\{
z' :\, (\exists g \in G)\, g(z)=z'
\right\}.
$  
By \citep[Lemma 3.3.6]{Burago2CourseonMetricGeometry2001}, the set $\yyy$ is made into a \textit{metric space} since the map:
$$
\bar{d}_{\zzz}([z_1],[z_2]) 
    = 
\inf_{g \in G}\, 
    \bar{d}_{\zzz}(z_1,g(z_2)) 
,
$$
is a well-defined metric on $\yyy$.  Note that, $\bar{d}_{\zzz}([z_1],[z_2])\leq d_{\zzz}(z_1,z_2)$ for any $z_i\in [z_i]$ and $i=1,2$.   We call $\bar{d}_{\zzz}$ the \textit{quotient metric} on $\yyy$ and $\yyy$ the \textit{quotient metric space of $\zzz$ generated by the symmetries in $G$}.  If $\zzz$'s geometry is compatible with the symmetries described by $G$ (Assumption~\ref{ass_properly_discontinuous_group_action} below), then the \textit{projection map}:
\begin{equation}
    \rho_1:\, \zzz\ni z \mapsto [z]\in \yyy
     \label{eq_quotient_map}
    ,
\end{equation}
implies that: ``locally, $\zzz$ looks like a disjoint union of identical pieces of $\yyy$ and that it looks the same everywhere".  Following \citep[Proposition 3.4.15.]{Burago2CourseonMetricGeometry2001}), this happens when:
\begin{assumption}[{Compatibility between $G$ and $\zzz$}]\label{ass_properly_discontinuous_group_action}
\hfill
\begin{enumerate}
\item[(i)] For each $z\in \zzz$ there is a $k_z>0$ such that if $d(z,g(z))<k_z$ then $g$ is the identity,
\item[(ii)] For each $z\in \zzz$ and each $g\in G$ if $g\neq 1_{\mathcal{F}}$ then $z\neq g(z)$.  
\end{enumerate}
\end{assumption}
\begin{example}[{\citep[Exercise 1.3.23]{HatcherAlgebraicTopology}}]\label{remark_finite_Group_actions}
If $G$ is finite, then Assumption~\ref{ass_properly_discontinuous_group_action} (ii) holds.  
\end{example}
The deep feature space must be connected by paths and every such path is topologically comparable.  
\begin{assumption}[The Deep Feature Space is Simply Connected]
\label{ass_simply_connectedness}
The deep feature space $\fff$ is connected and simply connected.  
\end{assumption}
\begin{example}\label{ex_simplyconnected_geometries}
The Euclidean space $\rrp$, {\color{black}{the space of symmetric positive definite matrices, and the hyperbolic space each}} satisfy Assumption~\ref{ass_simply_connectedness}.  Nevertheless, if $p\geq 1$, then Examples~\ref{ex_Euclidean_is_simple} and~\ref{ex_massive_difference} shows that $S^p$ is still much more complicated than $\rrp$, topologically.  
\end{example}
We arrive at the following quantitative non-Euclidean universal approximation theorem.
%
\begin{theorem}[Universal Approximation for Quotient Metric Spaces]\label{thrm_Quotient_UAT}
Suppose that Assumptions~\ref{ass_group} and~\ref{ass_properly_discontinuous_group_action} hold, and that Assumption~\ref{ass_simply_connectedness} also holds for $\xxx$ and for $\zzz$ and let $\sigma\in C(\rr)$ satisfy Assumption~\ref{ass_Kidger_Lyons_Condition}.  Let $\epsilon>0$, $x \in \xxx$, $f \in C(\xxx,\yyy)$, $\XX\subset B_{\xxx}(x,\mathcal{U}_f(x))$ be a compact dataset, and $\rho_1:\mathcal{F}\ni z \mapsto [z]\in \yyy$.  There is a GDN $\hat{f}$ with representation: $
\hat{f} = \text{Exp}_{\mathcal{F},f}\circ g\circ \text{Exp}_{\xxx,x}^{-1},
$ for some $y \in \mathcal{F}$, where $g \in \NN[p,m,W]$ such that:
\begin{equation}
    \sup_{x\in \XX}\, 
    \bar{d}_{\zzz}\left(f(x),\rho_1\circ \hat{f}(x)\right) 
<\epsilon
\label{eq_thrm_Quotient_UAT_rate}
.
\end{equation}
Moreover, 
$W=p+m+2$ and $g$'s depth is as in Table~\ref{tab_rates_and_depth}.
\end{theorem}

\subsubsection*{Examples: Universal Approximators with computational {Subgraph of Figure~\ref{fig_factorization}}}
The following example is an essential component of the projective shape space introduced by \cite{RealProjectiveShapeSpaces2005AnnalsOfStats}.  We return to the following manifold in Section~\ref{ss_Applications_II_sophisiticated_models}.  
\begin{example}[{Universal Approximators to the Real Projective Space ($\rr P^m$)}]
\label{ex_shape_projective}
An element of the real projective space $\rr P^m$ is a line in $\rrflex{m+1}$ passing through the origin.  Since every such line is determined by its intersections with $S^m\triangleq \left\{x \in \rrflex{m+1}:\, \|x\|=1\right\}$ then elements of $\rr P^m$ are:
$$
[x]\triangleq \left\{\{x,-x\}:\, x \in S^m\right\}
;
$$
where, $G=\{1_{S^m}, [x\mapsto -x]\}$.  Furthermore, in \citep[Example 1.43]{HatcherAlgebraicTopology}, it is shown that $G$ and $S^m$ satisfy Assumptions~\ref{ass_group} and~\ref{ass_properly_discontinuous_group_action}.  By Example~\ref{ex_simplyconnected_geometries}, $\rrp$ and $S^m$ satisfy Assumption~\ref{ass_simply_connectedness}.  Thus, the projection map $\rho_1:S^m\ni y \mapsto [y] \in \rr P^m$ verified the conditions of Theorem~\ref{thrm_Quotient_UAT} and $\rr P^m$ is a quotient of $S^m$ by the \textit{symmetry} defined by $G$. 
Since Corollary~\ref{cor_spheres_obstruction} implies that deep neural models 
$$
\{\operatorname{Exp}_{S^m,f(x)} \circ g:\,g \in \NN[p,m]\}
,
$$ 
are locally universal in $C(\rrp,S^m)$, then Theorem~\ref{thrm_Quotient_UAT} implies that each $f \in C(\rrp,\rr P^m)$ can be locally be approximated by a deep neural model of the form $
\rho_1 \circ \operatorname{Exp}_{S^m,f(x)} \circ g
.
$
\end{example}
Our next illustration of Theorem~\ref{thrm_Quotient_UAT} concerns universal approximators into the ``flat torus".  Examples of the torus geometry in data visualization in \cite{FlatTorusCitation2} and in \cite{FlatTorusCitation1}.
\begin{example}[{Universal Approximators on the Flat Torus ($\mathbb{T}^m$)}]\label{ex_torsu}
Let $G$ be the ``integer lattice translations": $
G \triangleq \left\{g:\rrm\ni z\mapsto z+k\in \rrm\right\}_{k \in \zz^m}.
$
Then $G$ satisfies Assumption~\ref{ass_group} and $\rrm$ satisfies Assumption~\ref{ass_properly_discontinuous_group_action}.  Example~\ref{ex_simplyconnected_geometries} states that $\rrp$ and $\rrm$ satisfy Assumption~\ref{ass_simply_connectedness}.  
The classes in $\yyy$ are therefore in correspondence with points in the cube $[0,1]^m$ but the distance between any $y_1,y_2\in \yyy$ is the ``flat toral distance":
$$
\bar{d}_{\rrm}(y_1,y_2) = \inf_{z \in \zz^m} \, \|y_1-(y_2 +z)\|.
$$
In this space, we are allowed to ``teleport" along nodes in integer lattice $\zz^m$ but every other movement counts.  It is a standard exercise to show that the above $G$ satisfies Assumptions~\ref{ass_group} and Assumption~\ref{ass_properly_discontinuous_group_action}.  Thus, Theorem~\ref{thrm_Quotient_UAT} implies that for every $f \in C(\rrp,\mathbb{T}^m)$ there is a DNN $\hat{f}\in \NN[p,m]$ such that $\rho_1\circ \hat{f}$ locally approximates $f$.
\end{example}
\subsection{Skip Connections and Parallelization: For Product Geometries}\label{ss_SK_Product_Geometries}
In \cite{gribonval2019approximation}, the authors describe a calculus for ``parallelizing" several feedforward networks $g_1,\dots,g_I \in C(\rrp,\rr)$ to efficiently form a deep neural model in $C(\rrp,\rrflex{I})$.  There, the parallelized model implements the map:
\begin{equation}
    \rrp \ni x \mapsto (g_1(x),\dots,g_I(x)) \in \rrflex{I}
    \label{eq_parallelization}
    .
\end{equation}
In \cite{Florian2_2021}, the author gave conditions on the activation function $\sigma$ under which the map~\eqref{eq_parallelization} could be implemented by a single feedforward network.  Otherwise, the map of~\eqref{eq_parallelization} are $I$ different learning models defined by a more complicated computational graph where the last layer can be understood as a sort of ``skip connection".  Note that, most commonly used deep learning software such as, \cite{tensorflow2015whitepaper} and \cite{Theano2016}, are designed to handle these types of computational graphs. 

\begin{figure}[H]
    \centering
    \begin{tikzcd}[row sep=small]
    &\mathcal{F}_1  \arrow[r,MidnightBlue,"\rho_1"] &\yyy_1 \arrow[MidnightBlueComplementingGreen,ddr] &\\
    & \vdots & \vdots & \\
    \xxx_0
        \arrow[uur,MidnightBlueComplementingRed,"\hat{f}_1"]  
        \arrow[r,MidnightBlueComplementingRed,"\hat{f}_i"]
        \arrow[ddr,MidnightBlueComplementingRed,"\hat{f}_I"]
    &\mathcal{F}_i  \arrow[r,MidnightBlue,"\rho_i"] &\yyy_i \arrow[MidnightBlueComplementingGreen,r] &\prod_{i=1}^I \yyy_i \\
    & \vdots & \vdots & \\
    &\mathcal{F}_I  \arrow[r,MidnightBlue,"\rho_I"] &\yyy_I \arrow[MidnightBlueComplementingGreen,uur] & \\
    \end{tikzcd}
    \caption{Parallelized Computational Graph.}
    \label{fig_parallelization}
\end{figure}
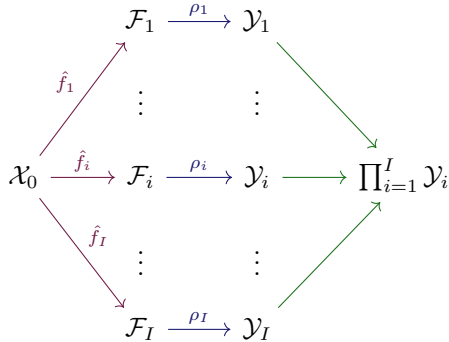

In the geometric deep learning situation, \textit{parallelization} is even more interesting since it allows us to simultaneously generated predictions on potentially very different output spaces.  Building on the ideas of Section~\ref{ss_NSE_sss_Quotient}, let $\yyy_1,\dots,\yyy_I$ be metric spaces and suppose that there are deep feature spaces $\mathcal{F}_1,\dots,\mathcal{F}_I$ such that each satisfies Assumptions~\ref{ass_group} and~\ref{ass_properly_discontinuous_group_action} (with $\mathcal{F}$ and $\yyy$ respectively replaced by $\mathcal{F}_i$ and $\yyy_i$).  Extending~\eqref{eq_parallelization}, we consider the problem of approximating functions in $C(\xxx,\prod_{i=1}^I \yyy_i)$ where the \textit{product} is defined by $
\prod_{i=1}^I \yyy_i \triangleq \left\{
(y_1,\dots,y_I):\, y_i \in \yyy_i
\right\}.
$
As usual, we equip $\prod_{i=1}^I \yyy_i$ with the \textit{product-metric} defined for $(y_1,\dots,y_I),(\tilde{y}_1,\dots,\tilde{y}_I) \in \prod_{i=1}^I \yyy_i$ by:
$$
d_{\prod_i\yyy_i}((y_1,\dots,y_I),(\tilde{y}_1,\dots,\tilde{y}_I)) \triangleq \max_{i=1,\dots,I}\, \{d_{\yyy_i}(y_i,\tilde{y}_i)\};
$$
where $d_{\yyy_i}$ denotes the metric on $\yyy_i$ for $i=1,\dots,I$.   
\begin{remark}[Notation]\label{rem_notation}
The dimension of each $\mathcal{F}_i$ is denoted by $m_i$ and $f_i\in C(\xxx,\yyy_i)$ denotes the composition $f \in C(\xxx,\yyy)$ and the canonical projection $\yyy\ni (y_1,\dots,y_I)\mapsto y_i \in \yyy_i$.  We also use $\rho_k:\mathcal{F}_k\ni z\mapsto [z]\in \yyy_k$ to denote the projection maps discussed in the previous section.  
\end{remark}

\begin{corollary}[Universality of Parallelized GDNs]\label{cor_parallelization}
Suppose that $\yyy_i$ and $\mathcal{F}_i$ satisfy Assumptions~\ref{ass_group},~\ref{ass_properly_discontinuous_group_action}, and~\ref{ass_simply_connectedness} (mutatis mutandis).  Let $\sigma\in C(\rr)$ satisfy Assumption~\ref{ass_Kidger_Lyons_Condition}, $f \in C(\xxx,\prod_{i=1}^I \yyy_i)$, and fix $\epsilon>0$.  Fix: $
0< \delta <\min_{i=1,\dots,I}\{\mathcal{U}_{f_i}(x)\}.
$
Then, for each $x \in \xxx$ and each compact dataset $\XX\subseteq B_{\xxx}(x,\delta)$ there exist $g_i\in \NN[p,m_i]$ (for $i=1,\dots,I$) such that:
\begin{equation}
    \hat{f}_i\triangleq \text{Exp}_{\yyy_i,f_i(x)}\circ g_i\circ \text{Exp}_{\xxx,x}^{-1}
    \label{cor_parallelization_paralleized_architecture}
    ,
\end{equation}
satisfy the estimate:
\begin{equation}
    \sup_{x \in \XX}\, 
d_{\prod_i\yyy_i}\left(
    f(x)
        ,
    (\rho_1\circ \hat{f}_1(x),\dots,\rho_I\circ \hat{f}_I(x))
\right) <\epsilon
\label{eq_cor_parallelization_estimate}
.
\end{equation}
Moreover, the complexity of each $\hat{f}_i$ depends on $\XX$'s geometry as follows:
\begin{enumerate}
    \item[(i)] \textbf{Efficient Case:} If there is an $n\in \nn_+$ such that $\XX$ is $f_i$-normalized, $n$-efficient for $f_i$, $\rho_i=1_{\mathcal{F}_i}$, and if $\sigma$ is piecewise linear then: 
    \begin{enumerate}
    \item[(i)] \textbf{Width:} satisfies $m_i\leq W\leq m(4p+10)$, 
    \item[(ii)] \textbf{Depth:} of order $
    \mathscr{O}\left(m_i+m_i
    \epsilon^{
        \frac{2p}{3(np+1)}
            -
        \frac{p}{np+1}
    }
    \right)
    $,
    \item[(iii)] \textbf{Number of trainable parameters:} is of order 
    $
    \mathscr{O}\left(
    m_i(m_i^2-1)
    \epsilon^{-\frac{2p}{3(np+1)}}
    \right)
    .
    $
\end{enumerate}
    Moreover, in this setting, the right-hand side of~\eqref{eq_cor_parallelization_estimate} is instead $\kappa \epsilon\sum_{i=1}^I\sqrt{m_i}$; where $\kappa>0$ is a constant not depending on $p$, $m_i$, or on $\epsilon$.
    \item[(ii)] \textbf{General Case:} If $\XX$ is not efficient for $f_i$, $\rho_i\neq 1_{\mathcal{F}_i}$, or $\XX$ is not $f_i$-normalized, then $W=p+m_i+2$ and each $g_i$ has depth as in Table~\ref{tab_rates_and_depth} (but with $m_i$ in place of $m$).
\end{enumerate}
\end{corollary}
\begin{remark}\label{remark_product_simplification}
In Corollary~\ref{cor_parallelization} (i), $\rho_i=1_{\mathcal{F}_i}$ implies that $\mathcal{F}_i=\yyy_i$ for each $i=1,\dots,I$.  
\end{remark}


\subsection{UAP-Preserving Layers: For Embedded Geometries and Parameterization}\label{ss_Main_global}
We require that the ``feature map" $\phi$ has the UAP-invariance property, which means that pre-composing the learning model by $\phi$ does not negatively impact the learning model's universal approximation property (UAP).  The following condition is sufficient and the condition is known to be sharp in a broad range of cases (see \cite{kratsios2020non}).  
\begin{assumption}[UAP-Invariant Feature Map]\label{assumptionPhi} 
$\phi: \xxx \rightarrow \xxx_0$ is continuous and injective.
\end{assumption}
\begin{example}[UAP-Invariant Feature Maps When $\xxx$ is Embedded in $\rrp$]\label{ex_UAP_building}
If $\xxx$ is compact and embedded in $\rrp$, then the reconfiguration networks of \cite{kratsios2021neu} and the injective ReLU networks of \cite{IVAN2020InjectiveDNNs} are both classes of UAP-invariant feature maps.  
\end{example}
Dually, UAP-Invariant Readout maps allow us to extend any universal approximation result to any space which is \textit{``almost parameterized by $\rho$"}.  Assumption~\ref{ass_boundary_homotopy} below, is the dual form of the UAP-invariant feature condition, above.  It both extends and significantly simplifies the condition of \citep[Assumption 3.2]{kratsios2020non}.  The key point is to reinterpret \cite{TorunczykZsets}'s ``homotopy negligible sets" and the $\mathcal{Z}$-sets of \citep[Section 5]{FavInfiniteDimensionalTopologyBookVanMills2001}.
\begin{assumption}[UAP-Invariant Readout Map]
\label{ass_boundary_homotopy}
A map $\rho:\prod_{i=1}^I \yyy_i \rightarrow \yyy$ is said to be a \textit{UAP-invariant readout map} if:
\begin{enumerate}
    \item[(i)] $\rho$ is continuous and admits a continuous right-inverse $R$ on its image $\im{\rho}\subseteq \yyy$,
    \item[(ii)] There is a (homotopy) $H\in C([0,1]\times \yyy,\yyy)$ satisfying:
    \begin{enumerate}
        \item[(a)] For each $0\leq t<1$ and every $y\in \yyy$ we have: $H_t(y)\in \im{\rho}$,
        \item[(b)] For every $\epsilon>0$, there is a $t_{\epsilon}\in [0,1)$ satisfying: 
        $$
        \sup_{y \in 
        \yyy
        }
        d_{\yyy}(H_{t_{\epsilon}}(y),y)<\epsilon.
        $$
    \end{enumerate}
\end{enumerate}
\end{assumption}
{\color{black}{
The simplest non-trivial instance of an interesting class of UAP-invariant readout maps arises from projecting the output of a GDL model taking values in a Euclidean space onto a non-empty closed and convex subset thereof.  Furthermore, this class trivially satisfies Assumption~\ref{ass_boundary_homotopy} (ii).  
\begin{example}[Projections onto Closed Convex Sets are UAP-Invariant Readout Maps]
\label{ex_closed_convex_UAP}
Let $\yyy\subseteq \rr^p$ be non-empty, closed, and convex.  By \citep[Theorem 3.16]{BauschkeCombetter_ConvexAnalysisMonotoneOpteratorHilbertSpaces_CMS_2017} the metric projection $P_{\yyy}$ onto $\yyy$ defined by:
\[
P_{\yyy}(x)
    :=
\underset{\tilde{y}\in \yyy}{\operatorname{argmin}}\,
    \|y-\tilde{y}\|,
\]
is a well-defined, $1$-Lipschitz surjection of $\rr^p$ onto $\yyy$.  A direct computation confirms that the inclusion map $R:\yyy\ni z \mapsto z\in \rr^p$ is a continuous right-inverse for $P_{\yyy}$; thus, Assumption~\ref{ass_boundary_homotopy} (i) is satisfied.  
Since $\yyy-\im{P_{\yyy}}=\emptyset$ then, we may take $H_t(y):=y$ to be the homotopy in Assumption~\ref{ass_boundary_homotopy} (ii).  Hence, the metric projection $P_{\yyy}$ is a UAP-invariant readout map if $\yyy$ is a non-empty, closed and convex set.
\end{example}

In Example~\ref{ex_closed_convex_UAP}, the right-inverse of the projection map $\pi_{\yyy}$ is never a continuous two-sided inverse, i.e. is a homeomorphism, if $\yyy$ is bounded%
\footnote{Since this would lead to a contraction of the non-compactness of $\rr^p$.}.  %
In particular, it never has a smooth two-sided inverse on its image, which can of-course be advantageous while training.  

Nevertheless, it can be preferable to instead map $\rr^p$ homeomorphically onto $\yyy$'s interior, provided that $\yyy$ has non-empty interior, and simply disregard $\yyy$'s boundary.  It turns out that this is possible by appealing to $\yyy$'s \textit{gauge}, also called $\yyy$'s \textit{Minkowski functional}, as is outlined by the next example.  

\begin{example}[UAP-Invariant Readouts on Convex Bodies via Gauges]
\label{ex_Minkwoski_homeomorphism}
Let $\yyy$ be a convex subset of $\rr^p$ with non-empty interior containing $0$.
Following \cite{KrieglMichor_1997_ConvenientGlobalAnalysis}, a gauge $\mu_{\yyy}$ of a bounded convex set $\yyy$ centered containing $0$ is the real-valued function defined on $x\in \rr^p$ by:
\[
\mu_{\yyy}(x)
    \triangleq 
\inf
\{
    \lambda>0:
    \,
x\in \lambda \yyy
\}
.
\]
Using $\yyy$'s gauge, we may define the map:
\begin{equation}
    \rho:\rr^p \ni y
 \mapsto 
\frac{1}{1+\mu_{\yyy}(y)}
y
\in \operatorname{int}\left(\yyy\right)
;
\label{eq_ex_Minkwoski_homeomorphism}
\end{equation}
which is in fact a continuous bijection with continuous inverse given by $z\mapsto \frac{1}{1-\mu_{\yyy}(z)}\,z$.  In other words, $\rho$ is a homeomorphism between $\rr^p$ and $\operatorname{int}\left(\yyy\right)$; in particular, Assumption~\ref{ass_boundary_homotopy} (i) holds.  

It remains to show that $\yyy$'s boundary is negligible, in the sense of Assumption~\ref{ass_boundary_homotopy} (ii).  For this we observe that, the convexity of $\yyy$ and the fact that $0\in \operatorname{int}{\yyy}$ implies that each $y\in \yyy$ is identified with the unique line segment $\gamma_{[0,y]}:[0,1]\rightarrow \yyy$ satisfying:
$
    \gamma_{[0,y]}(0)=0
$, $\gamma_{[0,y]}(1)=y,$ and such that $\gamma_{[0,y]}(t)\in \operatorname{int}{\yyy}$ whenever $0\leq t<1$.    Therefore, the following homotopy ``pushing $\yyy$ towards $0$'':
\begin{equation}
    H_t(y) \triangleq  ty
\label{eq_starshaped_homotopy}
;
\end{equation}
satisfies Assumption~\ref{ass_boundary_homotopy} (ii).

NB, a benefit of the readout map $\rho$ defined in~\eqref{eq_ex_Minkwoski_homeomorphism} over the readout map defined in Example~\ref{ex_closed_convex_UAP} is that $\rho$ and its two-sided inverse are often differentiable on most of $\rr^p$; which is of course convenient for training.  More precisely, by the implicit function theorem, $\rho$ and $\rho^{-1}$ are continuously differentiable on $\rr^p-\{0\}$ if and only if $\mu$ is.  Since $\yyy$ is convex then, $\mu$ defines a norm on $\rr^p$ by \citep[Exercise 5.105]{MariciBeckenstein_2011_TVSs} and therefore by \citep[Proposition 13.14]{KrieglMichor_1997_ConvenientGlobalAnalysis} $\mu$ is $k$-times continuously differentiable on $\rr^p-\{0\}$ if and only if $\yyy$ has a $C^k$-boundary.  
\end{example}
Besides illustrating the non-vacuousness of Assumption~\ref{ass_boundary_homotopy} (ii), Example~\ref{ex_Minkwoski_homeomorphism} suggests that a UAP-invariant readout map's must be ``topologically generic''\footnote{A ``topologically generic'' set here is meant in the sense of Baire Category; i.e. a dense $G_{\delta}$-subset of $\yyy$.  } and surjective thereon up to $\yyy$'s boundary.  This is indeed the case whenever $\yyy$ is a Riemannian manifold with boundary, as implied by the following geometric description of UAP-invariant readout maps' images.  
\begin{proposition}[Geometric Description of UAP-Invariant Readout Map's Images]
\label{prop_quantitative_Z_set_description}\hfill\\
Suppose that $\rho$ satisfies Assumption~\eqref{ass_boundary_homotopy}.  Then:
\begin{enumerate}
    \item[(i)] \textbf{$\rho$'s Image is Topologically Generic in $\yyy$:} $\im{\rho}$ is a dense open subset of $\yyy$,
    \item[(ii)] \textbf{$\rho$'s Remainder Belongs to $\yyy$'s Boundary:} If $\yyy$ is a topological manifold whose topology is induced by the metric $d_{\yyy}$ then, $\yyy-\im{\rho}$ is contained in $\yyy$'s boundary.
\end{enumerate}
\end{proposition}
Proposition~\ref{prop_quantitative_Z_set_description} can be used to rule out maps $\rho$ which are not UAP-invariant.  In particular, we deduce the following necessary condition for UAP-invariant maps between Euclidean spaces.  
\begin{example}[UAP-Invariant Maps Between Euclidean Spaces are Surjective]
\label{ex_UAP_invariance_Rp}
Since $\rr^p$ is a topological manifold without boundary then, any
$\rho:\rr^p \rightarrow \rr^p$ which is UAP-invariant must be surjective $\rho$ since $\rr^p-\im{\rho}$ must be contained in the empty set by Proposition~\ref{prop_quantitative_Z_set_description} (ii).  
\end{example}
We bring these concepts together in our final example of a UAP-invariant readout map, namely the softmax function of \cite{bridle1990probabilistic} which is omnipresent in classification.  
\begin{example}[Softmax Function and the Simplex]\label{ex_softmax_and_el_simplex}
Fix $C \in \nn_+$ with $C\geq 2$, consider the closed convex set $\Delta_C\triangleq \{y \in [0,1]^C:\, \sum_{c=1}^C y_c =1\}$, and consider the \textit{Softmax function}: 
\[
\vspace{-0.5em}
    \operatorname{Softmax}_C:\rrflex{C}\ni y 
        \mapsto 
    \left(
        \frac{e^{y_c}}{\sum_{c=1}^C e^{y_c}}
    \right)_{c=1}^C
        \in 
    \operatorname{int}(\Delta_C)
    .
\]
Define the affine map
 $
W: \rr^{C-1}\ni x\mapsto (x_1,\dots,x_{C-1},1)\in \rr^C
$ 
and define the map:
\[
\rho: \rr^{C-1} \ni y \mapsto \operatorname{Softmax}_C\circ W(y) \in \operatorname{int}(\Delta_C)
.
\]
Then, $\rho:\mathbb{R}^{C-1}\rightarrow 
\operatorname{int}(\Delta_C)
$
is continuous, $1$-Lipschitz, and a simple calculation verifies that 
$R(y)\triangleq 
    \left(
        \ln(y_c) - \ln(y_C) + 1
    \right)_{c=1}^{C-1}
$ is a continuous right-inverse of $\rho$ defined on $\operatorname{int}(\Delta_C)$.  Thus, Assumption~\ref{ass_boundary_homotopy} (i) holds.  Since $\Delta_C$ is convex and since $\rho$ maps $\rr^{C-1}$ surjectively onto star-shaped set $\operatorname{int}(\Delta_C)$ then, the following homotopy verifies Assumption~\ref{ass_boundary_homotopy} (ii)
\[
    H_t(y)
        := 
    t(y-\bar{\Delta}) + \bar{\Delta}
,
\]
where $\bar{\Delta}:=(1/C,\dots,1/C)$.  Thus, $\rho$ is a UAP-invariant readout map.
\end{example}
}}
Our last result's statement is substantially simplified by considering \textit{continuous}, but possibly sub-optimal, moduli of continuity.  The relevant moduli of continuity are the following.
\begin{remark}[{Technical Notation regarding the Last Theorem}]
\label{remark_techincal_modulus_for_last_theorem}
In this case, for a uniformly continuous function $f:\xxx\rightarrow \yyy$, between metric spaces $\xxx$ and $\yyy$, {\color{black}{with (possibly discontinuous) modulus of continuity $\omega(f,\cdot)$}} we define a continuous modulus of continuity: $\tilde{f}_{\rho}$ as follows.  If the modulus of continuity $\omega(f,\cdot)$ of  $f$ is continuous on $[0,\infty)$ then set $\tilde{\omega}_{f}\triangleq \omega(f,\cdot)$ otherwise, set 
$
\tilde{\omega}_{f}\triangleq \lim_{\tilde{t}\downarrow t}\, \tilde{t}^{-1}\int_{\tilde{t}}^{2\tilde{t}} \omega(f,s)ds
.
$  
{\color{black}{We maintain this notation throughout the remainder of the paper.  }}
\end{remark}
\begin{example}\label{ex_collapse_to_simple_modulus}
If $f$ is Lipschitz, or more generally, H\"{o}lder then $\omega(f,\cdot)=\tilde{\omega}_{f}$.
\end{example}
\subsubsection{Controlled Universal Approximation: General Version}\label{sss_Controll_UAT_Full_Version}
We may now state our final and main universal approximation theorem of this paper.    
\begin{theorem}[Controlled Universal Approximation: General Version] \label{thrm_full_computational_graph}
Suppose that $\yyy_i$ and $\mathcal{F}_i$ satisfy Assumptions~\ref{ass_group},~\ref{ass_properly_discontinuous_group_action}, and~\ref{ass_simply_connectedness} (mutatis mutandis).  Let $\sigma\in C(\rr)$ satisfy Assumption~\ref{ass_Kidger_Lyons_Condition}.  Suppose also that $\phi$ and $\rho$ are UAP-invariant.  Fix $f \in C(\xxx,\yyy)$, $0<\epsilon< 2^{-1}\sup_{t\in [0,\infty)}\, \omega_{\rho}(t)$, and fix a compact $\XX\subseteq \xxx$ satisfying the following condition.  \textit{There is an $x^{\star}\in \XX$ such that:}
$$
\XX\subset B_{\xxx}(x^{\star},\tilde{\omega}_{\phi}^{-1}(\eta \min_{i=1,\dots,I} \mathcal{U}_{[R\circ H_{t^{\frac{\epsilon}{2}}}\circ f\circ \phi^{-1}]_i}(x^{\star})).
$$
Then, for $i=1,\dots,I$, there exist $g_i\in \NN[p,m_i]$ and $y_i\in \yyy_i$ such that the model:
$$
\hat{f}\triangleq 
 \rho\left(
 \rho_1\circ \text{Exp}_{\mathcal{F}_1,y_1}\circ {g}_1\circ\text{Exp}_{\xxx,x}^{-1}\circ \phi
            ,
            \dots
            ,
        \rho_I\circ \text{Exp}_{\mathcal{F}_I,y_I}\circ {g}_I\circ\text{Exp}_{\xxx,x}^{-1}\circ \phi
        \right)
,
$$
(whose computational graph is in Figure~\ref{fig_full_computational_graph}) 
satisfies the estimate:
\begin{equation} 
    \sup_{x \in \XX} 
    d_{\yyy}
    \left(
        f(x)
            ,
        \hat{f}(x)
    \right)<\epsilon
    \label{thrm_full_computational_graph_main_estimate}
    .
\end{equation}
%
%
Moreover, the complexity of each $g_i$ depends on $\XX$'s geometry as follows:
\begin{enumerate}
    \item[(i)] \textbf{Efficient Case:} If there is an $n\in \nn_+$ such that $\phi(\XX)$ is $f_i\circ \phi^{-1}$-normalized, $n$-efficient for $f_i$, and if $\rho_i=1_{\mathcal{F}_i}$ then each $\hat{g}_i$: 
    \begin{enumerate}
                    \item[(a)] \textbf{Width:} satisfies $m\leq W\leq m_i(4p+10)$, 
                \item[(b)] \textbf{Depth:} of order $
                \mathscr{O}\left(m_i+m_i
                (
                    {\tilde{\omega}_{\rho}^{-1}(\frac{\epsilon}{2})}
                )^{
                    \frac{2p}{3(np+1)}
                        -
                    \frac{p}{np+1}
                }
                \right)
                $,
                \item[(c)] \textbf{Number of trainable parameters:} is of order 
                $
                \mathscr{O}\left(
                m_i(m_i^2-1)
                (
                    {\tilde{\omega}_{\rho}^{-1}(\frac{\epsilon}{2})}
                )^{-\frac{2p}{3(np+1)}}
                \right)
                ,
                $
                \item[(d)] The right-hand side of~\eqref{thrm_full_computational_graph_main_estimate} is instead $\kappa \epsilon\sum_{i=1}^I\sqrt{m_i}$; where $\kappa>0$ is a constant not depending on $p$, $m_i$, or on $\epsilon$.
            \end{enumerate}
    \item[(ii)] \textbf{General Case:} If $\XX$ is not efficient for $f_i$, $\rho_i\neq 1_{\mathcal{F}_i}$, or $\XX$ is not $f_i$-normalized, then $W=p+m_i+2$ and each $g_i$ has depth as in Table~\ref{tab_rates_and_depth_full_computational_graph}. 
\end{enumerate}
\end{theorem}
\begin{table}[h!]
    \centering
    \begin{tabular}{l|l}
    \toprule
        Regularity of $\sigma$ & Order of Depth \\
    \midrule
        $C^{\infty}(\rr)$ + Non-polynomial
         & 
         $
         O\left(
        \frac{
            m_i (2\operatorname{diam}(\phi(\XX)))^{2p} 
            }{
            \kappa_2^{2p}(\omega^{-1} \big(
                [R\circ H_{t_{2^{-1}\epsilon}}\circ f\circ \phi^{-1}]_i
        , \frac{\tilde{\omega}_{\rho}^{-1}(2^{-1}\epsilon)\kappa_1}{(1+\frac{p}{4})m_i} \big))^{2p} }
        \right)
         $
         \\
         Non-affine polynomial\footnote{We must allow for one extra neuron per layer.}
         & 
         $
         O \left(
    \frac{
        m_i(m_i+p)(2\operatorname{diam}(\phi(\XX)))^{4p+2}
        }{
        \kappa_2^{4p+2} (\omega^{-1} \big(
        [R\circ H_{t_{2^{-1}\epsilon}}\circ f\circ \phi^{-1}]_i
        , \frac{\tilde{\omega}_{\rho}^{-1}(2^{-1}\epsilon)\kappa_1}{(1+\frac{p}{4})m_i} \big))^{4p+2}
        }
        \right)
         $
         \\
         $C(\rr)$ + Non-polynomial
         &
         $
          O\left( 
    \frac{
       \big(\kappa_2\omega^{-1}(
        [R\circ H_{t_{2^{-1}\epsilon}}\circ f\circ \phi^{-1}]_i
        , \frac{\tilde{\omega}_{\rho}^{-1}(2^{-1}\epsilon)\kappa_1}{2m_i(1+\frac{p}{4})})
        \big)^{-2p}
        m_i (2\operatorname{diam}(\phi(\XX)))^{2p}
    }{
        \bigg(\kappa_2 \omega^{-1} \big( \sigma, \frac{\tilde{\omega}_{\rho}^{-1}(2^{-1}\epsilon)}{2Bm_i(2^{(2\operatorname{diam}(\phi(\XX)))^{2}[\omega^{-1}(
        [R\circ H_{t_{2^{-1}\epsilon}}\circ f\circ \phi^{-1}]_i
        , \frac{\tilde{\omega}_{\rho}^{-1}(2^{-1}\epsilon)\kappa_1}{2m_i(1+\frac{p}{4})})]^{-2}+1} -1)} \big) \bigg)
    }
        \right)
         $
         \\
    \bottomrule
    \end{tabular}
    Where $\kappa_1,\kappa_2>$ are independent of $\epsilon$, $p$, and of $m_i$.  
    \caption{{Approximation Rates for Geometric Deep Learning Model in Figure~\ref{fig_full_computational_graph}}}
    \label{tab_rates_and_depth_full_computational_graph}
\end{table}
\subsection{Applications}\label{ss_Applications_II_sophisiticated_models}
We use Theorem~\ref{thrm_full_computational_graph} to directly derive the UAP of various commonly implemented learning models.  
\subsubsection{Deep Softmax Classifiers are Universal}\label{ss_Applications_Softmax_Classifiers}
Multiclass classification is one of the most common uses of deep learning.  Here, the aim is to learn a function mapping $\xxx$ to $\Delta_C$, where $C$ is the number of classes.  The outputs of this function are typically interpreted as the probability that any input $x \in \xxx$ belongs to one of the $C$ classes.   Since most decision problems ultimately require the user to make a concrete decision as to which class(es) any $x \in \xxx$ belongs to.  Thus, the most important outputs of any classifier are the $1$-hot vectors (i.e., $y\in \Delta_C$ with $1$ in a single coordinate and $0$ elsewhere).

This problem is typically solved computationally by applying a softmax layer (see Example~\ref{ex_softmax_and_el_simplex}) to the output of a feedforward network $g \in \NN[p, C]$.  The composite model $\rho\circ g$ is then trained to approximate the target classifier $c \in C(\rrp,\Delta_C)$.  

We remark that it is clear that deep feedforward networks with softmax output layer can approximate any classifier taking values in the \textit{interior} of the $C$-simplex; i.e. in:
\begin{equation}
    \{y \in (0,1)^C:\, \sum_{c=1}^C y_c =1\}
\label{eq_interior_simplex_is_boring}
    .
\end{equation}
However, every $1$-hot vector in $\Delta_C$ never belongs to~\eqref{eq_interior_simplex_is_boring} as it is in the boundary of the $C$-simplex.  This topological obstruction has prevented uniform approximation results for continuous multiclass classifiers from appearing in the literature thus far.  Nevertheless, Theorem~\ref{thrm_full_computational_graph} implies the result.  
\begin{corollary}[Deep Classifiers are Universal]\label{cor_univ_multi_class_classification}
Let $\XX$ be a subset of a compact metric space $\xxx$, $C\in \nn_+$, and $\sigma$ satisfy Condition~\ref{ass_Kidger_Lyons_Condition} and suppose that there exists a UAP-preserving feature map $\phi:\xxx\rightarrow \rrp$.  For every $\epsilon>0$ and every classifier $f$ in $C(\xxx,\Delta_C)$ there is a $g\in \NN[p,C]$ satisfying:
$$
\max_{x \in \XX}\,
\sqrt{\sum_{c=1}^C\left( 
\frac{e^{g_c(x)}}{\sum_{\tilde{c}=1}^Ce^{g_{\tilde{c}}(x)}}
- f(x)_c
\right)^2}<\epsilon
.
$$
Moreover the following complexity estimates hold, depending on $\XX$ and $f$:
\begin{enumerate}
    \item[(i)] \textbf{Efficient Case:} If $\XX$ is $n$-efficient for $f$, for some $n\in \nn_+$ and $f$-normalized, then $g$ is as in Theorem~\ref{thrm_efficient_rates} (but with $2^{-1}\epsilon$ in place of $\epsilon$),
    \item[(ii)] \textbf{General Case:} If $\XX$ is not efficient for $f$, then $g$ has width $p+C+1$ and depth recorded in Table~\ref{tab_rates_and_depth} (but with $2^{-1}\epsilon$ in place of $\epsilon$).
\end{enumerate}
\end{corollary}

\subsubsection{Universality of the Deep Kalman Filter's Update Rule}\label{ss_EucGaussian} 
Our next application concerns the approximation of unknown functions with \textit{non-degenerate Gaussian measures}.   We study a mild extension of the architecture implemented in the \textit{deep Kalman filter} of \cite{krishnan2015deepDeepKalman}.  Our analysis begins by first constructing a universal deep neural model which processes non-degenerate Gaussian measures.    
The set of non-degenerate Gaussian measures on $\rrn$, denoted by $\ggg_n$, consists of all Borel probability measures $\nu_{\mu,\Sigma}$ on $\rrn$ with density 
$$
(2\pi)^{-\frac{n}{2}}\det(\boldsymbol\Sigma)^{-\frac{1}{2}} \, e^{ -\frac{1}{2}(\mathbf{x} - \boldsymbol\mu)^{{{\!\mathsf{T}}}} \boldsymbol\Sigma^{-1}(\mathbf{x} - \boldsymbol\mu) }
; 
$$
where $\mu \in \rrn$ and $\Sigma$ is a symmetric positive-definite $n\times n$-matrix; the set of which is denoted $P_n^+$.  

There are various geometries on $\ggg_n$ designed to highlight its different statistical properties while circumventing its non-linear structure.  Notable examples include the restriction of the Wasserstein-$2$ distance from optimal-transport theory, see \cite{VillaniOptTrans}, the Fisher-Rao metric introduced \cite{FisherRaoOriginal1945} from information-geometry, and the invariant metric introduced in \cite{GeometryOfMultivariateNormal_Lie_Canada} based on Lie-theoretic methods.  We focus on the former due to its uses in modern adversarial learning, such as in \cite{pmlrv70arjovsky17a,WGANTrainingNIPS2017892c3b1c}.
	
	The Wasserstein$-2$ distances on the spaces of probability measures with finite-variance are notoriously challenging.  However, \cite{DowsonLandau1982} found that when this distance is restricted to $\ggg_n$ then it reduces to 
	$$
	\mathcal{W}_2\left(
	\nu_{\mu_1,\Sigma_1}
	,
	\nu_{\mu_2,\Sigma_2}
	\right)
	=
	\sqrt{
		\left\|
		\mu_1-\mu_2
		\right\|^2
		+
		\operatorname{tr}\left(
		\Sigma_1 + \Sigma_2 - 2(\Sigma_1\Sigma_2)^{\frac1{2}}
		\right)
	}
	;
	$$ 
	where $(\cdot)^{\frac1{2}}$ denotes the matrix-square-root.
	Following \cite{MalagoWasserstein2018}, the map $\phi_0$ sending $\nu_{\mu,\Sigma}\in \ggg_n$ to $(\mu,\Sigma) \in \rrn\times P_n^+$ is not only a bijection, but it is also a homeomorphism when $P_n^+$ is equipped with the Fr\"{o}bnius metric $d_F(A,B)\triangleq \sqrt{\sum_{i,j=1}^n(A_{i,j}-B_{i,j})^2}$.  Building on the discussion of Section~\ref{ss_Applications_SPDplus},
%
%
%
we note that the map 
	$
	\left(
	1_{\rrn}\times (\text{Sym}_n^{-1}\circ \log)
	\right)
	\circ \phi_0
	,
	$
	is a homeomorphism from $\ggg_n$ to $\rrflex{ n(n+1)/2}$ with inverse function given by:
	\begin{equation}
	\phi_{\ggg_n}\triangleq \phi^{-1}_0\circ \left(
	1_{\rrn}\times (\exp\circ \text{Sym}_n)
	\right)
	;
	\label{eq_deep_kalman_feature_map}
	\end{equation}
where $\text{Sym}_n$ parameterizes the set of symmetric $n\times n$-matrices using $\rrflex{\frac{n(n+1)}{2}}$ and is defined  in~\eqref{eq_matrix_symmetrization_map}.
\begin{corollary}[Universal Approximation with Gaussian Inputs/Outputs]\label{cor_UAT_Gaussian_Data}
Let $\sigma$ satisfy Assumption~\ref{ass_Kidger_Lyons_Condition}, $f \in C(\ggg_n, \ggg_m)$, fix an $\epsilon >0$, and let $\XX\subseteq\ggg_n$ be non-empty and compact.  Then, there is a $g \in \mathcal{NN}_{n(n+1)/2,m(m+1)/2,(n(n+1)+m(m+1))/2+2}^{\sigma}$ satisfying:
\begin{equation}
    \sup_{\nu \in \XX} \mathcal{W}_2 \big ( f(\nu), \phi_{\ggg_m}^{-1} \circ g \circ \phi_{\ggg_n} (\nu) \big) \leq \epsilon
    \label{eq_cor_UAT_Gaussian_Data_estimate}
    .
\end{equation}
Moreover the following complexity estimates hold, depending on $\XX$ and $f$:
\begin{enumerate}
    \item[(i)] \textbf{Efficient Case:} If $\XX$ is $n$-efficient for $f$, for some $n\in \nn_+$ and $f$-normalized, then $g$ is as in Theorem~\ref{thrm_efficient_rates} (but with $2^{-1}n(n+1)$ and $2^{-1}m(m+1)$ in place of $p$ and $m$, respectively) and the right-hand side of~\eqref{eq_cor_UAT_Gaussian_Data_estimate} is $\kappa \sqrt{m(m+1)}\epsilon$, for a $\kappa>0$ independent of $m$, $n$, and of $\epsilon$.
    \item[(ii)] \textbf{General Case:} If $\XX$ is not efficient for $f$, then $g$ has width $2^{-1}(n(n+1)+m(m+1))+2$ and depth recorded in Table~\ref{tab_rates_and_depth} (but with $2^{-1}n(n+1)$ and $2^{-1}m(m+1)$ in place of $p$ and $m$, respectively).
\end{enumerate}
\end{corollary}

The architecture of Corollary~\ref{cor_UAT_Gaussian_Data} is an ``uncontrolled version" of the architecture implementing the deep Kalman filter's update rule.  Fix $a,p \in \nn$, at every increment, the deep Kalman filter of \cite{krishnan2015deepDeepKalman} maps an observation $x_t\in \rrp$, an action $u_t\in \rr^a$, and the previous latent sate, which is a measure $\nu_t \in \ggg_n$, to a measure $\nu_{t+1}\in \ggg_n$ via a deep neural model of the form:
\begin{equation}
    \nu_{t} = \hat{f}(x_{t-1},u_{t-1},\nu_{t-1}) \qquad (\forall t \in \nn_+)
    \label{eq_our_generalized_kalman_filter_update_rule}
    ,
\end{equation}
where $\nu_0\in \ggg_n$, $x_0\in \rrp$, and $u_0\in \rr^a$ are fixed and $\hat{f}$ is the deep neural model with representation:
\begin{equation}
    \hat{f} = \phi_{\ggg_n}^{-1}\circ g\circ (1_{\rrp}\times 1_{\rr^a}\times  \phi_{\ggg_n})
    \label{eq_our_generalized_kalman_filter_update_rule_representation}
    ,
\end{equation}
where $g \in \NN[p+a+2^{-1}n(n+1),2^{-1}n(n+1)]$ and $\phi_{\ggg_n}$ is as in~\eqref{eq_deep_kalman_feature_map}.  Then a mild modification to Corollary~\ref{cor_UAT_Gaussian_Data} implies the universality of the ``update map" of~\eqref{eq_our_generalized_kalman_filter_update_rule_representation} defining the deep Kalman filter.  
\begin{corollary}[Universality of the Deep Kalman Filter's Update Map]\label{cor_UAT_Deep_Kalman_Filter_Update_MAp}
Let $\sigma$ satisfy Assumption~\ref{ass_Kidger_Lyons_Condition}, $f \in C(\rrp\times \rr^a\times \ggg_n, \ggg_n)$, fix an $\epsilon >0$, and let $\XX\subseteq\rrp\times \rr^a\times \ggg_n$ be non-empty and compact.  Then, there is an $\hat{f}\in C(\rrp\times \rr^a\times \ggg_n, \ggg_n)$ with representation~\eqref{eq_our_generalized_kalman_filter_update_rule_representation} satisfying:
\begin{equation}
    \sup_{(x,u,\nu) \in \XX} \mathcal{W}_2 \big ( f(x,u,\nu), \hat{f}(x,u,\nu) \big) \leq \epsilon
    \label{eq_cor_UAT_Gaussian_Data_estimate_Kalman}
    .
\end{equation}
Moreover the following complexity estimates hold, depending on $\XX$ and $f$:
\begin{enumerate}
    \item[(i)] \textbf{Efficient Case:} If $\XX$ is $n$-efficient for $f$, for some $n\in \nn_+$ and $f$-normalized, then $g$ is as in Theorem~\ref{thrm_efficient_rates} (but with $p+a+2^{-1}n(n+1)$ and $2^{-1}n(n+1)$ in place of $p$ and of $m$, respectively) and the right-hand side of~\eqref{eq_cor_UAT_Gaussian_Data_estimate_Kalman} is $\kappa \sqrt{n(n+1)}\epsilon$, for a $\kappa>0$ independent of $m$, $n$, and of $\epsilon$.
    \item[(ii)] \textbf{General Case:} If $\XX$ is not efficient for $f$, then $g$ has width $p+a+2^{-1}(n(n+1)+n(n+1))+2$ and depth recorded in Table~\ref{tab_rates_and_depth} (but with $p+a+2^{-1}n(n+1)$ and $2^{-1}n(n+1)$ in place of $p$ and $m$, respectively).
\end{enumerate}
\end{corollary}
\begin{remark}\label{remark_gen_cov}
The architecture of Corollary~\ref{cor_UAT_Gaussian_Data} is actually more general than deep Kalman filter of \cite{krishnan2015deepDeepKalman} since it can process Gaussian measures with non-diagonal covariances.  
\end{remark}
\subsubsection{Universal Approximation to Projective Shape Space}\label{ss_Appliations_sss_UAP_SHAPESPACE}
Since it's introduction in \cite{KendallOG1984}, various (pre-)shape spaces have appeared in the literature (e.g. \cite{begelfor2006affine}, \cite{goodall1999projective}) most of which are summarized in either of the monographs \cite{NonparametricInferenceonManioldsShapeBhattacharyaTimes2} and \cite{DrydenMardiaKendallsShapeSpaceinR2016}.  In each case, the user seeks to filter out certain transformation $k$-tuples (called $k$-ads in the computer-vision literature).  The major difference between these shape spaces is which transformations are filtered out and how they are filtered.  In this section, we focus on the recently introduced \textit{projective shape space} $P\Sigma_m^k$ of \cite{RealProjectiveShapeSpaces2005AnnalsOfStats} for two major reasons.  First, its geometry is much more well-behaved than Kendall's shape space and second, its structure provides a perfect example of how to utilize the entire computational graph of Figure~\ref{fig_full_computational_graph}.  

We fix $k>m+2$.  An element of $P\Sigma_m^k$ is a \textit{projective shape of an $m+1$-dimension $k$-ad}; which is defined as follows.  Considers a $k$-ads $X\triangleq (x_k')_{k=1}^K$, where each $x_k' \in \rrflex{m+1}$, which correspond to non-degenerate shapes; which, for \cite{RealProjectiveShapeSpaces2005AnnalsOfStats}, means that $(x_k')_{k=1}^K$ span $\rrflex{m+1}$ and each $x_k'\neq 0$.  The effect of scaling is then removed by setting $x_k\triangleq \frac1{\|x_k'\|}x_k' \in S^m$.  One then applies the projection of $S^m \ni x \mapsto \{x,-x\}=[x]\in \rr P^m$.  Thus, the $k$-ad $\{x_k\}_{k=1}^K$ is sent to the element $([x_k])_{k=1}^K\in (\rr P^m)^K$.  Finally, all fractional linear transformations are filtered out of the $k$-ad by sending $([x_k])_{k=1}^K\in (\rr P^m)^K$ to its \textit{projective shape} defined by:
$$
[X]\triangleq \{([Ax_1],\dots,[Ax_K]):\, A \in \text{GL}_{m+1}\}
,
$$
where $\text{GL}_{m+1}$ is the set of all invertible $m+1\times m+1$-matrices.  The space $P\Sigma_m^k$ is the manifold of all \textit{projective shapes} constructed in this manner.  In \cite{RealProjectiveShapeSpaces2005AnnalsOfStats}, it is shown that $P\Sigma_k^m$ is a Riemannian manifold and a detailed description of its distance $d_{P\Sigma_k^m}$ is provided therein.  

Though this description most clearly explains what a projective shape is, there are analytically simpler descriptions.  Markedly, in \cite[Proposition 2.2-2.3]{RealProjectiveShapeSpaces2005AnnalsOfStats} a smooth bijection with smooth inverse $\rho: \prod_{k=1}^{K-m-2}\rr P^m \xrightarrow P\Sigma_m^k$ is defined in closed-form via certain algebraic relations.  Most notably for our context is the fact that $\rho$ satisfies Assumption~\ref{ass_boundary_homotopy} (i) since $\rho^{-1}$ is a well-define continuous map on all of 
$\prod_{k=1}^{K-m-2} \rr P^m$ and, since $\rho$ is a bijection, then it satisfies Assumption~\ref{ass_boundary_homotopy} (ii) via the trivial homotopy $H:(t,y)\in [0,1]\times P \Sigma_m^k\mapsto y\in P\Sigma_m^k$.  Note, $\rho$ is also Lipschitz (with some constant $\kappa>0$) since it is smooth and each $\rr P^m$ is compact.

As a contextual interpretation of our approximation theorem, consider the following dimension-reduction task for computer vision.  Suppose that we are provided with a dataset of high-resolution images $\XX\subseteq \rrflex{P\times p}$ ($P$ are the number of pixels in each image, $k$ is the number of features described by each pixel (e.g. RGB, alpha, etc...)).  We then want to extract the most essential $k\ll P$ pixels from each shape, and we want to compress the information in each pixel to some low-dimensional ``deep features" $m\ll p$.  Then, our objective would be to learn a function $f\in C(\rrflex{P\times p}, P\Sigma_m^k)$ which most efficiently performs this learning task according to some performance metric.  In that case, Theorem~\ref{thrm_full_computational_graph} guarantees that any such ``compression" function can be learned via the following deep neural models:
\begin{equation}
    \hat{f}\triangleq \rho([\text{Exp}_{S^m,y_1}\circ g_1],\dots,[\text{Exp}_{S^m,y_k}\circ g_k])
    \label{eq_GDN_projective_shape}
    ,
\end{equation}
where $g_i \in \NN[Pp,m]$, $y_i$, $\text{Exp}_{S^m,y_1}$ is defined in~\eqref{eq_sphere_Exp_Log}, and $S^m\ni x \mapsto [x]\in \rr P^m$ is defined in Example~\ref{ex_shape_projective} (for $i=1,\dots,k$).  
The learning model of~\eqref{eq_GDN_projective_shape} is a universal projective shape analogue of the well-known ``linear" regression models of \cite{davis2010population,Fletcher2013} designed for Kendall's pre-shape space.  
\begin{corollary}[Universal Approximator to Projective Shape Space {$P\Sigma_m^k$}]\label{cor_deep_image_compression_PSigma_m_k}
Let $\epsilon>0$, $\sigma$ satisfy Condition~\ref{ass_Kidger_Lyons_Condition}, and $f\in C(\rrflex{Pp},\Sigma P_m^k)$.  Then, for every compact $K\subseteq \rrflex{P\times p}$ of diameter at-most $
\frac{\pi}{2}$ there is an $\hat{f}$ as in~\eqref{eq_GDN_projective_shape} satisfying:
$$
\max_{x \in K} \,
d_{P\Sigma_m^k}\left(
f(x),\hat{f}(x)
\right)<\epsilon
.
$$
\end{corollary}
\subsection{{New Insights for ``Classical" Euclidean Universal Approximation}}\label{ss_Euc}
Our main results provide a number of new insights into the behaviour of deep feedforward networks between classical Euclidean spaces.  First, since $\rrp$ and $\rrm$ have everywhere $0$ sectional curvature, then Corollary~\ref{cor_Cartan_Hadamard} implies that $\mathcal{U}_f(x)=\infty$ for every $f \in C(\rrp,\rrm)$ and every $x \in \rrp$.  
As shown in \cite{jost2008riemannian}, for every $x \in \rrd$ we have the identities:
$$
\text{Exp}_{\rrd,x}(y)=y+x \mbox{ and }
\text{Exp}_{\rrd,x}^{-1}(y)=y-x
.
$$
Thus, Theorem~\ref{thrm_main_Local} reduces to a quantitative version of the \textit{(qualitative) universal approximation theorem for deep and narrow feedforward networks} derived in \cite{kidger2019universal}.  
\begin{corollary}[Quantitative Deep and Narrow Universal Approximation Theorem]\label{cor_quantitative}
Fix $f\in C(\rrp,\rrm)$ and suppose that $\sigma\in C(\rr)$ is non-affine and piecewise linear.  Then, for every compact $K\subset \rrp$ and every $\epsilon>0$ there exists a DNN $g \in \NN[p,m,p+m+2]$ satisfying:
$$
\sup_{x \in X}\|g(x)-f(x)\|<\epsilon.
$$
Moreover, the $g$'s depth is recorded in Table~\ref{tab_rates_and_depth}.  
\end{corollary}
Our results on efficient dataset imply the following extension of the dimension-free approximation results of \cite{yarotsky2019phase}.  
\begin{corollary}\label{cor_eff_approxim_EUC}
Let $f:\xxx\rightarrow \yyy$ and let $\XX\subseteq [0,1]^p$ be an $n$-efficient dataset for $f$.  For each $\epsilon>0$, there is a $W\in \nn_+$, a $g\in \NN[p,m:W]$, and a $\kappa>0$ not depending on $\epsilon$, $p$, or on $m$, such that:
\begin{equation}
    \sup_{x \in K}\, 
        \|
                f(x)
        -
                \hat{f}(x)
        \|
\leq 
    \kappa
    m^{\frac1{2}}
    \epsilon
    \label{eq_Euclidean_thrm_efficient_rates_rates}
    .
\end{equation}
Moreover, $g$ satisfies the following sub-exponential complexity estimates:
\begin{enumerate}
\item[(i)] \textbf{Width:} satisfies $m\leq W\leq m(4p+10)$, 
\item[(ii)] \textbf{Depth:} of order $
\mathscr{O}\left(m+m
\epsilon^{
    \frac{2p}{3(np+1)}
        -
    \frac{p}{np+1}
}
\right)
$,
\item[(iii)] \textbf{Number of trainable parameters:} is of order 
$
\mathscr{O}\left(
m(m^2-1)
\epsilon^{-\frac{2p}{3(np+1)}}
\right)
.
$
\end{enumerate}
\end{corollary}
Corollary~\ref{cor_eff_approxim_EUC} is a strict extension of the analogous result of \cite{yarotsky2019phase}.  This is because, Propositions~\ref{propsmmoth} and~\ref{prop_every_finite_dataset_is_efficient} guaranteed that every finite dataset is efficient for every function (and not only highly smooth ones as considered in \cite{yarotsky2019phase}).  Moreover, every dataset is efficient for functions in $C^{nd,1}([0,1]^p,\rrm)$.  

The first of these remarks gives a concrete and general explanation of why deep ReLU networks perform so well in practice.  Namely, they can always avoid the curse of dimensionality for ``real-world datasets".  We emphasise that the proof of this claim is non-trivial and relies on the notorious \textit{Whitney(-Fefferman) Extension Theorem}.  We emphasize that there is no continuity assumption on the target function in the next result (let-alone any smoothness assumptions).  
\begin{corollary}[{ReLU DNN Approximation is Not Cursed on ``Real-World" Datasets}]\label{cor_deep_ReLU_break}
If $\XX\subseteq  [0,1]^p$ is a ``real-world dataset" (i.e.: finite and non-empty) and $\sigma:x\mapsto \max\{0,x\}$.  
For every function $f:\rrp\rightarrow \rrm$ there is a $\kappa>0$ such that, for every $\epsilon>0$ there is a $g\in \NN[p,m,W]$ satisfying:
$$
\max_{x \in \XX} \|f(x)-g(x)\|
\leq 
    \kappa
    m^{\frac1{2}}
    \epsilon
    .
$$
Moreover, $g$ satisfies the following sub-exponential complexity estimates:
\begin{enumerate}
\item[(i)] \textbf{Width:} satisfies $m\leq W\leq m(4p+10)$, 
\item[(ii)] \textbf{Depth:} of order $
\mathscr{O}\left(m+m
\epsilon^{
    \frac{2p}{3(np+1)}
        -
    \frac{p}{np+1}
}
\right)
$,
\item[(iii)] \textbf{Number of trainable parameters:} is of order 
$
\mathscr{O}\left(
m(m^2-1)
\epsilon^{-\frac{2p}{3(np+1)}}
\right)
.
$
\end{enumerate}
\end{corollary}

We now summarize the contributions made in this paper.  
\section*{Conclusion}
This paper introduces a general and flexible \textit{``differentiable geometric deep learning framework"} for explicitly building universal approximators of a feedforward type between any differentiable manifold.  Our results are quantitative; they provide a (non-trivial) lower bound on the maximum radius in which our approximation guarantees could hold, and they estimate the complexity of the involved deep neural models' depth and width.  We derive dimension-free approximation rates on any efficient dataset. This novel concept strictly expand the efficient approximation results for smooth functions with Lipschitz higher-order partial derivatives.  We also prove that it is impossible to obtain global universal approximation results in the general non-Euclidean context utilizing a topological obstruction result (which always vanishes in the Euclidean setting, and more generally for pairs of Cartan-Hadamard input/output spaces).  

We use our theory to derive quantitative and efficient universal approximation theorems for a host of commonly implemented differentiable geometric deep learning architectures.  We show how popular geometric regression models could be simply extended to universal GDL models.  

As a final application, we show that our results imply multiple new insights even in the classical Euclidean setting.  These include a quantitative version of the approximation theorem of \cite{kidger2019universal} for deep and narrow feedforward networks and a general guarantee that deep ReLU networks always beat the curse of dimensionality when approximating any target function on any real-world dataset.  

We believe that the versatility, simplicity, and scope of the proposed framework provide a satisfactory differentiable geometric deep learning theory since, as we have seen, any continuous function between differentiable manifolds can be universally approximated via our \textit{differentiable geometric deep learning} models.    

\subsection*{Future Work}\label{ss_future}
The work developed in the current project leads to at least two future GDL research questions, which we will be working on shortly.  The first is a dynamic follow-up to this project, where we would like to use the theory developed herein to build universal RNNs and reservoir computers between differentiable manifolds.  The second follow-up problem is the treatment of non-differentiable input and output spaces.  In that follow-up project, we envision building on our theory by extending it to a GDL framework capable of processing non-differentiable input and output spaces such as graphs, trees, and hierarchical structures.  

\acks{
We are very grateful to Florian Krach for his helpful feedback and help in the manuscript's finalization stages.  We grateful to Patrick Kidger for his helpful insights and encouragement.  The authors would also like to thank Florian Rossmannek for his helpful discussion on piecewise linear activation functions and the $c$-identity requirement.  L\'{e}onie would like to thank Arash Salarian, who gave her time to work on this project during her internship at Logitech.  }

\appendix
This appendix complements the article's main body by providing additional background material to the mathematical background used throughout the paper, and by providing detailed proofs of each of the results.  

\section{Additional Background}\label{s_appendix_background}
To help the paper be as self-contained as possible, this appendix contains some relevant results and background from constructive approximation theory and from algebraic topology.  
\subsection{Bernstein polynomials and a quantitative shallow universal approximation theorem}\label{ss_Bernstein}

The classical result on neural networks is the so-called universal approximation theorem, mentioned in the introduction. It is at the core of our quantitative estimates, so we restate it here. Mathematically, a single-hidden layer neural network with activation function $\sigma$ can be written, for $x \in \mathbb{R}^{p}$,
\begin{equation} \label{NNasSum}
    \sum_{i=1}^{n} c_{i} \sigma( \langle w_{i}, x \rangle - \theta_{i})
\end{equation}
for some $w_{i} \in \mathbb{R}^{p}$, $c_{i}, \theta_{i} \in \mathbb{R}$. Here, the number of terms $n$ is the sum corresponds to the number of neurons in the hidden layer. We then have the following.

\begin{theorem}\citep[Theorem 3.1]{PinkusMLP} \label{uniapprox_main}
Let $\sigma \in \mathcal{C}(\mathbb{R})$. Then 
\begin{equation*}
    \mathcal{N}(\sigma) := \left \lbrace \sum_{i=1}^{n} c_{i} \sigma( \langle w_{i}, x \rangle - \theta_{i}) \ : \ n \in \mathbb{N}, \ w_{i} \in \mathbb{R}^{p},  \ c_{i}, \ \theta_{i} \in \mathbb{R} \right \rbrace
\end{equation*}
is dense in $\mathcal{C}(\mathbb{R}^{p})$ in the topology of uniform convergence on compact sets if and only if $\sigma$ is not a polynomial.
\end{theorem}

The first step to prove theorem \ref{uniapprox_main} is to use the Stone--Weierstrass theorem to approximate the target function by a polynomial. This multivariate polynomial is then approximated by a shallow network. Therefore, in order to derive quantitative estimates for theorem \ref{uniapprox_main},  we need to obtain a rate of convergence for the Stone--Weierstrass theorem. For continuous functions defined on $\mathbb{R}$, a proof of this theorem relies on Bernstein polynomials, thereby providing an explicit rate of convergence. In higher dimensions, a rate of convergence can be obtained using a multi-dimensional version of these Bernstein polynomials, that we introduce next. We use the following notation. For $x \in [0,1]$, $n, k \in \mathbb{N}$ such that $n \geq k$, we denote
\begin{equation*}
    p_{n,k}(x) = 
    \binom{n}{k}
    x^{k}(1-x)^{n-k}.
\end{equation*}

\begin{definition}\label{BernPolydef}
The multidimensional Bernstein operator $B_{n}: \mathcal{C}([0,1]^{p}, \mathbb{R}) \rightarrow \mathcal{C}([0.1]^{p}, \mathbb{R})$ is defined by, for $x = (x_{1}, \dots, x_{p}) \in \mathbb{R}^{p}$,
\begin{equation*}
    B_{n}(f,x) := \sum_{k_{1}=0}^{n} \dots \sum_{k_{p}=0}^{n} f(\frac{k_{1}}{n}, \dots \frac{k_{p}}{n}) \ p_{n, k_{1}}(x_{1}) \dots p_{n, k_{p}}(x_{p}).
\end{equation*}
$B_{n}(f, \, \cdot \,)$ is called the multivariate Bernstein polynomial associated to $f$.
\end{definition}

Using these multidimensional Bernstein operators, we obtain the following quantitative version of the Stone--Weierstrass theorem for real-valued functions defined on $\mathbb{R}^{p}$. The rate of convergence depends on the modulus of continuity of the target function $f \in \mathcal{C}([0,1]^{p}, \mathbb{R})$ which is defined as, for $\epsilon > 0$,
\begin{equation*}
    \omega(f, \epsilon) = \sup \{ \vert f(x) - f(y) \vert: \| x - y\| \leq \epsilon \}.
\end{equation*}

\begin{proposition} \label{prop_BernApprox} 
Let $f \in \mathcal{C}([0,1]^{p}, \mathbb{R})$. Then, for all $n \in \mathbb{N}, \ h>0$:
\begin{equation*}
    \| B_{n}(f) - f \|_{\infty} \leq \bigg(1 + \frac{p}{4} \bigg) \ \omega(f, \frac{1}{\sqrt{n}}).
\end{equation*}
\end{proposition}

Proposition \ref{prop_BernApprox} is a consequence of two theorems on properties of a certain Bochner-type integral and of the convexity of $[0,1]^{p}$. The first theorem, that we state below, defines this Bochner-type integral.

\begin{theorem}\citep[Theorem 6.2.1]{ApproxTheory} \label{theorem_BochInt}
Let $\mu$ be a Borel positive measure on $[0,1]^{p}$ such that $\mu([0,1]^{p})>0$. For any $f \in \mathcal{C}([0,1]^{p}, \mathbb{R})$, there is a unique $b \in \mathbb{R}$ having the following property. For any $\epsilon > 0$, there exists $\delta > 0$ such that for any partition $\{D_{1}, \dots, D_{m} \}$ of $[0,1]^{p}$ with $\mu(D_{i}) \leq \delta$ and for any choice of $x_{i} \in [0,1]^{p}$, $i=1, \dots, m$, we have
\begin{equation*}
    \big \vert b - \sum_{i=1}^{m} f(x_{i})\mu(D_{i}) \big \vert \leq \epsilon.
\end{equation*}
We denote $b:=F_{\mathbb{R}}(f)$ and $F_{\mathbb{R}}: \mathcal{C}([0,1]^{p}, \mathbb{R}) \rightarrow \mathbb{R}$ is a Bochner-type integral.
\end{theorem}

The second relevant result to prove Proposition \ref{prop_BernApprox} bounds the difference between a point evaluation of a function and the value of its Bochner integral.

\begin{theorem} \citep[Theorem 6.2.3]{ApproxTheory} \label{Bochnerint}
For any $f \in \mathcal{C}([0,1]^{p}, \mathbb{R})$, $x \in [0,1]^{p}$ and $h > 0$,
\begin{equation} \label{ineq_Bochner}
    \big \vert F_{\mathbb{R}}(f) - f(x) \big \vert \leq \vert f(x) \vert \ \vert F_{\mathbb{R}}(e_{0}) - 1 \vert + \big( F_{\mathbb{R}}(e_{0}) + h^{-2}F_{\mathbb{R}}(\| \, \cdot \, - x \|^{2})\big) \omega(f,h)
\end{equation}
where $e_{0}(x) = 1$ for all $x \in [0,1]^{p}$.
\end{theorem}
\begin{proof}[{Proof of Theorem~\ref{Bochnerint}}]
The proof of this theorem relies on elementary properties of the Bochner-type integral $F_{\mathbb{R}}$ constructed in theorem \ref{theorem_BochInt}. In particular, this operator is linear and positive. For $f \in \mathcal{C}([0,1]^{p}, \mathbb{R})$, $x \in [0,1]^{p}$, $h>0$, we then have
\begin{align*}
    \vert F_{\mathbb{R}}(f) - f(x) \vert & \leq \vert F_{\mathbb{R}}(f) - F_{\mathbb{R}}(f(x)e_{0}) \vert + \vert F_{\mathbb{R}}(f(x)e_{0}) - f(x) \vert \\
    & = \vert F_{\mathbb{R}}(f - f(x)) \vert + \vert f(x) \vert \vert F_{\mathbb{R}}(e_{0}) -1 \vert \quad \text{by linearity of } F_{\mathbb{R}} \\
    & \leq F_{\mathbb{R}}(e_{0}+h^{-2} \| \, \cdot \, - x \|_{\mathbb{R}^{p}}) \omega(f, h) + \vert f(x) \vert \vert F_{\mathbb{R}}(e_{0}) -1 \vert \\
    & \leq \bigg( F_{\mathbb{R}}(e_{0}) + h^{-2}F_{\mathbb{R}}(\| \, \cdot \, - x \|_{\mathbb{R}^{p}}) \bigg)  \omega(f, h) + \vert f(x) \vert \vert F_{\mathbb{R}}(e_{0}) -1 \vert.
\end{align*}
\end{proof}

We can now turn to the proof of proposition \ref{prop_BernApprox}. The idea is to define, for $x \in [0,1]^{p}$ fixed, a measure on $[0,1]^{p}$ for which the Bochner integral of $f \in \mathcal{C}([0,1]^{p}, \mathbb{R})$ with respect to this measure is precisely the Bernstein polynomial associated to $f$ (see definition \ref{BernPolydef}). Then, it suffices to compute the right-hand side of inequality (\ref{ineq_Bochner}) to obtain the stated convergence result.
\begin{proof}[{Proof of Proposition~\ref{prop_BernApprox}}]
For $n \in \mathbb{N}, \, x \in [0,1]^{p}$, we define the positive Borel measure on $[0,1]^{p}$
\begin{equation*}
    B_{n}(\, \cdot \,, x) := \sum_{k_{1}=0}^{n} \dots \sum_{k_{p}=0}^{n} \delta_{\frac{k_{1}}{n}, \dots, \frac{k_{p}}{n}} p_{n, k_{1}}(x_{1}) \dots p_{n, k_{p}}(x_{p})
\end{equation*}
where $\delta$ is the Dirac measure. By example 6.2.1 in \cite{ApproxTheory}, for all $f \in \mathcal{C}([0,1]^{p}, \mathbb{R})$, $x \in [0,1]^{p}$, it holds that
\begin{equation*}
     (B_{n})_{\mathbb{R}}(e_{0},x) = 1, \quad \text{and} \quad (B_{n})_{\mathbb{R}}( \| \, \cdot \, - x \|^{2}, x) = \sum_{j=1}^{p} \frac{x_{j}(1-x_{j})}{n}.
\end{equation*}
Therefore, by theorem \ref{Bochnerint}, for all $f \in \mathcal{C}([0,1]^{p}, \mathbb{R})$, $x \in [0,1]^{p}$, $h>0$,
\begin{equation*}
    \big \vert  B_{n} (f, x) - f(x)  \big \vert \leq \bigg( 1 + h^{-2} \sum_{j=1}^{p} \frac{x_{j}(1-x_{j})}{n} \bigg)\omega(f,h).
\end{equation*}
Thus, taking $h=n^{-1/2}$ and bounding $x_{j}(1-x_{j})$ by $1/4$, we obtain
\begin{equation*}
    \| B_{n}(f) - f \|_{\infty} \leq \bigg( 1 + \frac{p}{4} \bigg) \omega(f, \frac{1}{\sqrt{n}}).
\end{equation*}
\end{proof}
\subsection{Riemannian Geometric Background}\label{ss_Background_Riem_Geo}
Fix $p \in \nn$.  Broadly speaking, an $p$-dimensional \textit{Riemannian manifold} is a topological space $\xxx$ which a locally analogous geometry to Euclidean space in that it has locally interrelated notions of distance, angle, and volume, all of which are locally comparable to their analogs in $\rrp$.    

We briefly build up Riemannian manifolds from more elementary geometric objects, beginning with \textit{smooth manifolds}.  A smooth manifold $\xxx$, introduced in \cite{RiemannOriginal}, is a topological space on which a familiar differential calculus may be built, analogous to $\rrp$.  Since the derivative is a purely local object of any function, we only require that $\xxx$ can locally be identified with $\rrp$.  This local identification is achieved through a system of open subsets $\{U_{\alpha}\}_{\alpha \in A}$ of $\xxx$ which are identified with open subsets of $\rrp$ via continuous bijections $\phi_{\alpha}:U_{\alpha}\rightarrow \rrp$ each of which has a continuous inverse.  For any $\alpha,\beta\in A$, the functions $\phi_{\beta}\circ \phi_{\alpha}^{-1}:\phi_{\alpha}(U_{\alpha}\cap U_{\beta})\rightarrow 
\phi_{\beta}(U_{\alpha}\cap U_{\beta})
$ are defined between subsets of the familiar Euclidean space $\rrp$. Therefore we enforce a well-defined local calculus on all of $\xxx$ by requiring that each of these maps is infinitely differentiable.  The collection $\{(U_{\alpha},\phi_{\alpha}\}_{\alpha}$ is called an \textit{atlas} and each $(U_{\alpha},\phi_{\alpha})$ therein is called a \textit{coordinate patch}.  

Since we would like to linearize functions defined on $\xxx$ via their derivative, we need to extend the notion of a tangent line from Calculus to $\xxx$ to a collection of vector spaces lying tangential to the points of $\xxx$.  The construction of a tangent space begins with the definition of a \textit{smooth} function $f:\xxx\rightarrow\rr$, which is a continuous function for which each $f\circ \phi_{\alpha}^{-1}$ is infinitely differentiable.  The set of smooth functions on $\xxx$ is denoted by $C^{\infty}(\xxx)$. 
The derivative of a differentiable function $f:\rrp\rightarrow \rrm$ at some $x \in \rrp$ is a linear map $df_x:\rrp\rightarrow \rrm$ which approximates $f$ via $f(x+\Delta)\approx f(x)+df_x \Delta +o(\Delta)$.  Furthermore, the linearization operation at $x$, sending $f \mapsto df_x$, is characterized by the product rule
\begin{equation}
    d(fg)_x = f(x)dg_x + df_x g(x)
    .
\end{equation}
Thus, the set of all tangent vectors at $x$ to some differentiable function is identified with the linear maps from $C^{\infty}(\rrp)$ to $\rr$. This perspective is convenient, since the set of tangent vectors at any $x \in \xxx$, denoted by $T_x(\xxx)$, is the $p$-dimensional vector space of linear maps from $C^{\infty}(\xxx)$ to $\rr$.  In many situations the vector space $T_x(\xxx)$ admits a simple description and, whenever convenient, it is identified with $\rrp$.  

Tangent spaces allow us to define an intrinsic notion of distance on $\xxx$.  The description begins with \textit{vector fields}; these can be understood a rule which smoothly assigns a vector at each point of $\xxx$ and they are defined as linear maps $X:C^{\infty}(\xxx)\rightarrow C^{\infty}(\xxx)$ satisfying the product-rule $X(fg) = f(X(g)) + g(X(f))$.  Vector fields allow us to defined \textit{Riemannian metrics} on $\xxx$, these are families of inner products $g\triangleq (g_x)_{x \in \xxx}$ with each $g_x$ defined on $T_x(\xxx)$ such that $x\mapsto g_x(X|_x,Y|_x)$ is a smooth map, for every pair of vector fields $X$ and $Y$ on $\xxx$.  Together, $(\xxx,g)$ define a \textit{Riemannian manifold}, which we denote by $\xxx$ when the context is clear.  Riemannian metrics are of interest, since they induce an intrinsic distance $d_{\xxx}$ on $\xxx$ which is locally analogous to the Euclidean distance, and represents the length of the shortest tractable path between any two points $x_1,x_2 \in \xxx$ via
$$
d_{\xxx}(x_1,x_2)
\triangleq 
\inf \left\{
\int_0^1 
\sqrt{
g\left(
    \dot{\gamma}(t)
,
    \dot{\gamma}(t)
\right)
}
dt
:\,
\gamma(0)=x_1,\,\gamma(1)=x_2,\, \mbox{ and }
\gamma \mbox{ is piece-wise smooth}
\right\}
,
$$
where $\dot{\gamma}$ denotes the derivative of the curve $\gamma$ and, by definition, it exists for almost all $0\leq t\leq 1$.  

If the intrinsic distance $d_{\xxx}$ defines a complete metric on $\xxx$, then there is a standard open neighborhood about any $x \in \xxx$ which can be identified with a Euclidean ball.  This is because, together the results of \cite{Rinow1964} and of \cite{Lindelof_1894aa} guarantee that for any $x\in \xxx$ and any $u \in T_x(\xxx)$ of sufficiently small norm, there exists a unique smooth curve $\gamma$ originating at $x$, with initial velocity $\dot{\gamma}(0)=u\in T_{x}(\xxx)$, and of minimal length, i.e.:  $d_{\xxx}(x,\gamma(T))
    =
\int_0^T
\sqrt{
g\left(
    \dot{\gamma}(t)
,
    \dot{\gamma}(t)
\right)
}
dt
,
$ for all $0\leq T\leq 1$.  For any $x \in \xxx$, the least upper-bound on the norm of $u \in T_{x}(\xxx)$ guaranteeing the existence of such a $\gamma$ is called the \textit{injectivity radius} of $\xxx$ at $x$, and it is denoted by $\operatorname{inj}_{\xxx}(x)$.  The injectivity radius is key in our analysis since it allows us to linearizing $\xxx$ about any point $x$ while preserving the intrinsic distance between $x$ and points on $\xxx$ near it.  We do this through the \textit{Riemannian exponential map} at $x \in \xxx$, denoted by $\operatorname{Exp}_{\xxx,x}$, that sends any tangent vector $u$ at $x$ to $\gamma(1)$, where $\gamma$ is the distance-minimizing curve with initial conditions $\gamma(0)=x$ and $\dot{\gamma}(0)=v$.  It is a well-defined homeomorphism from $B_{\rrp}(0,\operatorname{inj}_{\xxx}(x))$ onto $B_{\xxx}(x,\operatorname{inj}_{\xxx}(x))$, making it a natural choice for a local feature and readout map, as it additionally preserves the distance between $x$ and any point $y \in B_{\rrp}(0,\operatorname{inj}_{\xxx}(x))$ through the \textit{radial isometry condition}
$
d_{\xxx}(x,y) = \left\|
\operatorname{Exp}_{\xxx,x}^{-1}(y)
\right\|
.
$

Two additional analogues between Euclidean space and Riemannian manifolds, which we use at different stages of our analysis, are its intrinsic volume and its orientation.  The \textit{intrinsic volume} $\operatorname{Vol}_{\xxx}(B)$ of any Borel set $B\subseteq U_{\alpha}$ in the coordinate patch $(U_{\alpha},\phi_{\alpha})$ is $\int_{x \in \phi_{\alpha}(B)} \sqrt{g\circ \phi_{\alpha}^{-1}} dx$.  A Riemannian manifold $\xxx$ is orientable if it is impossible to smoothly move a three-dimensional figure along $\xxx$ in such a way that the moving eventually results in the figure being flipped, rigorously, $\xxx$ must admit an atlas $\{(U_{\alpha},\phi_{\alpha})\}_{\alpha}$ where each $\phi_{\alpha}$ has positive Jacobian determinant.  
We also denote the set of 2 dimensional planes attached smoothly across $\xxx$ by $G_{p,2}(\xxx)$.
For more details on Riemannian geometry we refer the reader to \cite{jost2008riemannian}.  
\section{Proofs}
The remainder of this appendix is devoted to the proofs of our paper's main results.  We emphasize that, the order in which the proofs are derived differs from the order they were exposed in the paper's main body.  We chose this order so as to simplify both the paper and the appendix's flow.  
\subsection{Technical Lemmas}\label{lem_techincal_lemmas}
This appendix contains proof of the paper's results as well as any relevant technical lemmas.  
The proof of Theorem~\ref{thrm_main_Local} relies of the following two localization lemmas.  
\begin{lemma}\label{lem_redux_to_Euclidean_diffeo_Lipschitz_tools}
Let $\xxx$ and $\yyy$ a complete connected Riemannian manifolds satisfying Assumption~\ref{ass_non_degenerate_spaces}, of respective dimension $p$ and $m$, and let $f:\xxx\rightarrow\yyy$ be a continuous function.  For any $x \in \xxx$, if 
$$
0<\delta < \min\left\{
\delta(\xxx,x,k^{\star}_{\xxx}),
\omega^{-1}\left(f,
\delta(\yyy,f(x),k^{\star}_{\yyy})
\right)
\right\}
$$
then the following hold:
\begin{enumerate}
    \item[(i)] $\operatorname{Exp}_{\xxx,x}^{-1}:\overline{B_{\xxx}(x,\delta)} \rightarrow \rrp$ is a diffeomorphism onto its image.  In particular, it is Lipschitz with constant $L_{\xxx,x}>0$,
    \item[(ii)] $\operatorname{Exp}_{\yyy,f(x)}^{-1}:f(\overline{B_{\xxx}(x,\delta)}) \rightarrow \operatorname{Exp}_{\yyy,f(x)}^{-1}\left(
    f(\overline{B_{\xxx}(x,\delta)})
    \right)\subseteq \rrm$ is a diffeomorphism onto its image.
    In particular, $\operatorname{Exp}_{\yyy,f(x)}$ is Lipschitz on 
    \\$
    \operatorname{Exp}_{\yyy,f(x)}^{-1}\left(
    f(\overline{B_{\xxx}(x,\delta)})
    \right)
    $
    with constant $L_{\yyy,f(x)}>0$.
\end{enumerate}
In particular, we have the following estimate:
$$
0 < 
\min\left\{
\delta(\xxx,x,k^{\star}_{\xxx}),
\omega^{-1}\left(f,
\delta(\yyy,f(x),k^{\star}_{\yyy})
\right)
\right\}
\leq 
\min\left\{
\delta(\xxx,x,k^{\star}_{\xxx}),
\omega^{-1}\left(f,
\delta(\yyy,f(x),k^{\star}_{\yyy})
\right)
\right\}
.
$$
\end{lemma}
\begin{proof}
By Assumption~\ref{ass_non_degenerate_spaces} (i) we have the following finite bound on the Riemannian curvature $K_{\xxx}$ of $\xxx$
$$
0\leq K(\xxx)
    \triangleq 
\sup_{\pi_x(u,v):x\in \kkk, \,\pi_x(u,v) \in G_{p,2}(\xxx)} 
    \left|
    K_{\xxx}(
        \pi_x(u,v)
    )
    \right|
< \infty
.
$$
Since $\xxx$ is a complete Riemannian manifold, then \citep[Theorem 4.7]{CheegerGromovTaylorTheorem1982} implies the following lower-bound on the injectivity radius at any $x \in \xxx$ 
\begin{equation}
    \operatorname{inj}(x) > 
    r \frac{
\operatorname{Vol}_{\xxx}\left(
B_{\xxx}(x,r)
\right)
}{
\operatorname{Vol}_{\xxx}\left(
B_{\xxx}(x,r)
\right)
+
\operatorname{Vol}_{T_x(\xxx)}\left(
B_{T_x(\xxx)}(0,2r)
\right)
}
,
    \label{eq_first_bound_input_space}
\end{equation}
for any $0<r<k^{\star}_{\xxx}$.  In particular, 
\begin{equation}
    \operatorname{inj}(x)\geq 
    \delta\left(
    \xxx,x,
    k^{\star}_{\xxx}
    \right).
    \label{eq_first_inequality}
\end{equation}
By Assumption~\ref{ass_non_degenerate_spaces} (ii), we have that $Vol_{\xxx}(B_{\xxx}(x,r))>0$ for any $r>0$ and therefore the right-hand side of~\eqref{eq_first_bound_input_space} is non-zero. Hence,~\eqref{eq_first_inequality} refines to
\begin{equation}
    \operatorname{inj}(x)\geq 
    \delta\left(
    \xxx,x,
    k^{\star}_{\xxx}
    \right)>\delta >0
    .
    \label{eq_first_inequality_useable}
\end{equation}
By \citep[Corollary 1.7.1]{jost2008riemannian}, the map $\operatorname{Exp}_{\xxx,x}^{-1}$ is a diffeomorphism from $B_{\rrp}(0,\operatorname{inj}(x))$ onto $\operatorname{Exp}_{\xxx,x}\left(
B_{\rrp}(0,\operatorname{inj}(x))
\right)$.  Moreover, since $\operatorname{Exp}_{\xxx,x}$ is a radial-isometry (see the discussion following \citep[Corollary 1.4.2]{jost2008riemannian}) then
$
\operatorname{Exp}_{\xxx,x}\left(
B_{\rrp}(0,\tilde{\delta})
\right) =
B_{\xxx}(x,\tilde{\delta})
$ for every $0<\tilde{\delta}\leq \operatorname{inj}(x)$. Since $\delta<\operatorname{inj}(x)$ then,
\begin{equation}
    \operatorname{Exp}_{\xxx,x}^{-1}|_{\overline{B_{\xxx}(x,\delta)}}:\overline{B_{\xxx}(x,\delta)} \rightarrow \overline{B_{\rrp}(0,\delta)}
    ,
    \label{eq_diffeo}
\end{equation}
is a diffeomorphism.  Since $\operatorname{Exp}_{\xxx,x}^{-1}$ is a diffeomorphism then it is in particular Lipschitz with constant $L_{\xxx,x}>0$.  Thus, (i) holds.  

Next, let $\tilde{x}\in B_{\xxx}(x,\delta)$.  Then, by the definition of the modulus of continuity of $f$, by its monotonicity, and since $\omega(f,\delta)<\delta(\yyy,f(x),k^{\star}_{\yyy})$ we compute
\begin{equation}
    \begin{aligned}
        d_{\yyy}\left(
        f(\tilde{x}),f(x)
        \right)
        \leq &
        \omega\left(
        f,
        d_{\xxx}\left(
        \tilde{x},x
        \right)
        \right)
        \\
        <&
        \omega\left(
        f,\delta
        \right)
        \\
        \leq &
        \delta(\yyy,f(x),k^{\star}_{\yyy})
        ;
    \end{aligned}
    \label{eq_push_bound}
\end{equation}
where the last inequality follows from \citep[Proposition 1 (5)]{EmbrechtsHofert}.  
In particular, $f\left(\overline{B_{\xxx}(x,\delta)}\right)\subseteq \overline{B_{\yyy}\left(
f(x),\delta(\yyy,f(x),k^{\star}_{\yyy})
\right)}$.

As in the proof of (i), under Assumptions~\ref{ass_non_degenerate_spaces} (i) and (ii), \citep[Theorem 4.7]{CheegerGromovTaylorTheorem1982} implies that 
$$
\operatorname{inj}_{\yyy}(f(x)) \geq \delta(\yyy,f(x),k^{\star}_{\yyy}) >0,
$$
and therefore \citep[Corollaries 1.7.1 and 1.4.2]{jost2008riemannian} implies that $\operatorname{Exp}_{\yyy,f(x)}$ is a radially-isometric diffeomorphism from $B_{\rrm}(0,\delta(\yyy,f(x),k^{\star}_{\yyy}))$ onto $B_{\yyy}(f(x),\delta(\yyy,f(x),k^{\star}_{\yyy}))$.  Moreover, as before, since $\overline{f(B_{\xxx}(x,\delta))}$ is a compact subset of $B_{\yyy}(f(x),\operatorname{inj}_{\yyy}(f(x)))$ then\\ $\operatorname{Exp}_{\yyy,f(x)}:f(B_{\xxx}(x,\delta)) \rightarrow \rrm$ is Lipschitz with some constant $L_{\yyy,f(x)}>0$.  This gives (ii).  
\end{proof}
So as not to disrupt the appendix's overall flow, we maintain the notation introduced in the proof of Lemma~\ref{lem_redux_to_Euclidean_diffeo_Lipschitz_tools} within the next Lemma's statement and its proof.  
\begin{lemma}\label{lem_local_representation}
Let $\xxx$ be a complete connected Riemannian manifold, $f:\xxx\rightarrow\yyy$.  
For any $x \in \xxx$, we denote the 
Lipschitz constant of $Exp_{\xxx, x}$ on 
$\overline{B_{\xxx}(0,\operatorname{inj}_{\xxx}(x))}$
by $L^{-1}_{\xxx,x}$ and we use $L^{-1}_{\yyy,f(x)}$ to denote the Lipschitz constant of $Exp_{\yyy, f(x)}^{-1}$ on 
$\overline{B_{\yyy}(0,\operatorname{inj}_{\yyy}(f(x)))}$
.  
If it holds that:
$$
0<\delta < 
\min\left\{
\operatorname{inj}_{\xxx}(x),
\omega^{-1}\left(f,
\operatorname{inj}_{\yyy}(f(x))
\right)
\right\}
$$
then, on the compact set $\overline{B_{\xxx}(x,\delta)}$ the map $f$ can be represented as 
\begin{equation}
    f=\operatorname{Exp}_{\yyy,f(x)}\circ \tilde{f}\circ \operatorname{Exp}_{\xxx,x}^{-1}
    \label{eq_representation_of_f_locally}
\end{equation}
where $\tilde{f}:\operatorname{Exp}_{\xxx,x}^{-1}\left(
    \overline{B_{\xxx}(x,\delta)}
    \right)
    \rightarrow \rrm$ is defined by:
    \begin{equation}
    \tilde{f} \triangleq  
    \operatorname{Exp}_{\yyy,f(x)}^{-1}
    \circ 
    f
    \circ 
    \operatorname{Exp}_{\xxx,x}
    .
    \end{equation}
    Furthermore, if $f$ is continuous, then so is $\tilde{f}$ and its modulus of continuity $\omega(\tilde{f},\cdot)$ is given by
    \begin{equation}
    \omega(\tilde{f},\epsilon) = 
    L_{\yyy,f(x)}^{-1}\omega\left(
    f,L_{\xxx,x}^{-1}\epsilon
    \right)
    .
    \label{eq_representation_of_f_tilde_locally}
\end{equation}
\end{lemma}
\begin{proof}
By Lemma~\ref{lem_redux_to_Euclidean_diffeo_Lipschitz_tools} the map $\tilde{f}$ is well-defined and continuous.  Furthermore, by the same result, representation~\eqref{eq_representation_of_f_locally} holds since $\operatorname{Exp}_{\yyy,f(x)}^{-1}$ and $\operatorname{Exp}_{\xxx,x}^{-1}$ are diffeomorphisms on $f\left(
\overline{B_{\xxx}(x,\delta)}
\right)$ and on $B_{\xxx}(x,\delta)$, respectively.  

Lastly, we compute the modulus of continuity of $\tilde{f}$.  By~\eqref{eq_representation_of_f_tilde_locally} and the fact that the modulus of continuity of a composition of uniformly continuous functions is equal to the composition of the moduli of continuity of the involved uniformly continuous functions, we have that
$$
\begin{aligned}
\omega(\tilde{f},\epsilon) =& \omega\left(
\operatorname{Exp}_{\yyy,f(x)}^{-1}
    \circ 
    f
    \circ 
    \operatorname{Exp}_{\xxx,x}
,\epsilon\right)\\
 = &\omega(
\operatorname{Exp}_{\yyy,f(x)}^{-1},\epsilon)\circ
    \omega(
    f,\epsilon)
    \circ 
    \omega(
    \operatorname{Exp}_{\xxx,x},\epsilon)
    \\
    = &
    L_{\yyy,f(x)}^{-1}\omega\left(
    f,L_{\xxx,x}^{-1}\epsilon
    \right).
\end{aligned}
$$
\end{proof}
Next, we begin by deriving our depth estimates for deep and narrow feedforward networks from $\rrp$ to $\rrm$.  The result will subsequently be combined with the above lemmas to derive our controlled approximation results between more general non-Euclidean spaces.  
\subsection{Depth estimates for deep and narrow feedforward networks}\label{s_A_Depth_Estimates}
This section is devoted to the proof of the quantification of the approximation results of \cite{kidger2019universal}.  More precisely, given a prespecified error $\epsilon > 0$, we instigate how deep a neural network $g \in \mathcal{NN}_{p,m,p+m+2}^{\sigma}$ should be in order to approximate within the margin error a continuous function $f: K \to \mathbb{R}^{m}$ where $K \subset \mathbb{R}^{p}$ is compact. This quantitative proposition is at the core of all our quantitative estimates, even when we consider non-Euclidean input and output spaces. 

\begin{proposition}\label{maindepth}
Let $\sigma: \mathbb{R} \rightarrow \mathbb{R}$ be an activation function satisfying assumption \ref{ass_Kidger_Lyons_Condition}. Let $K \subset \mathbb{R}^{p}$ be a compact set and let $f \in \mathcal{C}(K, \mathbb{R}^{m})$. Then, for any $\epsilon >0$, there exists $g \in \mathcal{NN}_{p, m, p+m+2}^{\sigma}$ such that $\| f - g\|_{\infty} \leq \epsilon$. Moreover:
\begin{enumerate}
    \item[(i)] if $\sigma$ is infinitely differentiable and non-polynomial, then the depth of $g$ is of order
    \begin{equation}
        O\bigg(m (\text{diam}K)^{2p} \bigg(\omega^{-1} \big(f, \frac{\epsilon}{(1+\frac{p}{4})m} \big) \bigg)^{-2p} \bigg)
    \end{equation}
    \item[(ii)] if $\sigma$ is non-polynomial, then the depth of $g$ is of order
    \begin{align} \label{depthcontinuous}
        O &\bigg( m (\text{diam}K)^{2p} \bigg(\omega^{-1} \big(f, \frac{\epsilon}{2m(1+\frac{p}{4})} \big) \bigg)^{-2p} \nonumber \\
        &\bigg( \omega^{-1} \big( \sigma, \frac{\epsilon}{2Bm(2^{\text{diam}K^{2}[\omega^{-1}(f, \frac{\epsilon}{2m(1+\frac{p}{4})})]^{-2}+1} -1)} \big) \bigg)^{-1}\bigg)
    \end{align}
    for some $B > 0$ depending on $f$.
    \item[(iii)] if $\sigma$ is a non-affine polynomial, then, if we allow an extra neuron on each layer of $g$, the depth of $g$ is of order
    \begin{equation}
        O \bigg(m(p+m)(\text{diam}K)^{4p+2}\bigg(\omega^{-1} \big(f, \frac{\epsilon}{(1+\frac{p}{4})m} \big) \bigg)^{-4p-2} \bigg).
    \end{equation}
\end{enumerate}
\end{proposition}

\begin{remark}
In point 3 of proposition \ref{maindepth}, we have chosen to allow an extra neuron for clarity. However, depth estimates in the case where no extra neuron is allowed are derived in section \ref{InverseEst}. 
\end{remark}

\subsubsection{An Extension Result}\label{s_Appendix_Proofs_sss_extension_result}
The depth estimates of Proposition \ref{maindepth} are derived in the case where the function $f$ being approximated is defined on $[0,1]^{p}$. Indeed, on such a domain of definition, the multivariate Bernstein polynomials introduced in definition \ref{BernPolydef} are well-defined and can be used to approximate each component of $f$. Restricting our analysis to functions defined on $[0,1]^{p}$ is enough to derive estimates for functions defined on an arbitrary compact set $K$. Indeed, in this case, the function $f \in \mathcal{C}(K, \mathbb{R}^{m})$ being approximated can be extended to the whole of $\mathbb{R}^{p}$ in such a way that the extension preserves its modulus of continuity. By a simple change of variables, this extension can be considered on $[0,1]^{p}$. Before stating our extension proposition, we introduce the following subadditive modulus of continuity, on which the extension relies. 

\begin{definition}
The concave majorant of $\omega(f, \, \cdot \,)$ is defined by, for $t>0$,
\begin{equation*}
    \omega_{c}(f,t) := \inf \lbrace \alpha t + \beta: \omega(f,s) \leq \alpha s + \beta \ \forall s \in \mathbb{R}_{+} \rbrace.
\end{equation*}
\end{definition}

The subadditivity of $\omega_{c}(f, \, \cdot \,)$ is the key to prove the following extension proposition. This proposition will allow us to leverage depth estimates for functions defined on $[0,1]^{p}$ to depth estimates for functions defined on an arbitrary compact set $K \subset \mathbb{R}^{p}$. It is a slightly different version of \citep[Corollary 2]{McShane}.

\begin{proposition} \label{ContExt}
Let $K \subset \mathbb{R}^{p}$ be a compact set. Let $f:K \rightarrow \mathbb{R}$ be a continuous function. Then $f$ can be extended to $\mathbb{R}^{p}$ by setting, for $x \in \mathbb{R}^{p}$,
\begin{equation} \label{defext}
    F(x) := \frac{1}{2} \sup_{y \in K} \lbrace f(y) - \omega_{c}(f, \|x-y\|) \rbrace.
\end{equation}
Moreover, $F$ preserves the modulus of continuity $\omega(f, \, \cdot \,)$, that is
\begin{equation} \label{Fmod}
    \forall x, y \in \mathbb{R}^{p}, \ \vert F(x) - F(y) \vert \leq \omega(f, \|x - y \|).
\end{equation}
\end{proposition}

We now illustrate why this proposition allows us to leverage our results for functions defined on $[0,1]^{p}$ to functions defined on an arbitrary compact set $K \subset \mathbb{R}^{p}$. Let $K \subset \mathbb{R}^{p}$ be a compact set and let $f: K \rightarrow \mathbb{R}^{m}$ be a continuous function. Write $f=(f_{1},\dots,f_{m})$. Each $f_{j}$ can be extended to a function $F_{j}$ defined on $\mathbb{R}^{p}$ by (\ref{defext}) and such that (\ref{Fmod}) is satisfied. Since for all $j$, $\omega(f_{j}, \, \cdot \,) \leq \omega(f,\, \cdot \,)$, each extension $F_{j}$ is such that for all $x,y \in \mathbb{R}^{p}$,
\begin{equation*}
     \vert F_{j}(x) - F_{j}(y) \vert \leq \omega(f, \|x-y\|).
\end{equation*}
We embed the compact set $K$ in a cube of the form $[a,b]^{p}$ where $a,b \in \mathbb{R}$. The extensions $F_{j}$ are restricted to $[a,b]^{p}$ and by a change of variables, they can be defined on $[0,1]^{p}$. Then, the results that will be obtained for functions defined on $[0,1]^{p}$ apply to the $F_{j}$. Let $\tilde F_{j}$ be the re-scaled version of $F_{j}$ that is defined on $[0,1]^{p}$. For $v, \, w \in [0,1]^{p}$, by proposition \ref{ContExt}, we have
\begin{align*}
    \vert \tilde F_{j}(v) - \tilde F_{j}(w) \vert &= \vert F_{j}((\text{diam} K) v + \inf K) - F_{j}((\text{diam} K) w + \inf K) \vert \\
    & \leq \omega(F_{j}, (\text{diam} K) \|v-w\|) \leq \omega(f_{j}, (\text{diam} K) \|v-w\|).
\end{align*}
This implies that the depth of the network is still controlled by the original function $f$. Indeed, we will show that the depth of the network depends on the convergence rate of the multivariate Bernstein polynomials $B_{n}(f_{j}, \, \cdot \,)$ associated to each component $f_{j}$ of $f$. By proposition \ref{prop_BernApprox} and the above inequality, for all $j \in \{1, \dots, m \}$, we have
\begin{equation*}
    \| B_{n}(\tilde F_{j}) - \tilde F_{j} \|_{\infty} \leq \bigg( 1+ \frac{p}{4} \bigg)\omega( \tilde F_{j}, \frac{1}{\sqrt{n}}) 
    \leq \bigg( 1+ \frac{p}{4} \bigg) \omega(f_{j}, \frac{\text{diam}K}{\sqrt{n}}).
\end{equation*}
This inequality guarantees that the extension operation does make the approximating network artificially deep.

We now turn to the proof of proposition \ref{ContExt}. We need the following lemma that relates $\omega_{c}(f, \, \cdot \,)$ and $\omega(f, \, \cdot \,)$. This corresponds to \citep[Lemma 6.1]{BookModCOnt}.

\begin{lemma} \label{mod2}
Let $h: \mathbb{R}^{p} \rightarrow \mathbb{R}$ be a continuous function. Then, for all $t \geq 0$,
\begin{equation*}
    \omega_{c}(h,t) \leq 2 \omega(h, t).
\end{equation*}
\end{lemma}

\begin{proof}[{Proof of Proposition~\ref{ContExt}}]
Since $f$ is continuous and $K$ is a compact set, $f$ is bounded and thus the supremum in (\ref{defext}) is well-defined. Moreover, if $x \in K$, $F(x) = f(x)$ since $x$ achieves the supremum in this case. It remains to check (\ref{Fmod}). If $x, y \notin K$, for $\epsilon > 0$, let $z \in K$ and $w \in K$ be such that
\begin{align*}
    & \frac{1}{2} \big( f(z) - \omega_{c}(f, \| x - z \|) \big) - \epsilon \leq F(x) \leq \frac{1}{2} \big( f(z) - \omega_{c}(f, \| x - z \|) \big)+ \epsilon, \\
    & \frac{1}{2} \big( f(w) - \omega(f, \| y - w \|) \big) - \epsilon \leq F(y) \leq \frac{1}{2} \big( f(w) - \omega(f, \| y - w \|) \big) + \epsilon.
\end{align*}
Then, we have
\begin{align*}
    F(x) - F(y) & \leq \frac{1}{2} \big( f(z) - \omega_{c}(\| x - z \|) \big) + \epsilon - \frac{1}{2} \big(f(w) - \omega_{c}(f, \| y - w \|) \big) + \epsilon \\
    & \leq \frac{1}{2} \bigg( \omega_{c}(f, \| z - w \|) - \omega_{c}(f, \| x - z \|) + \omega_{c}(f, \| y - w \|) \bigg) \\
    & \leq \frac{1}{2} \bigg( \omega_{c}(f, \| z - w \|) - ((\omega_{c}(f, \| x -  w\|) + \omega_{c}(f, \| w - z \|)) + \omega_{c}(f, \| y - w \|) \bigg) \\
    & = \frac{1}{2} \bigg( - \omega_{c}(f, \| x - w\|) + \omega_{c}(f, \| y - w \|) \bigg) \\
    & \leq \frac{1}{2} \bigg( - \omega_{c}(f, \| x - w\|) + \omega_{c}(f, \| y - x \|) + \omega_{c}(f,\| x- w \|) \bigg)\\
    & = \frac{1}{2}\omega_{c}(f, \| y - x \|) \\
    & \leq \omega(f, \| y - x \|)
\end{align*}
where the last inequality follows from lemma \ref{mod2}. By symmetry, we obtain (\ref{Fmod}) in this case. If $x \in K$, $y \notin K$, similar computations yield the desired result, up to an arbitrary additive $\epsilon > 0$. If $x, y \in K$, then $F$ obviously preserves the first-order modulus of continuity $\omega(f, \, \cdot \,)$. We note that (\ref{Fmod}) also gives the continuity of the extension $F$.
\end{proof}

\subsubsection{Depth estimates for non-polynomial activation functions}
In the case where the activation function $\sigma$ is non-polynomial, the proof of \citep[Theorem 3.2]{kidger2019universal} relies on a deep network that is constructed by "verticalizing" shallow networks and to some extent the depth of the resulting network corresponds to the sum of the widths of these shallow networks. These networks are obtained via the universal approximation theorem. Hence, to derive the estimates of Proposition \ref{maindepth}, we should constructively prove the universal approximation theorem (Theorem \ref{uniapprox_main}). For this, we follow the strategy sketched in \citep[section 3]{PinkusMLP}.

\subsubsection{Depth estimates for a smooth activation function}
We start with a lemma that provides a decomposition of a multivariate polynomial as a sum of univariate polynomials. The proof can be found in the proof of \citep[Theorem 4.1]{PinkusMLP}.

\begin{lemma} \label{MultiAsUni}
Let $h: \mathbb{R}^{p} \rightarrow \mathbb{R}$ be a polynomial of degree $k$. Set
\begin{equation} \label{defr}
    r := 
    \binom{p - 1 + k}{k}
    .
\end{equation}
Then, there exist $a^{i} \in \mathbb{R}^{p}$ and univariate polynomials $p_{i}$ of degree at most $k$, $i=1,...,r$ such that
\begin{equation*}
    h(x) = \sum_{i=1}^{r} p_{i} \big( \langle a^{i}, x \rangle \big).
\end{equation*}
\end{lemma}

We can now state and prove our first quantitative result for approximation with neural networks.

\begin{proposition} \label{propsmmoth}
Let $\sigma : \mathbb{R} \rightarrow \mathbb{R}$ be a non-polynomial smooth function. Let $f: [0,1]^{p} \rightarrow \mathbb{R}^{m}$ be a continuous function. Then, for any $\epsilon > 0$, there exists $g \in \mathcal{NN}_{p,m,p+m+2}^{\sigma}$ such that $\| f - g \|_{\infty} \leq \epsilon$. The depth of $g$ is of order
\begin{equation} \label{depthsmooth}
    O\bigg(m \bigg(\omega^{-1} \big(f, \frac{\epsilon}{(1+\frac{p}{4})m} \big) \bigg)^{-2p} \bigg).
\end{equation}
\end{proposition}

\begin{proof}
Let $f: [0,1]^{p} \rightarrow \mathbb{R}^{m}$ be a continuous function and write $f=(f_{1}, \dots, f_{m})$. Let $\sigma: \mathbb{R} \rightarrow \mathbb{R}$ be an infinitely differentiable non-polynomial function and let $\epsilon > 0$. For $j=1, \dots, m$, we approximate $f_{j}$ by a Bernstein polynomial $B_{n_{j}}(f_{j}): [0,1]^{p} \rightarrow \mathbb{R}$
\begin{equation*}
    B_{n_{j}}(f_{j}, x) = \sum_{k_{1}=0}^{n_{j}} \dots \sum_{k_{p}=0}^{n_{j}} f_{j}(\frac{k_{1}}{n_{j}}, \dots \frac{k_{p}}{n_{j}}) \ p_{n_{j}, k_{1}}(x_{1}) \dots p_{n_{j}, k_{p}}(x_{p}).
\end{equation*}
The $n_{j}$ are chosen such that
\begin{equation} \label{choicen}
   \bigg(1 + \frac{p}{4} \bigg) \  \sum_{j=1}^{m} \omega(f_{j}, \frac{1}{\sqrt{n_{j}}}) \ \leq \epsilon.
\end{equation}
Each $B_{n_{j}}(f_{j}, \, \cdot \,)$ is written as a sum of univariate polynomials thanks to lemma \ref{MultiAsUni}. For each $j \in \lbrace 1, \dots, m \rbrace$, we have a family of vectors of $\lbrace a_{j,1}, \dots, a_{j,r} \rbrace \subset \mathbb{R}^{p}$ and univariate polynomials $\lbrace p_{j,1}, \dots, p_{j,r}\rbrace$ of degree at most $n_{j}$ such that for all $x \in [0,1]^{p}$
\begin{equation} \label{BernAsUnivariate}
    B_{n_{j}}(f_{j},x) = \sum_{i=1}^{r} p_{\tiny{j,i}} \big( \langle a_{j,i}, x \rangle \big).
\end{equation}
We now focus on how to approximate a given polynomial $p: \mathbb{R} \rightarrow \mathbb{R}$ by a function in $\mathcal{N}(\sigma)$, i.e by using a shallow neural network. This will tell us how to approximate the univariate polynomials in (\ref{BernAsUnivariate}). For $z \in \mathbb{R}$, write $p(z)$ as $p(z) = \sum_{k=0}^{n} b_{k}z^{k}$. Since $\sigma \in \mathcal{C}^{\infty}(\mathbb{R})$ and $\sigma$ is not a polynomial, there exists $\theta_{0} \in \mathbb{R}$ such that, for all $k \in \mathbb{N}$, $\sigma^{(k)}(-\theta_{0}) \neq 0$.
Moreover, we have
\begin{equation*}
    \frac{d^{k}}{dw^{k}} \sigma(wz - \theta)_{\big|_{ \scriptstyle w=0, \ \theta = \theta_{0}}} = z^{k} \sigma^{(k)}(- \theta_{0}).
\end{equation*}
Therefore, $p(z)$ can be written as
\begin{equation} \label{pDer}
    p(z) = \sum_{k=0}^{n} \frac{b_{k}}{\sigma^{(k)}(- \theta_{0})} \ \frac{d^{k}}{dw^{k}} \sigma(wz - \theta)_{\big|_{ \scriptstyle w=0, \ \theta = \theta_{0}}}.
\end{equation}
We now need to approximate the derivative terms. This is done via finite differences. Define
\begin{equation} \label{finiteDiff}
    \Delta_{h}^{k} [\sigma] (w=0, z, \theta) := \sum_{i=0}^{k} 
    \binom{k}{i}
    (-1)^{k-i} \ \sigma \big( (k-i)hz - \theta \big).
\end{equation}
Then, we have
\begin{equation} \label{approxDer}
    \frac{d^{k}}{dw^{k}} \sigma(wz - \theta)_{\big|_{ \scriptstyle w=0, \ \theta = \theta_{0}}} = \ \frac{\Delta_{h}^{k} [\sigma] (w=0, z, \theta_{0})}{h^{k}} + O(h),
\end{equation}
as $h \rightarrow 0$. Therefore, for all $z \in \mathbb{R}$, $p(z)$ is of the form
\begin{equation} \label{pfinal}
    p(z) = \sum_{l=0}^{N} c_{l} \ \sigma(w_{l}z - \theta_{0}) + O(h),
\end{equation}
as $h \rightarrow 0$ and where $w_{l} = 0$ or $w_{l} = (k-i)h$ for some $k$ and $i$. In (\ref{pfinal}), $N$ corresponds to the width of the one-hidden layer network that approximates $p$. Since there are a lot of overlaps in the evaluation of $\sigma$ in (\ref{finiteDiff}), we simply obtain
\begin{equation*}
    N = n,
\end{equation*}
where we recall that $n$ is the degree of $p$. This is an equality as the error term $O(h)$ in (\ref{approxDer}) can be made arbitrarily small with no extra cost regarding the width of the network.

We now return to the polynomials $B_{n_{j}}(f_{j})$. From the previous step, each polynomial $p_{j,i}$ that appears in (\ref{BernAsUnivariate}) can be computed with $n_{j}$ neurons. Since there are $r$ terms in the sum, this gives a total width of
\begin{equation*}
    W_{j} = r n_{j} = 
    \binom{p -1 +n_{j}}{n_{j} } 
    n_{j}.
\end{equation*}

The idea behind the deep network $g$ constructed to prove \citep[Theorem 3.2]{kidger2019universal} is to make shallow networks vertical. Each $f_{j}$ is approximated by a shallow network $g_{j}$. Each $g_{j}$ is then made vertical and these verticalised networks are then glued together to obtain the network $g$ --see the original paper for more details. The network $g$ outputs $(g_{1}, \dots , g_{m})$ and its depth is the sum of the widths of the $g_{j}$, since one neuron in $g_{j}$ corresponds to one layer in $g$. The networks $g_{j}$ can be taken to be the networks computing each $B_{n_{j}}(f_{j})$. Hence, from the previous step, the depth of $g$ is
\begin{equation*}
    D = \sum_{j=1}^{m}  n_{j} 
    \binom{p -1 +n_{j} }{n_{j}}
    .
\end{equation*}
If we set $n_{j}=n$ for all $j$ and if $n$ is chosen accordingly to (\ref{choicen}) with the $n_{j}$ replaced by $n$ there, then, by proposition \ref{prop_BernApprox},
\begin{equation*}
    \| f - g \|_{\infty} \leq \sum_{j=1}^{m} \| f_{j} - g_{j} \|_{\infty} 
    \leq \sum_{j=1}^{m} \| f_{j} - B_{n}(f_{j}) \|_{\infty} 
    \leq \bigg( 1 + \frac{p}{4} \bigg) \sum_{j=1}^{m} \omega(f_{j}, \frac{1}{\sqrt{n}})
    \leq \epsilon .
\end{equation*}
Therefore, the network $g$ is such that $\| g -f \|_{\infty} \leq \epsilon$ and has depth 
$mn
\binom{p-1 +n}{n}
$. Observe that
\begin{equation*}
    m n 
    \binom{n+p-1}{n}
    = mn \frac{\prod_{k=1}^{n}\big( k+p-1 \big)}{n !} 
    = mn \bigg[ \prod_{k=1}^{n} \frac{k+p-1}{k} \bigg]  
    = mn \bigg[ \prod_{k=1}^{n} 1+ \frac{p-1}{k} \bigg].
\end{equation*}
Moreover, we have
\begin{equation*}
    \exp \bigg( \log \bigg[ \prod_{k=1}^{n} 1+\frac{p-1}{k} \bigg] \bigg) = \exp \bigg( \sum_{k=1}^{n} \log \big( 1 + \frac{p-1}{k}\big) \bigg)
    = O(n^{p-1}).
\end{equation*}
Therefore, the depth of $g$ behaves like $O(mn^{p})$. Thanks to (\ref{choicen}) and since for all $j$, $\omega(f_{j}, \, \cdot \,) \leq \omega(f, \, \cdot \,)$ $n$ can be related to $\omega(f, \, \cdot \,)$ and $\epsilon$, showing the estimate of proposition \ref{propsmmoth}.

\end{proof}

\subsubsection{Depth estimates for a non-polynomial continuous activation function}

To derive depth estimates in the case of a non-polynomial continuous activation function, we use convolutions to smooth the activation function thereby obtaining derivatives of all order that are used to approximate polynomials. Convolutions themselves are approximated by sums and the following lemma quantifies the error made when using such an approximation schema.

\begin{lemma} \label{errorconv}
Let $\epsilon >0$. Let $\sigma: \mathbb{R} \rightarrow \mathbb{R}$ be a non-polynomial continuous function and $\phi: \mathbb{R} \rightarrow \mathbb{R}$ be an infinitely differentiable function having support in an interval $[a, b]$. If $L \in \mathbb{N}$ is such that
\begin{equation*}
    \| \phi \| _{\footnotesize L^{1}} \omega(\sigma, \frac{b-a}{L}) \leq \epsilon,
\end{equation*}
then there exist $c_{l} \in \mathbb{R}$, $y_{l} \in [a+(l-1)\frac{b-a}{L}, \ a+l \frac{b-a}{L}]$ for $l=1, \dots, L$ such that
\begin{equation*}
    \big \vert \ (\sigma * \phi) (t) - \sum_{l=1}^{L} c_{l} \sigma(t- y_{l}) \ \big \vert \leq \epsilon.
\end{equation*}
\end{lemma}

\begin{proof}
This proof is inspired by the beginning of \citep[Chapter 24]{Cheney}. Define the intervals
\begin{equation*}
    I_{l} := \left[ a+(l-1)\frac{b-a}{L}, \ a+l \frac{b-a}{L} \right], \quad 1 \leq l \leq L.
\end{equation*}
They form a partition of $[a,b]=\text{supp} \ \phi$. Set
\begin{equation} \label{coeffRiemann}
    c_{l} = \int_{I_{l}} \phi(y) \ dy, \quad 1 \leq l \leq L.
\end{equation}
Let $y_{l}$ be in $[a+(l-1)\frac{b-a}{L}, a+l \frac{b-a}{L}]$ for $l=1, \dots, L$. Then,
\begin{align*}
    \big \vert \ \sigma * \phi (t) - \sum_{l=1}^{L} c_{l} \sigma(t-y_{l}) \ \big \vert &= \bigg \vert \int_{\text{supp} \phi} \sigma(t-y) \ \phi(y) \, \mathrm{d}y - \sum_{l=1}^{L} \sigma(t-y_{l}) \int_{I_{l}} \phi(y) \, \mathrm{d}y \bigg \vert \\
    & \leq \sum_{l=1}^{L} \int_{I_{l}} \vert \sigma(t-y) - \sigma(t-y_{l})\vert \ \vert \phi(y) \vert \, \mathrm{d}y \\
    & \leq \omega \big(\sigma, \frac{b-a}{L} \big) \ \bigg( \sum_{l=1}^{L} \int_{I_{l}} \vert \phi(y) \vert \, \mathrm{d}y \bigg) \\
    &= \| \phi \|_{L^{1}} \omega \big(\sigma, \frac{b-a}{L} \big) \\
    & \leq \epsilon \quad \text{by assumption.}
\end{align*}
\end{proof}

By revisiting the proof of proposition \ref{propsmmoth} and using lemma \ref{errorconv}, we can derive the following quantitative depth estimate for deep neural networks with a non-polynomial continuous activation function.

\begin{proposition} \label{propcontinuous}
Let $\sigma : \mathbb{R} \rightarrow \mathbb{R}$ be a non-polynomial continuous function which is continuously differentiable at at least one point with a nonzero derivative at that point. Let $f: [0,1]^{p} \rightarrow \mathbb{R}^{m}$ be a continuous function. Then, for any $\epsilon > 0$, there exists $g \in \mathcal{NN}_{p,m,p+m+2}^{\sigma}$ such that $\| f - g \|_{\infty} \leq \epsilon$. The depth of $g$ is of order
\begin{equation}
        O \bigg( m \bigg(\omega^{-1} \big(f, \frac{\epsilon}{2m(1+\frac{p}{4})} \big) \bigg)^{-2p} \bigg( \omega^{-1} \big( \sigma, \frac{\epsilon}{2Bm(2^{\omega^{-1}(f, \frac{\epsilon}{2m(1+\frac{p}{4})})^{-2}+1} -1)} \big) \bigg)^{-1}\bigg)
    \end{equation}
for some $B > 0$ depending on $f$.
\end{proposition}

\begin{proof}
Let $f:[0,1]^{p} \rightarrow \mathbb{R}^{m}$ be a continuous function and let $\sigma: \mathbb{R} \rightarrow \mathbb{R}$ be a function satisfying the assumptions of Proposition \ref{propcontinuous}. To prove the result, we only need to refine the part of the proof of Proposition \ref{propsmmoth} that involves the differentiability of $\sigma$. To write the Bernstein polynomials $B_{n_{j}}(f_{j})$ as functions in $\mathcal{N}(\sigma)$, we use the idea of the proof of \citep[Theorem 3.1]{PinkusMLP}.

Let $\phi \in \mathcal{C}^{\infty}_{c}(\mathbb{R})$ be such that $\sigma * \phi$ is not a polynomial. Such a function $\phi$ does exist, see for example the proof of \citep[Proposition 3.7]{PinkusMLP}. By standard properties of convolution, $\sigma * \phi$ belongs to $\mathcal{C}^{\infty}(\mathbb{R})$ and by considering mollifiers, we can choose a sequence $(\phi_{n})_{n}$ such that $\sigma * \phi_{n}$ converges uniformly to $\sigma$ on compact sets. For $t \in \mathbb{R}$,
\begin{equation*}
    \sigma * \phi \ (t) = \int_{ - \infty}^{+ \infty} \sigma(t-y) \phi(y) \, \mathrm{d}y 
    = \int_{\text{supp} \ \phi} \sigma(t-y) \phi(y) \, \mathrm{d}y,
\end{equation*}
where $\text{supp} \ \phi$ is compact. We can assume that it is an interval of $\mathbb{R}$ by choosing an appropriate mollifier. Moreover, for any integer $k \geq 1$:
\begin{equation*}
    \frac{d^{k}}{dw^{k}} (\sigma * \phi) (wt - \theta) = \int_{ - \infty}^{+ \infty} \sigma(y) t^{k}  \phi ^{(k)}(wt - \theta - y) \, \mathrm{d}y. 
\end{equation*}
Since $\sigma * \phi$ is not a polynomial, there exits $\theta_{0} \in \mathbb{R}$ such that for all $k \in \mathbb{N}$, $\sigma * \phi^{(k)}(-\theta_{0}) \neq 0$. Then:
\begin{equation} \label{convolutionDer}
    \frac{d^{k}}{dw^{k}} ( \sigma * \phi )(wt - \theta)_{\big|_{ \scriptstyle w=0, \ \theta = \theta_{0}}} = \int_{ - \infty}^{+ \infty} \sigma(y) t^{k} \phi ^{(k)} (- \theta_{0} - y) \, \mathrm{d}y.
\end{equation}
Therefore, in order to express (\ref{BernAsUnivariate}) as a function in $\mathcal{N}(\sigma)$, the right-hand side of (\ref{convolutionDer}) must be approximated as a sum. It is first approximated via finite differences:
\begin{equation*}
    \frac{d^{k}}{dw^{k}} (\sigma * \phi) (wt - \theta)_{\big|_{ \scriptstyle w=0, \ \theta = \theta_{0}}} = \frac{1}{h^{k}} \sum_{i=0}^{k} 
    \binom{k}{i}
    (-1)^{k-i} \ (\sigma * \phi) \big((k-i)hz - \theta_{0} \big) + O(h)
\end{equation*}
as $h \rightarrow 0$. Then, each term $\sigma * \phi $ is further approximated as a Riemann sum in the manner of Lemma \ref{errorconv}. This is possible since $\phi$ has compact support, hence so have the $\phi ^{(k)}$. For $L \in \mathbb{N}$, $t \in \mathbb{R}$, we have:
\begin{equation} \label{RiemannSum}
   (\sigma * \phi) (t) \sim \sum_{l=1}^{L} c_{l} \sigma(t- y_{l}).
\end{equation}
Let $p: \mathbb{R} \rightarrow \mathbb{R}$ be a polynomial of degree $n$, written as $p(z) = \sum_{k=0}^{n} b_{k}z^{k}$. Based on the above analysis, we can approximate it by an element of $\mathcal{N}(\sigma)$ and estimate the error of approximation. Observe that
\begin{equation} \label{pwithsigma}
    p(z) = \sum_{k=0}^{n} \tilde b_{k} \frac{\mathrm{d}^{k}}{\mathrm{d}w^{k}} (\sigma * \phi)(wz - \theta)_{\big|_{ \scriptstyle w=0, \ \theta = \theta_{0}}}  \quad \text{where} \ \tilde b_{k} =  \left( \int_{ - \infty}^{+ \infty} \sigma(y) \ \phi ^{(k)}(- \theta_{0} - y) \, \mathrm{d}y \right)^{-1} b_{k}.
\end{equation}
The terms involving the derivatives of $\sigma * \phi$ are first approximated by finite differences. Then, the terms $\phi * \sigma$ that appear in the finite differences are further approximated by Riemann sums. This yields for $p$
\begin{equation} \label{pRiemann}
    p(z) \sim \sum_{k=0}^{n} \tilde b_{i} \sum_{i=0}^{k} 
    \binom{k}{i}
    (-1)^{k-i} \sum_{l=1}^{L} c_{l} \sigma \big( (k-i)hz - \theta_{0} - y_{l} \big).
\end{equation}
Denote the right-hand side of (\ref{pRiemann}) by $R(z)$. Notice that if we assume that $z$ belongs to a compact set $K$, then we can choose $\phi$ such that $\| \sigma - \sigma * \phi \|_{\infty} $ is arbitrarily small. Here, $\| \, \cdot \, \|_{\infty}$ is taken on the set
\begin{equation*}
    \lbrace (k-i)hz - \theta_{0}: z \in K, i=0,\dots, k \rbrace = \bigcup_{i=0}^{k} \big \lbrace (k-i)hz - \theta_{0}: z \in K \big \rbrace
\end{equation*}
which is compact since this is a finite union of compact sets. By (\ref{pRiemann}), approximating $p(z)$ by $R(z)$ requires $Ln$ evaluations of $\sigma$. To use $R(z)$ to derive depth estimates, it remains to investigate the error made when approximating $p(z)$ by $R(z)$.  By Lemma \ref{errorconv}, we have, for $z \in K$, 
\begin{equation*}
    p(z) -  R(z) \leq \sum_{k=0}^{n} \tilde b_{i} \sum_{i=0}^{k} 
    \binom{k}{i}
    (-1)^{k-i} \ (-1)^{k-i} \| \phi \|_{L^{1}} \omega \big( \sigma, \frac{b-a}{L}\big).
\end{equation*}
The right-hand term is bounded above by
\begin{equation*}
    (2^{n+1}-1) \| \phi \|_{L^{1}} \omega \big( \sigma, \frac{b-a}{L}\big) \max_{0 \leq i \leq n} \tilde b_{i}.
\end{equation*}
Similarly, we obtain a lower bound and this yields the estimate, for $z \in K$,
\begin{equation} \label{errorBound}
    \big \vert \ p(z) -  R(z) \big \vert \leq (2^{n+1}-1) \| \phi \|_{L^{1}} \omega \big( \sigma, \frac{b-a}{L}\big) \max_{0 \leq i \leq n} \vert \tilde b_{i} \vert .
\end{equation}
We now come back to the setting of the proof of Proposition \ref{propcontinuous}. We make the approximation
\begin{equation} \label{approxBern}
    B_{n_{j}}(f_{j}, x) \sim \sum_{i=0}^{r} \sum_{k=0}^{n_{j}} \tilde b_{i,k} \sum_{l=0}^{k} 
    \binom{k}{l}
    (-1)^{k-l} \sum_{l'=1}^{L} c_{l'} \sigma \big( (k-l)h\langle a_{j,i}, z \rangle - \theta_{0} - y_{l'} \big)
\end{equation}
for some coefficients $\tilde b_{i,k}$ and where the $c_{l'}$ are defined in (\ref{coeffRiemann}). The coefficients $a_{j,i}$ come from Lemma \ref{MultiAsUni}. Notice that the function $\phi$ can be chosen such that $\| \sigma - \sigma * \phi \|_{\infty}$ is arbitrarily small where $\| . \|_{\infty}$ is considered on the compact set
$\cup_{i=0}^{r} \cup_{l=0}^{k}  \big \{ (k-l)h \langle a_{j,i}, z\rangle - \theta_{0}: z \in [0,1]^{p} \big \}$.
Let $g$ denote the neural network approximating $f$ and write $g=(g_{1}, \dots, g_{m})$. Each $g_{j}$ corresponds to the approximation of $B_{n_{j}}(f_{j})$ in (\ref{approxBern}). We have, by Proposition \ref{prop_BernApprox},
\begin{align*}
    \| f - g \|_{\infty} \leq \sum_{j=1}^{m} \| f_{j} - g_{j} \|_{\infty} 
     &\leq \sum_{j=1}^{m} \| f_{j} - B_{n_{j}}(f_{j}) \|_{\infty} +  \|B_{n_{j}}(f_{j}) - g_{j} \|_{\infty} \\
     & \leq  \sum_{j=1}^{m} \bigg( 1 + \frac{p}{4}\bigg) \omega(f_{j}, \frac{1}{\sqrt{n_{j}}})  + \|B_{n_{j}}(f_{j}) - g_{j} \|_{\infty}.
\end{align*}
Let $\epsilon > 0$. If $n$ and $L$ are chosen such that
\begin{equation} \label{condnL}
     \mathlarger{\mathlarger{\sum}}_{j=1}^{m} \bigg[ \big( 1 + \frac{p}{4} \big) \omega(f_{j}, \frac{1}{\sqrt{n}_{j}})+ \omega \big( \sigma, \frac{b-a}{L}\big) \| \phi \|_{L^{1}} \big( 2^{n_{j}+1}-1 \big) \sum_{i=0}^{r} \max_{0 \leq k \leq n_{j}} \tilde \vert b_{i,k} \vert \bigg] \leq \epsilon,
\end{equation}
then the network $g$ is such that $\| f - g \|_{\infty} \leq \epsilon$. Its depth is the sum of the depth of the $g_{j}$. Typically, if $\phi$ is a mollifier, then $\| \phi \|_{L^{1}}=1$. By choosing $n_{j}=n$ for all $j$ and by taking each term in (\ref{condnL}) smaller than $\epsilon/2$, we obtain the depth estimate of Proposition \ref{propcontinuous}.
\end{proof}

\begin{remark}
We see in (\ref{condnL}) that it would be useful to have an idea of the magnitude of the coefficients $\tilde b_{i,k}$. However, these coefficients depend on the decomposition given by Lemma \ref{MultiAsUni} of the Bernstein polynomials. To the best of our knowledge, the proof of this decomposition is not constructive and hence it does not allow us to control the magnitude of the $\tilde b_{i,k}$. This decomposition is actually closely related to the Waring problem which is still not totally solved. However, it can be mentioned that the value $r$ given in (\ref{defr}) is an upper bound of the exact number of terms in the decomposition. \cite{Alexander} show a smaller estimate for $r$. This could result in a decrease of the depth. As for the coefficients $\tilde b_{i,k}$, \cite{Dreesen} propose an algorithm to construct the decomposition of Lemma \ref{MultiAsUni} which relates these coefficients to the Jacobian matrix of the multivariate polynomial being approximated. So it could be that the depth of a network approximating a function further depends on its actual smoothness, i.e not only on its first-order modulus of continuity.
\end{remark}

\subsubsection{Depth estimates for a non-affine polynomial activation function}
\paragraph{When one extra neuron on each layer is allowed}
By building a deep network whose construction does not rely on the universal approximation theorem, we can obtain the following quantitative result in the case of a polynomial activation function.

\begin{proposition} \label{proppoly}
Let $\sigma : \mathbb{R} \rightarrow \mathbb{R}$ be a non-affine polynomial. Let $f: [0,1]^{p} \rightarrow \mathbb{R}^{m}$ be a continuous function. Then, for any $\epsilon > 0$, there exists $g \in \mathcal{NN}_{p,m,p+m+3}^{\sigma}$ such that $\| f - g \|_{\infty} \leq \epsilon$. The depth of $g$ is of order
\begin{equation*}
    O \bigg(m(p+m)\bigg(\omega^{-1} \big(f, \frac{\epsilon}{(1+\frac{p}{4})m} \big) \bigg)^{-4p-2} \bigg).
\end{equation*}
\end{proposition}

\begin{proof}
Let $\sigma: \mathbb{R} \rightarrow \mathbb{R}$ be a non-affine polynomial. Let $f: [0,1]^{p} \rightarrow \mathbb{R}^{m}$ be a continuous function. Verticalisation of shallow networks cannot be used anymore to construct a deep network approximating $f$: the universal approximation theorem rules out polynomial activation functions. To overcome this issue, an inspection of the proof of \citep[Theorem 3.2]{kidger2019universal} shows that it suffices to approximate each component of $f$ by a polynomial and then count the number of multiplications that are required to evaluate each polynomial. To make the estimates meaningful, in Propositions \ref{maindepth} and \ref{proppoly}, we allow an extra neuron on each layer, that is we consider feedforward neural networks with width $p+m+3$. The case of networks with width $p+m+2$ is treated in appendix \ref{InverseEst}:  \cite{kidger2019universal} use an approximation of the inverse function $x \mapsto 1/x$ to restrict the width to $p+m+2$, which results in an increase of the depth.

Let $\epsilon > 0$. Each component $f_{j}$ of $f$ is approximated by a Bernstein polynomial $B_{n}(f_{j})$ of degree $n$ with $n$ such that
\begin{equation} \label{choicenpoly}
    \bigg(1 + \frac{p}{4} \bigg) \sum_{j=1}^{m} \omega(f_{j}, \frac{1}{\sqrt{n}}) \leq \epsilon.
\end{equation}
The network $g$ approximating $f$ computes these $m$ polynomials. Its depth is given by the number of multiplications necessary to compute the monomials constituting the $B_{n}(f_{j})$, for $j=1, \dots, m$, see \cite{kidger2019universal}. By Definition \ref{BernPolydef}, each $B_{n}(f_{j})$ can be rewritten as
\begin{align*}
    B_{n}(f_{j}, x) &= \sum_{k_{1}=0}^{n} \dots \sum_{k_{p} = 0}^{n} f_{j} \left( \frac{k_{1}}{n} \dots \frac{k_{p}}{n} \right) 
    \binom{n}{{k_{1}}}
    \sum_{c_{1}=0}^{n-k_{1}} 
    \binom{n-k_{1}}{c_{1}}
    (-1)^{n-k_{1}-c_{1}} x_{1}^{n-c_{1}} \\
    & \dots 
    \binom{n}{{k_{p}}}
    \sum_{c_{p}=0}^{n-k_{p}} 
    \binom{n-k_{p}}{c_p}
    (-1)^{n-k_{p}-c_{p}} x_{p}^{n-c_{p}} .
\end{align*}
When $k_{1},...,k_{p}$ and $c_{1},...,c_{p}$ are fixed, we have a monomial given by
\begin{equation*}
    \gamma (x) = f_{j} \left( \frac{k_{1}}{n}, \dots, \frac{k_{p}}{n} \right) \prod_{j=1}^{p} 
    \binom{n}{k_{j}}
    \binom{n-k_{j}}{c_j}
    (-1)^{n-k_{j}-c_{j}} x_{j}^{n-c_{j}}. 
\end{equation*}
Computing $x_{j}^{n-c_{j}}$ requires $n-c_{j}-1$ multiplications if $n-c_{j} \geq 1$, $0$ otherwise. Let $A$ denote the set $\lbrace c_{j}: n-c_{j} = 0, \ j=1, \dots, p \rbrace$. Taking into account multiplications between powers of coordinates, computing $\gamma$ requires
\begin{equation} \label{nbmonommial}
    \bigg( \sum_{j=1}^{p}(n-c_{j}-1) \mathbbm{1}_{c_{j} < n} \bigg) +  \sum_{j=1}^{p} \big( \mathbbm{1}_{c_{j} < n} \big) - \mathbbm{1}_{\vert A \vert > 1}  = \bigg( \sum_{j=1}^{p} n-c_{j} \bigg) - \mathbbm{1}_{\vert A \vert > 1}
\end{equation}
multiplications. Above, $\vert A \vert$ denotes the cardinality of $A$. We now let the values of $k_{1},...,k_{p}$ and $c_{1},...,c_{p}$ vary to obtain all the monomials appearing in $B_{n}(f_{j})$. Let $M$ denote the total number of monomials and let $M_{0}$ be the number of monomials where at most one coordinate $x_{j}$ has a non-zero power. From (\ref{nbmonommial}), computing all these monomials requires
\begin{equation*}
    \bigg( \sum_{k_{1}=0}^{n} \dots \sum_{k_{p}=0}^{n} \sum_{c_{1}=0}^{n-k_{1}} \dots \sum_{c_{p}=0}^{n - k_{p}} \sum_{j=1}^{p} \big( n-c_{j} \big) \ \bigg) - M + M_{0}
\end{equation*}
multiplications. Let $P$ denote the expression above. It can be shown that
\begin{equation*}
    P = O(n^{2p+1}),
\end{equation*}
see appendix \ref{EstMonomial} for the details. By \citep[Lemma 4.3]{kidger2019universal}, a multiplication can be computed by two neurons with square activation function $\rho (x) = x^{2}$. This implies that a neural network with square activation must have $2 \times P$ layers to approximate $B_{n}(f_{j})$. Therefore, a network approximating $f = (f_{1}, ..., f_{m})$ should have depth $2m \times P$ and width $p+m+1$, where the value for the width is explained in \cite{kidger2019universal}. Moreover, the proof of \citep[Proposition 4.11]{kidger2019universal} shows that the operation of a single neuron with square activation function may be approximated by two neurons with activation function $\sigma$. Therefore, a computation neuron needs to be added to each layer. Since we also add an extra neuron to each layer, the same method can be used to approximate all the $f_{j}$, i.e use Bernstein polynomials and simply count the number of multiplications that their evaluation requires. In order to have access to the approximated values of the inputs and to the intermediate values of the $B_{n}(f_{j})$'s, we make the network with square activation function "vertical", i.e each layer in this model becomes $p+m+1$ layers in the network with activation function $\sigma$. Therefore, the network $g \in \mathcal{NN}_{p,m,p+m+3}^{\sigma}$ approximating $f$ has depth
\begin{equation} \label{depthpoly}
    2m(p+m+1) O(n^{2p+1}) = O(m(p+m) n^{2p+1}).
\end{equation}
Moreover, by (\ref{choicenpoly}) and Proposition \ref{prop_BernApprox}, $g$ is such that $\| f - g \|_{\infty} \leq \epsilon$. By (\ref{choicenpoly}) and since $\omega(f_{j}, \, \cdot \,) \leq \omega(f, \, \cdot \,)$ for all $j$, $n$ can be related to $\omega(f, \, \cdot \,)$ and $\epsilon$, showing the estimate of Proposition \ref{proppoly}.
\end{proof}

\paragraph{When no extra neuron is allowed}\label{InverseEst}
The proof of \citep[Proosition 4.11]{kidger2019universal} relies on an approximation of the inverse function $x \mapsto 1/x$ to avoid adding an extra neuron to each layer, i.e to approximate a function $f \in \mathcal{C}([0,1]^{p}, \mathbb{R}^{m})$ by a network $g \in \mathcal{NN}_{p,m,p+m+2}^{\sigma}$. This increases the depth of $g$, as shown by the following proposition.

\begin{proposition} \label{prop_poly_noextra}
Let $\sigma : \mathbb{R} \rightarrow \mathbb{R}$ be a non-affine polynomial. Let $f: [0,1]^{p} \rightarrow \mathbb{R}^{m}$ be a continuous function. Then, for any $\epsilon > 0$, there exists $g \in \mathcal{NN}_{p,m,p+m+2}^{\sigma}$ such that $\| f - g \|_{\infty} \leq \epsilon$. For $0<\alpha < 1$ small enough so that $[0,1] + 1 - \alpha \subset [\frac{1}{2-\alpha}, 2-\alpha]$, the depth of $g$ is of order
\begin{align*}
    &O\bigg(p(p+m) \big[ \omega^{-1} \big(f, \frac{\epsilon }{2m(1+\frac{p}{4})}\big) \big]^{-4p} \log \bigg( \log \big(\frac{\epsilon (2-\alpha)}{2\big[ \omega^{-1}\big(f, \frac{\epsilon}{2m(1+\frac{p}{4})}\big) \big]^{-8p}} \big) [ \log(1-\alpha) ]^{-1} \bigg) \bigg) \\
    &+O\bigg( m(p+m) \big[ \omega^{-1}\big(f, \frac{\epsilon}{2m(1+\frac{p}{4})}\big) \big]^{-4p-2} \bigg).
\end{align*}
\end{proposition}

\begin{proof}
Let $f \in \mathcal{C}([0,1]^{p}, \mathbb{R}^{m})$ and let $\epsilon > 0$. First, for $\alpha > 0$ as in the statement of Proposition \ref{prop_poly_noextra}, we perform a change of variables $\tilde x_{i} = x_{i} + 1 - \alpha, \ i=1, \dots , p$,  so that the function $f$ is defined on $[1- \alpha, 2 - \alpha]$ and the inverse function $\tilde x_{i} \mapsto 1/ \tilde x_{i}$ is well-defined for all $i=1, \dots, p$. Each component $f_{j}$, $j=1, \dots, m$, of $f$ is approximated by a Bernstein polynomial $B_{n}(f_{j})$. The polynomials $B_{n}(f_{j})$, $j=2, \dots, m$, are computed by a network with activation function $\sigma$ satisfying the assumption of Proposition \ref{prop_poly_noextra}, as in the proof of Proposition \ref{proppoly}. For $B_{n}(f_{1})$, the technique used by \cite{kidger2019universal} is first to decompose $B_{n}(f_{1})$ into a sum of monomials. We obtain
\begin{equation*}
    M := \bigg( \bigg( \frac{n}{2} + 1 \bigg) \big( n + 1\big) \bigg)^{p}
\end{equation*}
monomials, see appendix \ref{EstMonomial}. Then $B_{n}(f_{1})$ is written as compositions of multiplications and evaluations of $r_{n_{a}}$, where $r_{n_{a}}$ is defined as, for $n_{a} \in \mathbb{N}$,
\begin{equation*}
    r_{n_{a}}(x) = (2-x)\prod_{i=1}^{n_{a}}(1+(1-x)^{2^{i}})
\end{equation*}
and converges to $1/x$ as $n_{a} \rightarrow \infty$. An inspection of the proof of \citep[Lemma 4.5]{kidger2019universal} shows that $r_{n_{a}}(\tilde x_{i})$ for $i=1, \dots, p$ can be computed using $3n_{a}$ layers. For clarity, we recall the decomposition used in \cite{kidger2019universal}
\begin{align} \label{decomposition}
    \tilde g_{1} = &\left[ \prod_{k=1}^{p} r_{n_{a}}^{2M-2}(x_{k})^{\theta_{1,k}} \right] \bigg( 1 + \left[\prod_{k=1}^{p} r_{n_{a}}^{2M-3}(x_{k})^{\theta_{1,k}} \right] \left[\prod_{k=1}^{p} r_{n_{a}}^{2M-4}(x_{k})^{\theta_{2,k}} \right] \nonumber \\
    & \bigg( 1 + \left[\prod_{k=1}^{p} r_{n_{a}}^{2M-5}(x_{k})^{\theta_{2,k}} \right] \left[\prod_{k=1}^{p} r_{n_{a}}^{2M-6}(x_{k})^{\theta_{3,k}} \right] \nonumber \\
    & \bigg( \dots \nonumber \\
    & \bigg( 1 + \left[\prod_{k=1}^{p} r_{n_{a}}^{3}(x_{k})^{\theta_{M-2,k}} \right] \left[\prod_{k=1}^{p} r_{n_{a}}^{2}(x_{k})^{\theta_{M-1,k}} \right] \\
    & \bigg( 1 + \left[\prod_{k=1}^{p} r_{n_{a}}(x_{k})^{\theta_{M-1,k}} \right] \left[\prod_{k=1}^{p} x_{k}^{\theta_{M,k}} \right] \bigg)\bigg) \nonumber \\
    & \dots \bigg) \bigg) \bigg) \nonumber
\end{align}
where $r_{n_{a}}^{\alpha}$ denotes $r_{n_{a}}$ composed $\alpha$ times. Only one neuron is available to compute (\ref{decomposition}), so this neuron has to store intermediate operations. We see that once $r_{n_{a}}$ is computed for all $\tilde x_{i}$, with only one neuron to store values, computing the most nested set of brackets requires
\begin{equation*}
    \big( \sum_{k=1}^{p}\theta_{M,k} - 1 \big) + \big( \sum_{k=1}^{p}\theta_{M-1,k} -1 \big) +1 =  \sum_{k=1}^{p} \big( \theta_{M,k} + \theta_{M-1,k} \big) - 1
\end{equation*}
multiplications. On the left-hand side, the $+1$ comes from the fact that once the two products have been computed, they still need to be multiplied together. Above, we have assumed that $\theta_{M,k}, \ \theta_{M-1,k} \geq 1$ for at least two different values of $k$. That is why there is an extra $-1$. We then compute $r \circ r(\tilde x_{i})$ and $r_{n_{a}} \circ r_{n_{a}} \circ r_{n_{a}}(\tilde x_{i})$ for all $i$ which requires $2 \times 3pn_{a}$ layers. Once $r_{n_{a}}^{2}$ and $r_{n_{a}}^{3}$ are computed, computing the second most nested set of brackets implies
\begin{equation*}
    \sum_{k=1}^{p} \big( \theta_{M-2,k} + \theta_{M-1,k} \big) - 1
\end{equation*}
multiplications, where again we have assumed that $\theta_{M-1, k}, \ \theta_{M-2, k} \geq 1$ for at lest two distinct values of $k$. Assume that $\theta_{M,k} \geq 1$ for at least two different values of $k$. This simplifies a bit the computations and does not increase depth. Reproducing the previous reasoning for all nested brackets, we find that we need
\begin{equation} \label{nbmultiplicationR}
    (2M-2) \times 3pn_{a} + 2 \left(\sum_{m=1}^{M-1} \sum_{k=1}^{p} \theta_{M-m, k} \right) + \sum_{k=1}^{p} \theta_{M,k} - 2(M-M_{0}) +\frac{2M-2}{2}
\end{equation}
multiplications to compute (\ref{decomposition}). As in the proof of Proposition \ref{proppoly}, $M_{0}$ denotes the number of monomials where at most one coordinate has a non-zero power. For $m<M$, the powers $\theta_{m,k}, \ k=1, \dots, p$, appear twice in (\ref{decomposition}), explaining the factor 2 in front of the second term in the sum (\ref{nbmultiplicationR}). Notice that if we choose the $M$th polynomial to be the one where all the coordinates have power $n$, then the depth is a bit reduced. The last term in (\ref{nbmultiplicationR}) accounts for multiplications of each bracket with the previous stored value. Since each multiplication requires two layers and that we need to make the network with approximate square activation function vertical -- see proof of Proposition \ref{proppoly} --, the number of layers needed to approximate $B_{n}(f_{1})$ is
\begin{align} \label{depthreci}
    & (2M-2)(p+m+1) \times 3pn_{a} + 2(p+m+1) \bigg( 2 \sum_{m=1}^{M-1}  \sum_{k=1}^{p} \theta_{M-m, k} 
    +  \sum_{k=1}^{p} \theta_{M,k} \bigg) \nonumber \\
    &+ 2(p+m+1) (-M+M_{0}+1).
\end{align}
No factor 2 appears in front of $(2M-2) \times 3pn_{a}$ because $3n_{a}$ is already the exact number of layers necessary to compute $r$ -- see \citep[Lemma 4.5]{kidger2019universal}. The depth of the network (\ref{depthreci}) depends on the ordering of the monomials only through the choice of the $M$th monomial: this $M$th monomial should be chosen to be the one for which the sum of powers is maximal. From the results of appendix \ref{EstMonomial}, we roughly have
\begin{equation*}
    \bigg( 2 \sum_{m=1}^{M-1}  \sum_{k=1}^{p} \theta_{M-m, k}  +  \sum_{k=1}^{p} \theta_{M,k} \bigg) = 2 O(n^{2p+1}) = O(n^{2p+1})
\end{equation*}
while $M=O(n^{2p})$ and $M_{0} = O(n)$. From the proof of Proposition \ref{proppoly}, we know that for $j=2, \dots, p$ each component $g_{j}$ of $g$ has depth $O((p+m+1)n^{2p+1})$. Therefore, the depth of the network $g \in \mathcal{NN}_{p,m,p+m+2}^{\sigma}$ approximating $f$ is of order
\begin{equation} \label{depthreci2}
    3pn_{a} \ O((p+m)n^{2p}) + O(m(p+m)n^{2p+1}).
\end{equation}
The last term in this sum is of the same order as the order of the depth of a network $\tilde g \in \mathcal{NN}_{p,m,p+m+3}^{\sigma}$ approximating $f$. We thus see that the increase of depth due to removing one neuron comes from the approximation of the inverse function. To obtain a more precise estimate of the depth of $g$, we should investigate the rate of convergence in $n_{a}$ of $r_{n_{a}}$. First, we can notice that
\begin{equation*}
    \prod_{i=0}^{n_{a}} (1+x^{2^{i}})= \sum_{i=0}^{2^{n_{a}+1}-1} x^{i}.
\end{equation*}
This can be shown by induction. For $n_{a}=0$, the equality is clear. Now, assume that it holds for $n_{a}=k$. Then
\begin{equation*}
    \prod_{i=0}^{k+1} (1+x^{2^{i}})= (1+x^{2^{k+1}}) \prod_{i=0}^{k}(1+x^{2^{i}}) = (1+x^{2^{k+1}}) \sum_{i=0}^{2^{k+1}-1} x^{i} = \sum_{i=0}^{2^{k+2}-1} x^{i}.
\end{equation*}
This yields, for $x \in (0,2)$,
\begin{equation*}
    r_{n_{a}}(x) = \prod_{i=0}^{n_{a}}(1+(1-x)^{2^{i}}) = \sum_{i=0}^{2^{n_{a}+1}-1} (1-x)^{i} = \frac{1- (1-x)^{2^{n+1}}}{x}.
\end{equation*}
Therefore, the error of approximation is simply
\begin{equation*}
    \big \vert \prod_{i=0}^{n_{a}}(1+(1-x)^{2^{i}}) - \frac{1}{x} \big \vert = \big \vert \frac{-(1-x)^{2^{n_{a}+1}}}{x} \big \vert.
\end{equation*}
In the context of the proof, $x$ belongs to $[1-\alpha, 2-\alpha]$. Moreover, in (\ref{decomposition}), $r_{n_{a}}$ is used to approximate $x \mapsto 1/x$ and $x \mapsto x$ alternatively. Therefore, we need to consider the error over the interval $[\frac{1}{2-\alpha}, 2-\alpha]$, where we assume that $\alpha$ is taken to be less than $\frac{1}{2}(3-\sqrt{5})$. This makes explicit the 'small enough' in the statement of Proposition \ref{prop_poly_noextra}. Then, as soon as $n_{a}\geq 1$, the right-hand side of the above equality reaches its maximum over $[\frac{1}{2-\alpha}, 2-\alpha]$ at $x=2-\alpha$ and is therefore bounded above by
\begin{equation} \label{defeta}
    \eta := \frac{(1-\alpha)^{2^{n_{a}+1}}}{2-\alpha}. 
\end{equation}
In (\ref{decomposition}), the approximation $r_{n_{a}}$ is composed multiple times. Thus, we need to investigate more precisely the error $\| g_{1} - \tilde g_{1} \|_{\infty}$. By decomposing $g_{1}$ into a sum of monomials, that is $g_{1}=\sum_{m=1}^{M} \gamma_{m}$, we may write, for $x \in [1-\alpha,2-\alpha]^{p}$,
\begin{align} \label{eqerrorR}
    \big \vert \tilde g_{1}(x) - g_{1}(x) \big \vert &= \big \vert \prod_{k=1}^{p} r_{n_{a}}^{2M-2}(x_{k})^{\theta_{1,k}} - \gamma_{1}(x) \big \vert \nonumber \\
    & + \big \vert \prod_{k=1}^{p} r_{n_{a}}^{2M-2}(x_{k})^{\theta_{1,k}} \, \prod_{k=1}^{p} r_{n_{a}}^{2M-3}(x_{k})^{\theta_{1,k}} \, \prod_{k=1}^{p} r_{n_{a}}^{2M-4}(x_{k})^{\theta_{2,k}} - \gamma_{2}(x) \big \vert \nonumber \\
    &+ \big \vert \prod_{k=1}^{p} r_{n_{a}}^{2M-2}(x_{k})^{\theta_{1,k}} \, \prod_{k=1}^{p} r_{n_{a}}^{2M-3}(x_{k})^{\theta_{1,k}} \, \prod_{k=1}^{p} r_{n_{a}}^{2M-4}(x_{k})^{\theta_{2,k}} \prod_{k=1}^{p} r_{n_{a}}^{2M-5}(x_{k})^{\theta_{2,k}} \nonumber \\
    &\times \prod_{k=1}^{p} r_{n_{a}}^{2M-6}(x_{k})^{\theta_{3,k}}- \gamma_{3}(x) \big \vert+ \dots
\end{align}
For the first term, we write
\begin{equation*}
    \big \vert \prod_{k=1}^{p} r_{n_{a}}^{2M-2}(x_{k})^{\theta_{1,k}} - \gamma_{1}(x) \big \vert \leq  \sum_{j=1}^{M-1} \big \vert \prod_{k=1}^{p} r_{n_{a}}^{2M-2j}(x_{k})^{\theta_{1,k}} - \prod_{k=1}^{p} r_{n_{a}}^{2M-2(j+1)}(x_{k})^{\theta_{1,k}} \big \vert .
\end{equation*}
Since $\eta <1$ for $n_{a}$ large enough, each term in the sum scales as $O(\eta)$. Therefore, the overall error scales as $O(M \eta)$. For the second term in (\ref{eqerrorR}), we have
\begin{equation*}
    \prod_{k=1}^{p} r_{n_{a}}^{2M-2}(x_{k})^{\theta_{1,k}} \, \prod_{k=1}^{p} r_{n_{a}}^{2M-3}(x_{k})^{\theta_{1,k}} = 1 + O(\eta).
\end{equation*}
Therefore, this term scales as $O(M \eta)$ too. We can repeat the argument for all the remaining terms in the sum (\ref{eqerrorR}). We obtain
\begin{equation*}
    \| \tilde g_{1}- g_{1} \|_{\infty} = O(M^{2} \eta).
\end{equation*}
Now, recall that $M = O(n^{2p})$. Therefore, for the network $g = (\tilde g_{1}, g_{2}, \dots, g_{m})$ to achieve an error of at most $\epsilon$ when approximation $f$, $n$ and $\eta$ should be such that
\begin{equation} \label{ineqna}
    O(n^{4p} \eta) + \big ( 1+ \frac{p}{4}\big) \sum_{j=1}^{m} \omega(f_{j}, \frac{1}{\sqrt{n}}) \leq \epsilon.
\end{equation}
Indeed, if this inequality holds, then
\begin{align*}
    \|f - g \|_{\infty} &\leq \| f_{1} - \tilde g_{1} \|_{\infty} + \sum_{j=2}^{m} \| f_{j} - g_{j} \|_{\infty} \\
    &\leq \| g_{1} - \tilde g_{1} \|_{\infty} + \|f_{1} - g_{1} \|_{\infty} + \big ( 1+ \frac{p}{4}\big)\sum_{j=2}^{m} \omega(f_{j}, \frac{1}{\sqrt{n}}) \\
    &= O(n^{4p} \eta) + \big ( 1+ \frac{p}{4}\big) \sum_{j=1}^{m} \omega(f_{j}, \frac{1}{\sqrt{n}}).
\end{align*}
To express the depth of $g$ in terms of $\omega(f, \, \cdot \,)$ and $\epsilon$, we can seek to make two terms in the sum on the left-hand side of (\ref{ineqna}) smaller then $\epsilon /2$. For the second term, this yields
\begin{equation} \label{uboundn}
    n \geq \big( \omega^{-1}(f, \frac{\epsilon}{2(1+\frac{p}{4})m}) \big)^{-2}.
\end{equation}
For the first term, we should have
\begin{equation*}
    n^{4p} \frac{(1-\alpha)^{2^{n_{a}+1}}}{2-\alpha} \leq \frac{\epsilon}{2}
\end{equation*}
where we used the definition (\ref{defeta}) of $\eta$. This gives
\begin{equation*}
    n_{a} \geq \log \bigg( \frac{\log \big(\frac{\epsilon (2-\alpha)}{2n^{4p}} \big)}{\log(1-\alpha)} \bigg) \times \frac{1}{\log 2} - 1
\end{equation*}
Using (\ref{uboundn}), $n_{a}$ can then be expressed in terms of $\epsilon$ and $\omega(f, \, \cdot \,)$. Finally, thanks to the estimate (\ref{depthreci2}), we obtain that the depth of $g$ is of order
\begin{align*}
    &O\bigg(p(p+m) \big[ \omega^{-1} \big(f, \frac{\epsilon }{2m(1+\frac{p}{4})}\big) \big]^{-4p} \log \bigg( \log \big(\frac{\epsilon (2-\alpha)}{2\big[ \omega^{-1}\big(f, \frac{\epsilon}{2m(1+\frac{p}{4})}\big) \big]^{-8p}} \big) [ \log(1-\alpha) ]^{-1} \bigg) \bigg) \\
    &+O\bigg( m(p+m) \big[ \omega^{-1}\big(f, \frac{\epsilon}{2m(1+\frac{p}{4})}\big) \big]^{-4p-2} \bigg).
\end{align*}
\end{proof}

\subsubsection{Estimates of the number of multiplications} \label{EstMonomial}

We provide a proof for the estimate of $P$ used in the proof of Proposition \ref{proppoly}. The first term appearing in $P$ is
\begin{equation*}
    \sum_{k_{1}=0}^{n} \dots \sum_{k_{p}=0}^{n} \ \sum_{c_{1}=0}^{n-k_{1}} \dots \sum_{c_{p}=0}^{n-k_{p}} \ \sum_{j=1}^{p} \ \big( n-c_{j} \big).
\end{equation*}
The sum is split into two terms:
\begin{align*}
    &P_{1} = \sum_{k_{1}=0}^{n} \dots \sum_{k_{p}=0}^{n} \ \sum_{c_{1}=0}^{n-k_{1}} \dots \sum_{c_{p}=0}^{n-k_{p}} \ np, \\
    &P_{2} = \sum_{k_{1}=0}^{n} \dots \sum_{k_{p}=0}^{n} \ \sum_{c_{1}=0}^{n-k_{1}} \dots \sum_{c_{p}=0}^{n-k_{p}} \ \sum_{j=1}^{p} \ c_{j}.
\end{align*}
For $P_{1}$, we have:
\begin{align*}
    P_{1} &= \sum_{k_{1}=0}^{n} \dots \sum_{k_{p}=0}^{n} \ np \ \prod_{j=1}^{p} \big(n-k_{j} + 1 \big) \\
    &= np \ \sum_{k_{1}=0}^{n} \dots \sum_{k_{p-1}=0}^{n} \ \bigg[\prod_{j=1}^{p-1} \big(n-k_{j} + 1 \big) \bigg]\ \sum_{k_{p}=0}^{n} \ \big(n - k_{p} + 1 \big) \\
    &= np \ \sum_{k_{1}=0}^{n} \dots \sum_{k_{p-1}=0}^{n} \ \bigg[ \prod_{j=1}^{p-1} \big(n-k_{j} + 1 \big) \bigg] \ \bigg( \frac{n}{2} +1 \bigg) \bigg( n+1 \bigg) \\
    &= np \ \bigg( \frac{n}{2} +1 \bigg) \bigg( n+1 \bigg) \ \sum_{k_{1}=0}^{n} \dots \sum_{k_{p-2}=0}^{n} \ \bigg[ \prod_{j=1}^{p-2} \big(n-k_{j} + 1 \big) \bigg] \ \sum_{k_{p-1}=0}^{n} \ \big(n - k_{p-1} + 1 \big) \\
    &= np \ \bigg( \bigg( \frac{n}{2} +1 \bigg) \bigg( n+1 \bigg) \bigg)^{2} \ \sum_{k_{1}=0}^{n} \dots \sum_{k_{p-2}=0}^{n} \ \bigg[ \prod_{j=1}^{p-2} \big(n-k_{j} + 1 \big) \bigg] \\
    &= np \ \bigg( \bigg( \frac{n}{2} +1 \bigg) \big( n+1 \big) \bigg)^{p} 
    = O(n^{2p+1}).
\end{align*}
For $P_{2}$, we have:
\begin{equation*}
    P_{2} = \sum_{k_{1}=0}^{n} \dots \sum_{k_{p}=0}^{n} \ \sum_{c_{1}=0}^{n-k_{1}} \dots \sum_{c_{p}=0}^{n-k_{p}} \ \sum_{j=1}^{p-1} \ c_{j} 
    + \sum_{k_{1}=0}^{n} \dots \sum_{k_{p}=0}^{n} \ \sum_{c_{1}=0}^{n-k_{1}} \dots \sum_{c_{p}=0}^{n-k_{p}} \ c_{p}.
\end{equation*}
Let $P_{3}$ denote the second term of this sum. We have:
\begin{align*}
    P_{3} &= \sum_{k_{1}=0}^{n} \dots \sum_{k_{p}=0}^{n} \ \sum_{c_{1}=0}^{n-k_{1}} \dots \sum_{c_{p-1}=0}^{n-k_{p-1}} \ \frac{\big( n -k_{p} \big) \big( n-k_{p} + 1 \big)}{2} \\
    &= \sum_{k_{1}=0}^{n} \dots \sum_{k_{p}=0}^{n} \ \frac{\big( n -k_{p} \big) \big( n-k_{p} + 1 \big)}{2} \bigg[ \prod_{j=1}^{p-1} \big( n-k_{j} + 1 \big) \bigg] \\
    &= \sum_{k_{1}=0}^{n} \dots \sum_{k_{p}=0}^{n} \ \frac{n -k_{p} }{2} \bigg[ \prod_{j=1}^{p} \big( n-k_{j} + 1 \big) \bigg] \\
    &= \frac{n}{2} \ \sum_{k_{1}=0}^{n} \dots \sum_{k_{p}=0}^{n} \ \bigg[ \prod_{j=1}^{p} \big( n-k_{j} + 1 \big) \bigg] \\
    &- \frac{1}{2} \ \sum_{k_{1}=0}^{n} \dots \sum_{k_{p}=0}^{n} \ k_{p} \bigg[ \prod_{j=1}^{p} \big( n-k_{j} + 1 \big) \bigg] \\
    &= \frac{n}{2} \bigg( \bigg( \frac{n}{2} + 1 \bigg) \big( n+1 \big) \bigg)^{p} \\
    &- \frac{1}{2} \ \sum_{k_{1}=0}^{n} \dots \sum_{k_{p}=0}^{n} \ k_{p} \ \bigg[ \prod_{j=1}^{p} n - k_{j} + 1 \bigg].
\end{align*}
The second term roughly behaves like $O(n^{2p-1})$: this can be seen by taking $k_{p} = n$ and $k_{j}=0$ for all $j=1, \dots, p-1$. Thus $P_{3}$ is of order $O(n^{2p+1})$. We apply the same decomposition for the first term of $P_{2}$ and the computations are similar to those for $P_{3}$. Therefore, we obtain
\begin{equation*}
    P_{2} = O(n^{2p+1}).
\end{equation*}
Now, to get the rough behavior of $P$, we still have to take into account the term $M-M_{0}$. On one hand, the total number $M$ of monomials is equal to
\begin{align*}
    M &= \sum_{k_{1}=0}^{n} \dots \sum_{k_{p}=0}^{n} \ \sum_{c_{1}=0}^{n-k_{1}} \dots \sum_{c_{p}=0}^{n-k_{p}} \ 1 \\
    &= \sum_{k_{1}=0}^{n} \dots \sum_{k_{p}=0}^{n} \ \bigg[ \prod_{j=1}^{p} \big( n -k_{j} + 1 \big) \bigg] \\
    &= \bigg( \bigg( \frac{n}{2} + 1 \bigg) \big( n + 1\big) \bigg)^{p} \\
    &= O(n^{2p}).
\end{align*}
On the other hand, there are $np$ monomials where only one coordinate has a non zero power and $1$ monomial where all the coordinates have $0$ power. Therefore, $M-M_{0}$ behaves like $O(n^{2p})$. This yields for $P$
\begin{equation*}
    P = O(n^{2p+1}).
\end{equation*}

\subsection{Proofs of GDN Approximation Results}
\subsubsection{{Proof of Theorem~\ref{thrm_main_Local}}}
\begin{proof}[{Proof of Theorem~\ref{thrm_main_Local}}]
We take $\mathcal{F}=\mathcal{NN}_{m,n,m+n+2}^{\sigma}$. Let $f \in C(\xxx,\yyy)$. Let $x \in \xxx$. By Lemma \ref{lem_local_representation}, on $\overline{B_{\xxx}(x,\delta)}$, $f$ can be represented by
\begin{equation*}
    f = \operatorname{Exp}_{\yyy,f(x)} \circ \tilde f \circ \operatorname{Exp}_{\xxx,x}^{-1} 
\end{equation*}
where $\tilde f \in \mathcal{C}(\mathbb{R}^{m}, \mathbb{R}^{n})$. Since $\operatorname{Exp}_{\xxx,x}^{-1}$ is continuous, $K:=\operatorname{Exp}_{\xxx,x}^{-1}(\overline{B_{\xxx}(x,\delta)}) \subset \mathbb{R}^{m}$ is compact. We can consider $\tilde f$ on $K$ and then approximate it with a neural network $g \in \mathcal{NN}_{m,n,m+n+2}$. By choosing $g$ such that $\| \tilde f - g \|_{\infty} \leq  L_{\yyy,f(x)}^{-1} \epsilon$ on $K$
\begin{equation*}
    \| f - \operatorname{Exp}_{\yyy,f(x)} \circ g \circ \operatorname{Exp}_{\xxx,x}^{-1}\| \leq L_{\yyy,f(x)} \| \tilde f \circ \operatorname{Exp}_{\xxx,x}^{-1} - g \circ \operatorname{Exp}_{\xxx,x}^{-1} \| \leq \epsilon.
\end{equation*}
The depth of $g$ is related to $\tilde f$ via its inverse modulus of continuity. By Lemma \ref{lem_local_representation}, it is given by
\begin{equation*}
    \omega(\tilde{f},\epsilon) = 
    L_{\yyy,f(x)}^{-1}\omega\left(
    f,L_{\xxx,x}^{-1}\epsilon
    \right) 
\end{equation*}
This equality is a consequence of the right continuity of the generalized inverse of modulus of continuity as shown in \cite{EmbrechtsHofert}. The result now follows from Proposition \ref{maindepth} with $K\triangleq \overline{B_{\xxx}(x,\delta)}$.  Note, in the theorem's statement we have set $\kappa_1\triangleq L_{\xxx,x}$ and $\kappa_2\triangleq L_{\yyy,f(x)}^{-1}$.
\end{proof}

\subsubsection{Proof of Corollary~\ref{cor_Cartan_Hadamard}}\label{sss_Appendix_local_to_global}
\begin{proof}[{Proof of Corollary~\ref{cor_Cartan_Hadamard}}]
By definition, $\xxx$ and $\yyy$ have everywhere non-positive sectional curvature $K\leq 0$.  Hence, $k^{\star}_{\xxx}=k^{\star}_{\yyy}=\infty$ for each $x\in \xxx$ and each $y \in \yyy$.  In particular, this is the case for $y=f(x)$ for any $f \in \bar{C}(\xxx,\yyy)$.  The result now following from the left-hand side of the estimate in~\eqref{eq_generic_lower_bound_Gromov_Application} within Theorem~\ref{thrm_main_Local}.  
\end{proof}
\subsection{{Proof of Results of Uncursed Approximation Results of Section~\ref{subsection_data_dependant_dimenion_free_approx}}}\label{s_Appendix_subsection_data_dependant_dimenion_free_approx}
\begin{proof}[{Proof of Proposition~\ref{ex_triviality_smooth_functions}}]
Fix $f \in \bar{C}(\xxx,\yyy)$, $x \in \xxx$, $0\leq \eta<1$, and let $\emptyset\neq \XX\subseteq \xxx$ such that condition~\eqref{eq_containement_condition} holds.
Since $\operatorname{Exp}_{\yyy,f(x)}$ and of $\operatorname{Exp}_{\xxx,x}^{-1}$ are smooth and since we have assumed~\eqref{eq_containement_condition} then, for each $i=1,\dots,D$, each of the functions is well-defined and belongs to $C^{nd,1}(\overline{B_{\rrp}(0,\eta\operatorname{inj}_{\xxx}(x))},\rrm)$:
$$
\tilde{f}_i\triangleq \pi_i\circ \operatorname{Exp}_{\yyy,f(x)} \circ f\circ \operatorname{Exp}_{\xxx,x}^{-1}
.
$$
For each $c=1,\dots,C^{\star}$ and $i=1,\dots,m$ define the data:
$$
\tilde{x}_{c,i}\triangleq \operatorname{Exp}_{\xxx,x}^{-1}(x_c)_i
.
$$
Observe that, if $C^{\star}>\#C$ then $C^{\star}$ is precisely the number of multi-indices $(\beta_1,\dots,\beta_p)$ satisfying $\sum_{i=1}^p \beta_i\leq nd-1$.  Otherwise, it is impossible to pick more $\#C$ distinct points in $C$ (in particular we cannot pick $(C+1)^{2^C}$ distinct points).  In either case, $C^{\star}$ is the quantity $k^{\#}$ defined on \citep[page 576]{Fefferman1995ExtensionTheorem}.  Hence the conditions hold by discussion on \citep[page 510]{Fefferman1995ExtensionTheorem} following \citep[Theorem A]{Fefferman1995ExtensionTheorem}'s statement.  

For the last statement, we only need to show that condition~\eqref{eq_containement_condition} is always satisfied if $\xxx$ and $\yyy$ are Cartan-Hadamard manifolds.  This is true because if $\xxx$ and $\yyy$ are both Cartan-Hadamard manifolds, then the Cartan-Hadamard Theorem (\citep[Corollary 6.9.1]{jost2008riemannian}) implies that $\operatorname{inj}_{\xxx}(x)=\infty=\operatorname{inj}_{\yyy}(y)$ for every $x \in \xxx$ and every $y \in \yyy$; thus~\eqref{eq_containement_condition}.  
\end{proof}
\begin{proof}[{Proof of Corollary~\ref{prop_everything_is_nice}}]
By \citep[Theorem 2.4]{HirschDifferentialTopology1994}, the set of smooth functions from $\xxx$ to $\yyy$ (denoted by $C^{\infty}(\xxx,\yyy)$) is dense in $\bar{C}(\xxx,\yyy)$.  Therefore, it is enough to show the result holds when $f \in C^{\infty}(\xxx,\yyy)$.  Fix an atlas $(\phi_{\alpha},U_{\alpha})_{\alpha \in A}$ of $\xxx$ and an atlas $(\rho_{\gamma},V_{\gamma})_{\gamma \in \Gamma}$ of $\yyy$. Fix any $\phi_{\alpha}$ and $\rho_{\gamma}$ such that the composition:
$$
\phi_{\alpha}\circ f\circ \rho_{\gamma}^{-1},
$$
is well-defined.  Since $\phi_{\alpha}$ and $\rho_{\gamma}$ are smooth and $f$ is smooth by hypothesis, then we have $\phi_{\alpha}\circ f\circ \rho^{-1}_{\gamma}\in C^{\infty}(K,\rrm)$ for any non-empty compact subset $K\subseteq \phi_{\alpha}(U_{\alpha})$.  By the Mean-Valued Theorem, every such smooth function has Lipschitz $k^{th}$-order partial derivatives locally on any compact subset $\emptyset\neq K\subseteq \phi_{\alpha}(U_{\alpha})$ in any smooth chart $(\phi_{\alpha},U_{\alpha})$ on $\xxx$.  Therefore, $C^{\infty}(\xxx,\yyy)\subseteq C^{k,1}_{\operatorname{loc}}(\xxx,\yyy)$ for any $k\in \nn_+$; in particular this is the case for $k=d$.  The conclusion now follows from Proposition~\ref{ex_triviality_smooth_functions}.  
\end{proof}

\begin{proof}[{Proof of Proposition~\ref{prop_every_finite_dataset_is_efficient}}]
Since $\xxx\subset B_{\xxx}(x,\eta\mathcal{U}_f(x))$, then Lemma~\ref{lem_redux_to_Euclidean_diffeo_Lipschitz_tools} implies that $\tilde{\XX}\triangleq \{\text{Exp}_{\xxx,x}^{-1}(z)\}_{z \in \XX}$ is well-defined.  Moreover, Lemma~\ref{lem_redux_to_Euclidean_diffeo_Lipschitz_tools} implies that (for each $c =1,\dots,\# \XX$) the following Lagrange-type polynomial, of degree exactly $\#\XX -1$, is well-defined:
\begin{equation}
    p_c(z) \triangleq \sum_{\tilde{x}\in \tilde{\XX}}
    \frac{
    \prod_{y\neq \tilde{x}} (z-y)
    }{
    \prod_{y\neq \tilde{x}} (\tilde{x}-y)
    }
    \text{Exp}_{\yyy,f(x)}^{-1}\circ f\circ \text{Exp}_{\xxx,x}(\tilde{x})
    \label{proof_eq_prop_every_finite_dataset_is_efficient_Lagrange_interpolator}
    .
\end{equation}
Since $\XX$ is finite, then $\overline{B_{\rrp}(0,\text{diam}(\tilde{\XX}))}$ is closed and bounded.  Thus by the Heine-Borel Theorem it is compact in $\rrp$.  Hence, the quantity:
$$
M \triangleq 2 \text{diam}(\tilde{\XX})
\max_{c=1,\dots,\#\XX,\,
\tilde{x}\in \tilde{\XX},
\,\|z\|\leq\|\tilde{x}\| }
\|p_c(z)\|
    +
    \sum_{|\beta|\leq p\#\XX}
|\partial^{\beta}p_c(z)|,
$$
must be finite.  By construction, $M$ satisfies Definition~\ref{ass_regularity_conditions} (ii) and (iii).  Observe that, Definition (i) holds by construction of the polynomials $p_1,\dots,p_{\# \XX}$.  Since each $p_c$ is of degree exactly $\#\XX -1$ then the condition that $p+1<\#\XX$ and $p$ divides $\# \XX-1$ implies that $\frac{\#\XX}{p}\in \nn_+$; thus, $\XX$ is $\frac{\#\XX -1}{p}$-efficient. 
\end{proof}

\subsubsection{{Proof of Theorem~\ref{thrm_efficient_rates}}}\label{sss_proof_of_dimension_Free_bounds_result}
In the following proof, denote use $\pi_i:\rrm\ni x \mapsto x_i\in \rr$ to denote each of the canonical coordinate projections (for $i=1,\dots,m$).  We also use $e_1,\dots,e_m$ to denote the standard orthonormal basis on $\rrm$.  
\begin{proof}[{Proof of Theorem~\ref{thrm_efficient_rates}}]
Since $\XX$ is normalizable, then there must exist some $\eta \in [0,1)$ and some $x \in \xxx$ such that:
\begin{align}
    \XX & \subseteq B_{\xxx}(x,\eta \, \mathcal{U}_f(x)) ,\\
\tilde{\XX}& 
    \triangleq 
\text{Exp}_{\xxx,x}^{-1}[\XX]
    \subseteq 
[0,1]^p
\label{proof_thrm_efficient_rates_eq_inclusions_0_normalizability_assumption_applied}
.
\end{align}
By Lemma~\ref{lem_local_representation}, there is a $\tilde{f}:\rrm\rightarrow \rrp$ satisfying:
\begin{equation}
    \operatorname{Exp}_{\yyy,f(x)}^{-1}\circ f\circ \text{Exp}_{\xxx,x}|_{\tilde{\XX}}
        =
    \tilde{f}|_{\tilde{\XX}} = \sum_{i=1}^m (\pi_i\circ \tilde{f}|_{\tilde{\XX}})e_i
    \label{proof_thrm_efficient_rates_eq_inclusions_3_setup}
    .
\end{equation}
Since $f$ was assumed to be $n$-efficient then, for each $i=1,\dots,m$, the sets:
    $
    \{(\pi_i\circ \tilde{f}(z),z):z\in \tilde{\XX}\}
    $ satisfy the conditions of the Whitney-type Extension Theorem (as formulated in \citep[Theorem A]{Fefferman1995ExtensionTheorem}).  Therfore, there exist:
    ${F}_1,\dots,{F}_m\in C^{nd,1}(\rrp,\rr)$ satisfying:
    \begin{equation}
    {F}_i|_{\tilde{\XX}} = \pi_i\circ \tilde{f}|_{\tilde{\XX}};
    \label{proof_thrm_efficient_rates_eq_inclusions_4_Whitney_extension}
    \qquad  i=1,\dots,m
    .
\end{equation}
Since $F_i\in C^{nd,1}(\rrp,\rr)$ then \citep[Theorem 3.3]{yarotsky2019phase} implies that, for every $\epsilon>0$, there are $\hat{g}_i\in \NN[p,1][\max\{0,\cdot\}]$ satisfying the uniform approximation estimate:
\begin{equation}
    \max_{i=1,\dots,m}\,\sup_{x\in [0,1]^p}\, |\hat{g}_i(x)-F_i(x)| <\epsilon
    \label{proof_thrm_efficient_rates_eq_inclusions_5_Yaroski_estimate}
    .
\end{equation}
Therefore, incorporating the identities~\eqref{proof_thrm_efficient_rates_eq_inclusions_3_setup} and~\eqref{proof_thrm_efficient_rates_eq_inclusions_4_Whitney_extension} into the estimate~\eqref{proof_thrm_efficient_rates_eq_inclusions_5_Yaroski_estimate} we may derive the following approximation bound:
\allowdisplaybreaks
\begin{align}
    &\sup_{x \in \XX}\, 
        d_{\yyy}\left(
                \text{Exp}_{\yyy,f(x)}
                    \left(
                    \sum_{i=1}^m \hat{g}_i\circ \text{Exp}_{\xxx,x}^{-1}(x) e_i
                    \right)
                ,
                    f(x)
                \right)
\label{proof_thrm_efficient_rates_eq_inclusions_6_0_Main_Bound_lipschitz_estimate}
\\
= &
    \sup_{x \in \XX}\, 
        d_{\yyy}\left(
                \text{Exp}_{\yyy,f(x)}
                    \left(
                    \sum_{i=1}^m \hat{g}_i\circ \text{Exp}_{\xxx,x}^{-1}(x) e_i
                    \right)
                ,
                    \text{Exp}_{\yyy,f(x)}\circ \tilde{f} \circ \text{Exp}_{\xxx,x}^{-1}(x)
                \right)
\label{proof_thrm_efficient_rates_eq_inclusions_6a_Main_Bound_lipschitz_estimate}
\\
\leq &
    L_{\yyy,f(x)}^{-1}
    \sup_{x \in \XX}\, 
        \left\|
                \sum_{i=1}^m \hat{g}_i\circ \text{Exp}_{\xxx,x}^{-1}(x) e_i
            -
                \tilde{f} \circ \text{Exp}_{\xxx,x}^{-1}(x)
        \right\|
\label{proof_thrm_efficient_rates_eq_inclusions_6b_Main_Bound_lipschitz_estimate}
\\
\nonumber
\leq &
    L_{\yyy,f(x)}^{-1}
    \sup_{z \in \tilde{\XX}}\, 
        \left\|
                \sum_{i=1}^m \hat{g}_i(z) e_i
            -
                \tilde{f}(z)
        \right\|
\\
\nonumber
\leq &
    L_{\yyy,f(x)}^{-1}
    \sup_{z \in [0,1]^p}\, 
        \left\|
                \sum_{i=1}^m \hat{g}_i(z) e_i
            -
                \tilde{f}(z)
        \right\|
\\
\nonumber
\leq &
    L_{\yyy,f(x)}^{-1}
    m^{\frac1{2}}
    \sup_{z \in [0,1]^p}\, 
        \left|
                \hat{g}_i(z)
            -
                \pi_i\circ \tilde{f}(z)
        \right|
\\
\leq & 
    L_{\yyy,f(x)}^{-1}
    m^{\frac1{2}}
    \epsilon
\label{proof_thrm_efficient_rates_eq_inclusions_6_Main_Bound}
    ;
\end{align}
where, the pass from~\eqref{proof_thrm_efficient_rates_eq_inclusions_6a_Main_Bound_lipschitz_estimate} to~\eqref{proof_thrm_efficient_rates_eq_inclusions_6b_Main_Bound_lipschitz_estimate} was implied by the Lipschitz estimate in Lemma~\ref{lem_local_representation}.  In particular, \citep[Theorem 3.4]{YAROTSKYSobolev} implies that, for $i=1,\dots,m$, $\hat{g}_i\in \NN[p,1:2p+1]$ each with $\mathscr{O}(
\epsilon^{-\frac{2p}{3(np+1)}}
)$ trainable weights each of depth of order 
$\mathscr{O}
\left   (\epsilon^{
    \frac{2p}{3(np+1)}
        -
    \frac{p}{np+1}
}\right)$ (since each $F_i$ belonged to $C^{nd,1}(\rrd,\rr)$).  

It remains to show that $\sum_{i=1}^m\hat{g}_ie_i$ can be implemented by DNN in $\NN[p,m]$.  Indeed, since $\sigma$ was non-affine and piecewise linear, then $B\in\nn_+$ (an the notation at the start of Section~\ref{sss_pw_lin_active}), breaks then \citep[Proposition 1 (b)]{YAROTSKYSobolev} implies that the there are $g_1,\dots,g_m\in \NN[p,1:W_i]$ such that for every $x \in [0,1]^p$ and each $i=1,\dots,m$ we have:
\begin{equation}
\hat{g}_i(x) = g_i(x)
\label{proof_thrm_efficient_rates_eq_inclusions_7_equality}
.
\end{equation}
Moreover, each $g_i$ has the same depth as $\hat{g}_i$; i.e., each $g_i$ has depth 
$\mathscr{O}
\left(\epsilon^{
    \frac{2p}{3(np+1)}
        -
    \frac{p}{np+1}
}\right)$, each $g_i$ has width at-most $2 (2p+1)$, and each $g_i$ has width at-most $4\times$ that of the corresponding $\hat{g}_i$; i.e., each $g_i$ has $\mathscr{O}(\epsilon^{-\frac{2p}{3(np+1)}})$ trainable weights.  

Next, we observe that by \citep[Proposition 1 (b)]{YAROTSKYSobolev} and the discussion on \citep[page 3, following Definition 4]{Florian2_2021} $\NN[1,1]$ has the $8$-identity requirement (defined in \citep[Definition 4]{Florian2_2021}).  Therefore, \citep[Proposition 5]{Florian2_2021} applies; whence, there is a network $\hat{g}\in \NN[p,m]$ "parallelizing the $g_1,\dots,g_m$ i.e. $g$ satisfies:
\begin{equation}
    g =\sum_{i=1}^m g_i e_i
    \label{proof_thrm_efficient_rates_eq_inclusions_7a_equality_ReLU_multivariate_identification}
    .
\end{equation}
Moreover, $\hat{g}$ has depth equal to $m$ plus the sum of the depths of $\hat{g}_1,\dots,\hat{g}_m$; i.e. it has depth of order
$
\mathscr{O}\left(m+m
\epsilon^{
    \frac{2p}{3(np+1)}
        -
    \frac{p}{np+1}
}
\right)
$
, it has width equal to the sum of the widths of the $\hat{g}_1,\dots,\hat{g}_m$ plus $m8$; i.e. $\hat{g}$ has width at-most $m(4p+10)$ but no less than $m$, and $\hat{g}$ has at-most 
$
(\frac{11}{4^2} 8^2 m^2 m^2 - 1)\sum_{i=1}^m p_i
$
trainable parameters (where $p_i$ is the number of trainable parameters in $\hat{g}_i$); thus, $\hat{g}$ has at-most 
$
\mathscr{O}\left(
m(m^2-1)
\epsilon^{-\frac{2p}{3(np+1)}}
\right)
$
trainable parameters.  Combining the identities~\eqref{proof_thrm_efficient_rates_eq_inclusions_7a_equality_ReLU_multivariate_identification} and~\eqref{proof_thrm_efficient_rates_eq_inclusions_7_equality} and plugging them into~\eqref{proof_thrm_efficient_rates_eq_inclusions_6_0_Main_Bound_lipschitz_estimate}-~\ref{proof_thrm_efficient_rates_eq_inclusions_6_Main_Bound} we obtain the desired estimate:
$$
\begin{aligned}
& \sup_{x \in \XX}\, 
        d_{\yyy}\left(
                \text{Exp}_{\yyy,f(x)}\circ g\circ \text{Exp}_{\xxx,x}^{-1}(x)
                ,
                    f(x)
                \right)
\\
= &
\sup_{x \in \XX}\, 
        d_{\yyy}\left(
                \text{Exp}_{\yyy,f(x)}
                    \left(
                    \sum_{i=1}^m \hat{g}_i\circ \text{Exp}_{\xxx,x}^{-1}(x) e_i
                    \right)
                ,
                    f(x)
                \right)
\\
\leq & 
    L_{\yyy,f(x)}^{-1}
    m^{\frac1{2}}
    \epsilon
.
\end{aligned}
$$
This concludes our proof.
\end{proof}
\subsection{{Proofs of The Results for GDL models of Section~\ref{s_Results_pt2}}}
\begin{proof}[{Proof of Theorem~\ref{thrm_Quotient_UAT}}]
Together, Assumption~\ref{ass_group}, Assumption~\ref{ass_properly_discontinuous_group_action}, and Assumption~\ref{ass_simply_connectedness} implies that $\xxx_0$ and $\fff$ are both simply connected; hence, \citep[Proposition 3.4.15]{Burago2CourseonMetricGeometry2001} applies; whence, the projection map~\eqref{eq_quotient_map} is a covering map.  

Since $\xxx$ satisfies Assumption~\ref{ass_properly_discontinuous_group_action} (i) (with $\xxx$ in place of $\zzz$) and since $\xxx$ is also a manifold, then it is path-connected and locally path-connected; thus \citep[Chapter 2, Section 2, Theorem 5]{SpanierTop1995} applies and therefore there exists some $\hat{f}\in C(\xxx,\zzz)$ satisfying the following "lifting property":
\begin{equation}
    f(x) = \rho_1\circ \tilde{f}(x) 
    \qquad (\forall x \in \xxx)
    \label{proof_thrm_Quotient_UAT_eq_homotopy_lifting_property}
    .
\end{equation}
Applying Theorem~\ref{thrm_main_Local}, we conclude that there exists some GDN $\hat{f}$ (with depth order recorded in Table~\ref{tab_rates_and_depth}) of width at-most $d+m+2$ satisfying the approximation bound:
\begin{equation}
    \max_{x \in \XX}\, 
    d_{\zzz}(\tilde{f}(x),\hat{f}(x))
    <\epsilon
\label{proof_thrm_Quotient_UAT_eq_homotopy_lifted_approximation_bound}
    .
\end{equation}

By \citep[Lemma 3.3.6]{Burago2CourseonMetricGeometry2001}, $\bar{d}_{\zzz}$ coincides with the quotient metric on $\zzz$ from \citep[Definition 3.1.12]{Burago2CourseonMetricGeometry2001} and the discussion at the bottom of \citep[page 62]{Burago2CourseonMetricGeometry2001} we conclude that:
\begin{equation}
    \bar{d}_{\zzz}([z_1],[z_2])
    \leq
    d_{\zzz}(\tilde{z}_1,\tilde{z}_2)
\label{proof_thrm_Quotient_UAT_eq_homotopy_lifted_bounded_distance_1}
    ,
\end{equation}
for every $z_1,z_2\in \zzz$ and every $\tilde{z}_1\in \rho_1^{-1}[z_1]$ and every $\tilde{z}_2\in \rho_1^{-1}[z_2]$.  In particular, for each $x \in \XX$, ~\eqref{proof_thrm_Quotient_UAT_eq_homotopy_lifted_bounded_distance_1} holds for $[z_1]=[\tilde{f}(x)]$ and $[z_2]=[\hat{f}(x)]$.  Now, since $\rho_1$ is a surjection, then by construction of $\rho_1$ we have:
\begin{equation}
    [\tilde{f}(x)]=\rho_1(\tilde{f}(x)) \mbox{ and }
    [\hat{f}(x)]=\rho_1(\hat{f}(x)) 
    \label{proof_thrm_Quotient_UAT_eq_obvious_identities}
    .
\end{equation}
Plugging~\eqref{proof_thrm_Quotient_UAT_eq_obvious_identities} into~\eqref{proof_thrm_Quotient_UAT_eq_homotopy_lifted_bounded_distance_1} and incorporating the result into~\eqref{proof_thrm_Quotient_UAT_eq_homotopy_lifted_approximation_bound} yields the estimate:
\begin{equation}
    \sup_{x \in \XX}\,
    \bar{d}_{\zzz}(\rho_1(\tilde{f}(x)),\rho_1(\hat{f}(x)) )
    \leq
    d_{\zzz}(\tilde{z}_1,\tilde{z}_2)
    <\epsilon
\label{proof_thrm_Quotient_UAT_eq_homotopy_lifted_bounded_distance_2}
    .
\end{equation}
By~\eqref{proof_thrm_Quotient_UAT_eq_homotopy_lifting_property} we may replace the quantity $\rho_1(\tilde{f}(x))$ in~\eqref{proof_thrm_Quotient_UAT_eq_homotopy_lifted_bounded_distance_2} with $f(x)$; therefore, we find that:
\begin{equation}
    \sup_{x \in \XX}\,
    \bar{d}_{\zzz}(f(x),\rho_1(\hat{f}(x)) )
    \leq
    d_{\zzz}(\tilde{z}_1,\tilde{z}_2)
    <\epsilon
\label{proof_thrm_Quotient_UAT_eq_homotopy_lifted_bounded_distance}
    .
\end{equation}
This concludes the proof.
\end{proof}
\begin{proof}[{Proof of Corollary~\ref{cor_parallelization}}]
By \citep[Theorem 19.6]{munkres2014topology} $f\in C(\xxx,\prod_{i=1}^I \yyy_i)$ if and only if, for $k=1,\dots,\XX$, there exist $f_i \in C(\xxx,\yyy_i)$ such that:
\begin{equation}
    f=(f_1,\dots,f_I)
    \label{proof_cor_parallelization_product_representation}
    .
\end{equation}
Since each $\yyy_i$ and $\fff_i$ satisfy Assumptions~\ref{ass_group} and~\ref{ass_properly_discontinuous_group_action} then, for every $0<\delta_i<\mathcal{U}_{f_i}(x)$ there exist $g_i \in \NN[p,m_i]$ (with width $p+m_i+2$ depth recorded in Table~\ref{tab_rates_and_depth} with $m_i$ in place of $m$) such that:
\begin{equation}
    \sup_{x \in \XX}\,
d_{\yyy_i}\left(
    f_i(x)
        ,
    \hat{f}_i(x)
\right)<\epsilon;
    \label{proof_cor_parallelization_parallel_rates}
\end{equation}
where each of the $\hat{f}_i$ are as in~\eqref{cor_parallelization_paralleized_architecture}.  Thus, together,~\eqref{proof_cor_parallelization_product_representation} and~\eqref{proof_cor_parallelization_parallel_rates} imply the estimate:
\allowdisplaybreaks
\begin{align}
\nonumber
\sup_{x\in \XX}\,
d_{\prod_i \yyy_i}(f(x),(\hat{f}_1(x),\dots,\hat{f}_I(x)))
= &
\sup_{x\in \XX}\,
d_{\prod_i \yyy_i}((f_1(x),\dots,f_I(x)),(\hat{f}_1(x),\dots,\hat{f}_I(x)))
\\
\nonumber
= & 
\sup_{x\in \XX}\,
\max_{k=1,\dots,I}
d_{\yyy_i}\left(
    f_i(x)
        ,
    \hat{f}_i(x)
\right)<\epsilon
.
\end{align}

Lastly, consider the "efficient case"; i.e, when $\rho_i=1_{\fff_i}$ and $\XX$ is $f_i$-normalized and $n$-efficient for $f_i$, for each $i=1,\dots,I$.  Then, the $\hat{f}_i$ implementing the estimates~\eqref{proof_cor_parallelization_parallel_rates} are instead given by Theorem~\ref{thrm_efficient_rates}. 
\end{proof}
{\color{black}{
\begin{proof}[{Proof of Proposition~\ref{prop_quantitative_Z_set_description}}]
By Assumption~\ref{ass_boundary_homotopy} (i), $\im{\rho}=R^{-1}\left[\prod_{i=1}^I \,\yyy_i\right]$, $R$ is continuous, and therefore $\im{\rho}$ is an open subset of $\yyy$.  To show (i), it only remains to show that $\im{\rho}$ is dense in $\yyy$.  Let $y\in \yyy$ and define the sequence $(y_n)_{n\in\nn}$ by $y_n\triangleq H_{1-n^{-1}}(y)$ and note that by Assumption~\ref{ass_boundary_homotopy} (ii) point (a), $\{y_n\}_{n\in \nn}\in \im{\rho}$.  By Assumption~\ref{ass_boundary_homotopy} (ii) point (b), we have that $\lim\limits_{n\to \infty}\, d_{\yyy}(y_n,y)=0$; thus, $\im{\rho}$ is dense in $\yyy$.

Suppose that $\yyy$ is a topological manifold with boundary whose topology is induced by the metric $d_{\yyy}$; ; i.e. $\yyy$ is a metrizable topological manifold with boundary (see \cite{brown1962locally} for example).  Since $\im{\rho}$ is an open subset of $\yyy$, then $Z:=\yyy-\im{\rho}$ is a closed subset of $\yyy$ with the property that there exists a homotopy $h:[0,1]\times \yyy \rightarrow \yyy$ satisfying $h(0,\cdot)=1_{\yyy}$ and for which $h(t,\yyy)\subseteq \yyy-Z$ for all $t\in (0,1]$; namely $h(y,t):=H_{1-t}(y)$.  Thus, $Z$ is a $\mathcal{Z}$-set (as defined at the begining of \citep[Section 3]{GuilbaultMoran_2019_ExpositionesMathematicae}).  Hence, \citep[Example 3]{GuilbaultMoran_2019_ExpositionesMathematicae} implies that $Z$ is contained in $\yyy$'s boundary; thus, (ii) holds.   
\end{proof}
}}

\begin{proof}[{Proof of Theorem~\ref{thrm_full_computational_graph}}]
Fix $\epsilon>0$.  By Assumption~\ref{ass_boundary_homotopy}, there exist a $t_{\frac{\epsilon}{2}}\in [0,1)$ satisfying:
\begin{equation}
    \sup_{y \in \yyy}
    \,
        d_{\yyy}(y,H_{t_{\frac{\epsilon}{2}}}(y))<2^{-1}\epsilon
    \mbox{ and }
    H_{t_{\frac{\epsilon}{2}}}(\yyy)\subseteq \im{\rho}
        .
    \label{proof_thrm_full_computational_graph_uniform_homotopy_negligeability_applied}
\end{equation}
Since $\xxx$ is (non-empty) and compact and since $\phi$ is continuous, then \citep[Theorem 26.5]{munkres2014topology} implies that $\phi(\xxx)\subseteq \xxx_0$ is (non-empty) compact.  

By Assumption~\ref{ass_boundary_homotopy} (i), there is a continuous $R\in C(\im{\rho},\prod_{k=1}^K \yyy_k)$ which is a right-inverse for $\rho$ on $\im{\rho}$.  Since $H$ is continuous for the product topology, then \citep[Theorem 18.4]{munkres2014topology} implies that $H_{t_{\frac{\epsilon}{2}}}\in C(\yyy,\yyy)$.  Since $\phi$ is continuous and $\xxx$ is compact then the Closed Map Lemma (see \citep[Exercise 26.6]{munkres2014topology}) implies that $\phi$ is injective and proper (see \cite{nlab_injective_proper_maps_to_locally_compact_spaces_are_equivalently_the_closed_embeddings}); thus, $\phi^{-1}$ exists and belongs to $C(\phi(\xxx),\xxx)$.  Therefore, the coincidence between the uniform convergence on compact sets topology and the compact-open topology (see \citep[Theorem 46.8]{munkres2014topology} and \citep[Exercise 46.7]{munkres2014topology}) implies that $R\circ H_{t_{\frac{\epsilon}{2}}}\circ f\circ \phi^{-1} \in C(\phi(\xxx),\prod_{k=1}^K \yyy_k)$.

Since $K$ is a non-empty compact subsets of 
$$
B_{\xxx}(x^{\star},
\tilde{\omega}_{\phi}^{-1}(\eta \min_{k=1,\dots,K} \mathcal{U}_{[R\circ H_{t_{\frac{\epsilon}{2}}}\circ f\circ \phi^{-1}]_k}(x^{\star}))
.
$$ 
Thus, the continuity of $\tilde{\omega}_{\phi}$ and \citep[Proposition 1 (4)]{EmbrechtsHofert} implies that:
$$
\phi(K)\subset B_{\xxx_0}(\phi(x^{\star}),\eta \min_{k=1,\dots,K} \mathcal{U}_{[R\circ H_{t_{\frac{\epsilon}{2}}}\circ f\circ \phi^{-1}]_k}(x^{\star})
.
$$
Since, moreover, $\sigma$ satisfies Assumption~\ref{ass_Kidger_Lyons_Condition} then, the conditions of Corollary~\ref{cor_parallelization} are met.  Therefore, for each $k=1,\dots,K$, there exist $g\in \NN[p,m_k]$ such that $\hat{f}'\triangleq (
            \rho_1\circ \text{Exp}_{\fff_1,y_1}\circ f\circ\text{Exp}_{\xxx,x}^{-1}(\cdot)
            ,
            \dots
            ,
        \rho_K\circ \text{Exp}_{\fff_K,y_K}\circ f\circ\text{Exp}_{\xxx,x}^{-1}(\cdot)
        )$, for some $y_k\in \yyy_k$ (where $k=1,\dots,K$), satisfies the estimate:
\begin{equation}
    \sup_{x \in K}\,
    d_{\prod_k\yyy_k}
    \left(
        f(z)
            ,
        \hat{f}'(z)
    \right) 
        <
    \tilde{\omega}_{\rho}^{-1}(2^{-1}\epsilon)
\label{proof_thrm_full_computational_graph_Quotient_product_construction_applied}
    .
\end{equation}

We may now derive the estimate~\eqref{thrm_full_computational_graph_main_estimate}.  We compute:
\allowdisplaybreaks
\begin{align}
\nonumber
    \sup_{x \in \xxx} \,
    d_{\yyy}(f(x),\rho\circ \hat{f}'\circ \phi(x))
    \leq &
    \sup_{x \in \xxx} \,
    d_{\yyy}(H_{t_{\frac{\epsilon}{2}}}\circ f(x),f(x))
    +
    d_{\yyy}(H_{t_{\frac{\epsilon}{2}}}\circ f(x),\rho\circ \hat{f}'\circ \phi(x))
    \\
    \nonumber
    \leq &
    \sup_{y \in \yyy} \,
    d_{\yyy}(H_{t_{\frac{\epsilon}{2}}}(y),y)
    +
    \sup_{x \in \xxx} \,
    d_{\yyy}(H_{t_{\frac{\epsilon}{2}}}\circ f(x),\rho\circ \hat{f}'\circ \phi(x))
    \\
    \nonumber
    \leq &
    2^{-1}\epsilon
    +
    \sup_{x \in \xxx} \,
    d_{\yyy}(H_{t_{\frac{\epsilon}{2}}}\circ f(x),\rho\circ \hat{f}'\circ \phi(x))
    \\
    \nonumber
    = &
    2^{-1}\epsilon
    +
    \sup_{x \in \xxx} \,
    d_{\yyy}(\rho\circ R\circ H_{t_{\frac{\epsilon}{2}}}\circ f\circ \phi^{-1}\circ \phi(x),\rho\circ \hat{f}'\circ \phi(x))
    \\
    \nonumber
    = &
    2^{-1}\epsilon
    +
    \sup_{z \in \phi(\xxx)} \,
    d_{\yyy}(\rho\circ R\circ H_{t_{\frac{\epsilon}{2}}}\circ f\circ \phi^{-1}(z),\rho\circ \hat{f}'(z))
    \\
    \nonumber
    = &
    2^{-1}\epsilon
    +
    \sup_{z \in \phi(\xxx)} \,
\tilde{\omega}_{\rho}\left(
        d_{\yyy}(R\circ H_{t_{\frac{\epsilon}{2}}}\circ f\circ \phi^{-1}(z),\rho\circ \hat{f}'(z))
\right)
    \\
    \label{proof_thrm_full_computational_graph_last_esimtatea}
    = &
    2^{-1}\epsilon
    +
        \sup_{z \in \phi(\xxx)} \,
    \tilde{\omega}_{\rho}(\tilde{\omega}^{-1}_{\rho}(2^{-1}\epsilon))
    \\
    \label{proof_thrm_full_computational_graph_last_esimtate}
    = &
    2^{-1}\epsilon
    +
    2^{-1}\epsilon=\epsilon
    ;
\end{align}
where the inequality transitioning from~\eqref{proof_thrm_full_computational_graph_last_esimtatea} to~\eqref{proof_thrm_full_computational_graph_last_esimtate} follows from \citep[Proposition 1 (4)]{EmbrechtsHofert} since $t\mapsto \tilde{\omega}_{\rho}$ is continuous and therefore $2^{-1}\epsilon\in (0,\sup_{t\in [0,\infty)}\tilde{\omega}_{\rho}(t)]\cap \{\tilde{\omega}_{\rho}(s):s\in [0,\infty)\}$.
\end{proof}

\subsection{Proofs of Applications}
\subsubsection{{Applications in Section~\ref{ss_applications_Pt1}}}\label{s_appendix_ss_applications_Pt1}
The proofs of Corollaries~\ref{ex_HNN_univ} and~\ref{ex_HNN_eff} are analogous.  
\begin{proof}[{Proof of Corollary~\ref{ex_HNN_univ}}]
Since $\mathbb{D}_c^p$ (resp. $\mathbb{D}_c^m$) is a Cartan-Hadamard manifold, then Corollary~\ref{cor_Cartan_Hadamard} implies that $\mathcal{U}_f(x)=\infty$ for every $f \in C(\mathbb{D}_c^p,\mathbb{D}_c^m)$.  The result now follows from Theorem~\ref{thrm_main_Local} and the representation~\eqref{eq_hyperbolic_NNs}. 
\end{proof}
\begin{proof}[{Proof of Corollary~\ref{ex_HNN_eff}}]
The result is a direct consequence of Theorem~\ref{thrm_efficient_rates} and the representation~\eqref{eq_hyperbolic_NNs}. 
\end{proof}
The proofs of Corollaries~\ref{ex_Aff_inv_GDNs_univ} and~\ref{ex_Aff_inv_GDNs_eff} are identical up to their last step.  
\begin{proof}[Proof of {Corollary~\ref{ex_Aff_inv_GDNs_univ}}]
The main result of \cite{pennec2006riemannian} showed that $d_{+}\left(
	        X^{\top}AX
	            ,
	        X^{\top}BX
	    \right) = d_+(A,B)$ for every $X \in O_p$ and every $A,B\in P_m^+$.  Since both $f(A),\hat{f}(A) \in P_m^+$ for each $A \in P_p^+$ then~\eqref{eq_ex_Aff_inv_GDNs_univ} reduces to
\begin{equation}
    \max_{A \in K}\sup_{X\in O_p}\,
	    d_{+}\left(
	        X^{\top}f(x)X
	            ,
	        X^{\top}\hat{f}(x)X
	    \right)
	   =
	  \max_{A \in K}\,
	    d_{+}\left(
	        f(A)
	            ,
	        \hat{f}(A)
	    \right)
    \label{eq_sup_identity}
    .
\end{equation}
The result now follows directly from Theorem~\ref{thrm_main_Local} applied to the right-hand side of~\eqref{eq_sup_identity}; in view of the representation~\eqref{eq_SPD_universal}.
\end{proof}
\begin{proof}[Proof of {Corollary~\ref{ex_Aff_inv_GDNs_eff}}]
The main result of \cite{pennec2006riemannian} showed that $d_{+}\left(
	        X^{\top}AX
	            ,
	        X^{\top}BX
	    \right) = d_+(A,B)$ for every $X \in O_p$ and every $A,B\in P_m^+$.  Since both $f(A),\hat{f}(A) \in P_m^+$ for each $A \in P_p^+$ then~\eqref{eq_ex_Aff_inv_GDNs_univ} reduces to
\begin{equation}
    \max_{A \in K}\sup_{X\in O_p}\,
	    d_{+}\left(
	        X^{\top}f(x)X
	            ,
	        X^{\top}\hat{f}(x)X
	    \right)
	   =
	  \max_{A \in K}\,
	    d_{+}\left(
	        f(A)
	            ,
	        \hat{f}(A)
	    \right)
    \label{eq_sup_identity2}
    .
\end{equation}
The result now follows directly from Theorem~\ref{thrm_efficient_rates} applied to the right-hand side of~\eqref{eq_sup_identity2}; in view of the representation~\eqref{eq_SPD_universal}.
\end{proof}

\begin{proof}[{Proof of Corollary~\ref{cor_NE_UAT_Sphereical}}]
The Quarter-Pinched Sphere Theorem of \cite{klingenberg1991simple}, yields the estimate $\pi\leq \operatorname{inj}_{S^k}(x)$ for any $x \in S^k$ and $k\in \{n,m\}$.  The remaining statement then follows directly from Theorem~\ref{thrm_main_Local}.  
\end{proof}

\begin{proof}[{Proof of Corollary~\ref{cor_spheres_obstruction}}]
We only need to demonstrate that $1_{S^m}$ is not null-homotopic.  Recall that a space is, by definition, contractible if and only if its identity function is homotopic to a constant function.  However, by the first corollary to Hopf's Theorem \citep[page 125]{FuchsFomenkoHomotopicalTopology2016Edition2}\footnote{There is no numbering to the results in \cite{FuchsFomenkoHomotopicalTopology2016Edition2}.} $S^m$ is not contractible and therefore $1_{S^m}$ does not lie in the same homotopy class as any constant function.  Hence, setting $\xxx=S^m$, $S^m=\yyy$, and $f=1_{S^m}$ in Theorem~\ref{thrm_homotopic_necessary_condition} yields the result.  
\end{proof}

\subsubsection{{Applications in Section~\ref{ss_Applications_II_sophisiticated_models}}}\label{s_appendix_ss_Applications_II_sophisiticated_models}
\begin{proof}[{Proof of Corollary~\ref{cor_univ_multi_class_classification}}]
Set $\phi=1_{\rrp}$, $\xxx=\xxx_0$, $K=1$, $\fff_1=\yyy_1=\yyy$, and let $\rho$ be the softmax function.  Then, together,  Example~\ref{ex_softmax_and_el_simplex} and Theorem~\ref{thrm_full_computational_graph} yield the result.  
\end{proof}

\begin{proof}[{Proof of Corollary~\ref{cor_UAT_Gaussian_Data}}]
Set $\phi=\phi_{\ggg_n}$, $K=1$, $\xxx=\ggg_n$, $\rrflex{n(n+1)/2}=\xxx_0$, $\fff_1=\yyy_1=\rrflex{m(m+1)/2}$, and $\yyy=\ggg_m$, and $\rho=\phi^{-1}_{\ggg_m}$.  Then, the conclusion is implied by Theorem~\ref{thrm_full_computational_graph}.  
\end{proof}
\begin{proof}[{Proof of Corollary~\ref{cor_UAT_Deep_Kalman_Filter_Update_MAp}}]
Set $\phi=1_{\rrp}\times 1_{\rr^a}\times \phi_{\ggg_n}$, $K=1$, $\xxx=\ggg_n$, $\rrflex{p+a+n(n+1)/2}=\xxx_0$, $\fff_1=\yyy_1=\rrflex{n(n+1)/2}$, and $\yyy=\ggg_n$, and $\rho=\phi^{-1}_{\ggg_n}$.  Then, the conclusion is implied by Theorem~\ref{thrm_full_computational_graph}.  
\end{proof}
Both Corollaries~\ref{cor_quantitative} and~\ref{cor_eff_approxim_EUC} are a joint consequence of Theorem~\ref{thrm_full_computational_graph}.  Therefore, it seems most natural to us to merge their proofs.  
\begin{proof}[{Proof of Corollaries~\ref{cor_quantitative} and~\ref{cor_eff_approxim_EUC}}]
Set $\phi=1_{\rrp}$, $K=1$, $\rrp=\xxx=\xxx_0$, $K=1$, $\fff_1=\yyy_1=\yyy=\rrm$, and $\rho=1_{\rrm}$.  Then, the conclusion is implied by Theorem~\ref{thrm_full_computational_graph}.  
\end{proof}
\begin{proof}[{Proof of Corollary~\ref{cor_deep_ReLU_break}}]
By Proposition~\ref{prop_every_finite_dataset_is_efficient}, if $\XX$ is an finite set, then $\XX$ is efficient.  The conclusion is therefore implied by Corollary~\ref{cor_eff_approxim_EUC}.  
\end{proof}
\section{{Proofs of Topological Obstruction Theorems of Section~\ref{ss_Main_Obstructions}}}\label{s_Appendix_Topological_Obstrution_Results}
\subsection{{Proof of Theorem~\ref{thrm_homotopic_necessary_condition}}}\label{ss_Appendix_proof_necessary_condition}
The proofs in this section are independent of the others.  They rely on some tools from \textit{algebraic topology}; a subject, which due to its interconnectedness with group theory and other topics in general topology, cannot realistically be comprehensibly summarized within the confines of a single paper's appendix.  Nevertheless, we refer the interested reader to the following extremely well-written book on the subject: \cite{HatcherAlgebraicTopology}.  

\begin{lemma}[{Geometry of $C(\xxx,\yyy)$}]\label{lem_path_components}
Let $\xxx$ and $\yyy$ be complete Riemannian manifolds and let $f_1,f_2\in C(\xxx,\yyy)$.  Define the subsets $C_i\subseteq C(\xxx,\yyy)$ by: \textit{$g\in C_i$ if and only if there exists a (homotopy) $h_g\in C([0,1]\times \xxx,\yyy)$}
$$
h_g(0,\cdot)=g \mbox{ and } h_g(1,\cdot)=f_i.
$$
Then, the sets $C_1$, and $C_2$ are closed in $C(\xxx,\yyy)$.  Moreover, if there does not exist a (homotopy) $h\in C([0,1]\times \xxx,\yyy)$ satisfying:
$$
h(0,\cdot) = f_1 \mbox{ and } h(1,\cdot)=f_2,
$$
then, $C_1\cap C_2=\emptyset$.  Therefore, $C(\xxx,\yyy)$ is disconnected.
\end{lemma}
\begin{proof}
Note that the sets $C_1$ and $C_2$ are precisely the definition of the homotopy class of $f_1$ and $f_2$.  Thus, we want to show that the homotopy classes in $C(\xxx,\yyy)$ for disjoint closed subsets thereof.  

Since $\yyy$ is a metric space then \citep[Theorem 46.8]{munkres2014topology} implies that the uniform topology on $C(\xxx,\yyy)$ and the compact-open topology thereon, generated by the sub-basic open sets $\{U_{K,O}:\,K\subseteq \xxx\mbox{ compact and } O\subseteq \yyy \mbox{ open}\}$ where:
$$
U_{K,O}\triangleq 
\left\{
f\in C(\xxx,\yyy):\,
f(K)\subseteq O
\right\},
$$
coincide.  We therefore consider $C(\xxx,\yyy)$ with the latter of the two topologies.

Since $\xxx$ and $\yyy$ are smooth manifolds (without boundary) then by \citep[Example 0.3]{HatcherAlgebraicTopology} both $\xxx$ and $\yyy$ are CW-complexes (see \citep[page 5]{HatcherAlgebraicTopology}).  Moreover, since $\xxx$ is a Riemannian manifold then it is a metric space, with (Riemannian) distance function $d_{\xxx}$; and in-particular so is $\xxx$.  Furthermore, since $\xxx$ is closed and bounded in $\xxx$ and since $\xxx$ is complete then by the Hopf-Rinow Theorem it is compact.  Hence, \citep[Corollary 2]{MilnorHomotopyTypeCWComplexMappingSpace1959} applies and therefore $C(\xxx,\yyy)$ is a CW-complex whose path components correspond the homotopy classes in $C(\xxx,\yyy)$.  By \citep[Proposition A.4]{HatcherAlgebraicTopology} every CW-complex is locally-contractible and therefore (see \citep[page 522]{HatcherAlgebraicTopology}) it is locally path-connected.  Hence, by \citep[Theorem 25.5]{munkres2014topology} every path component of $C(\xxx,\yyy)$ is a connected component and therefore every path component of $C(\xxx,\yyy)$ is closed.  Hence, every homotopy class in $C(\xxx,\yyy)$ is closed.  
\end{proof}
\begin{proof}[{Proof of Lemma~\ref{lem_top_nec}}]
Direct consequence of Lemma~\ref{lem_path_components} and the fact that disjoint closed sets in a metric space (such as $C(\xxx,\yyy)$) are separated by a positive distance.  
\end{proof}
\begin{proof}[{Proof of the claim in Example~\ref{ex_massive_difference}}]
By Serre's Finiteness Theorem (see \citep[Ravenel 86, Chapter I, Lemma 1.1.8]{RavenelComplexBordism1986}) there exists countably infinite number of distinct homotopy classes in $C(S^3,S^2)$.  Thus, there exists an infinite $\fff\subset C(S^3,S^2)$ whose elements are all in distinct homotopy classes.  Thus, the claim now follows from Lemma~\ref{lem_path_components}.
\end{proof}
\begin{proof}[{Proof of Theorem~\ref{thrm_homotopic_necessary_condition}}]
Under Assumption~\ref{ass_non_degenerate_spaces}, by arguing analogously to the proof of Lemma~\ref{lem_techincal_lemmas} (mutatis mutandis) we find that for every $g \in C(\rrp,\rrm)$, the map 
$
\rho_{\zeta}^{-1}\circ 
g
\circ \phi_{\alpha}
$ is well-defined and continuous on 
$
C\left(
U_{\alpha}
,
\yyy
\right)
.
$  
In particular, it is continuous on any non-empty compact subset of thereof.  Therefore (i) holds.  The remainder of the proof is devoted to showing (ii).

\textbf{Step 1 - Functions Of Different Homotopy Types are in Different Connected Components of $C(\xxx,\yyy)$:}  
By Lemma~\ref{lem_path_components}, the homotopy classes in $C(\xxx,\yyy)$ are disjoint closed subsets.  

\textbf{Step 2 - GDNs are Homotopic to Constant Functions:}  
Every function $\hat{f} \in C(\xxx,\yyy)$ of the form $\hat{f} = \rho_{\zeta}^{-1}\circ g\circ \phi_{\alpha}$, for some $y \in \yyy$, 
and some 
$g \in 
C(\rrp,\rrm)
$, is homotopic to a constant function.  
Observe that for any $t \in [0,1]$ the function 
$
g_t(x)\triangleq tg(x),
$ 
simply corresponds to re-scaling last affine layer map defining $g$; thus, 
$
g_t \in 
C(\rrp,\rrm)
$.  Moreover, since multiplication is continuous in $\rrflex{m}$ and the composition of continuous functions is again continuous.  Thus, 
$$
\begin{aligned}
 G:[0,1]\times \rrflex{p}&\rightarrow \rrflex{m}    \\
 (t,x)&\mapsto g_t(x)
,
\end{aligned}
$$
is itself continuous.  In particular the restriction of $G$ to $[0,1]\times \phi_{\alpha}(U_{\alpha})$ is continuous and takes values in $\yyy$.  Hence, $G$ defines a homotopy from $g$ to the constant function $g_0(x)\mapsto \rho_{\zeta}^{-1}(0)$.  Therefore, for every $y \in \yyy$, 
and every $g\in 
C(\rrp,\rrm)
$, 
the function $\hat{f}:x\mapsto \rho_{\zeta}^{-1}\circ g$ is homotopic to a constant function.  Applying by the remark on \citep[page 26]{FuchsFomenkoHomotopicalTopology2016Edition2} we conclude that 
$g\in
C(\rrp,\rrm)
$, 
the function $\hat{f}:z\mapsto \rho_{\zeta}^{-1}\circ g\circ \phi_{\alpha}(z)$ is homotopic to a constant function. 

\textbf{Step 3 - The $\hat{f}$ and the target function $f$ lie in disjoint closed subsets of $C(\xxx,\yyy)$:}  
Now, Step 2 implies that every $\hat{f}\in C(\xxx,\yyy)$ of the form~\eqref{eq_definition_hat_f_for_proof} is homotopic to the constant function.  By hypothesis, $f$ is not homotopic to the constant function.  Hence, Step 1 guarantees that $\hat{f}$ and $f$ belong to disjoint connected components of $C(\xxx,\yyy)$.  

\textbf{Step 4 - Conclusion:}  
Again applying \citep[Theorem 46.8]{munkres2014topology}, the thus must exist some $0<\epsilon$ bounding the uniform distance by these sets; i.e., (ii) holds.
\end{proof}
\subsection{{Proof of Theorem~\ref{thrm_negative_motiation}}}\label{ss_Appendix_proof_negative_motivation_result}
Our goal is to apply Theorem~\ref{thrm_homotopic_necessary_condition} to draw out the conclusion.  To this end, it is enough for us to demonstrate the existence of a function in $C(\xxx,\yyy)$ which is not null-homotopic.  

In the following, given a topological space $X$, we use $H_{k}(X)$ (resp. $H^{k}(X;\zz)$, $H^{k}(X)$) to denote $X$'s $k^{\text{th}}$ singular homology group (resp. relative homology group with coefficients in $\zz$, resp. cohomology group).  We refer the reader to \citep[pages 108 and 190]{HatcherAlgebraicTopology} for definitions.  
\begin{proof}
\textbf{Step 1 - Reduction to Spherical Input Spaces:} 
Since $U_{\alpha}$ is homeomorphic to $\rrp$, via $\phi_{\alpha}$; then, for any $1\leq d\leq p$, let $S^d\left(
0,1
\right)\triangleq 
\left\{
u \in \rrflex{p}:\,
\|u\|=1
\right\}
,
$ and note that $S^d\left(
0,1
\right)$ is homeomorphic to $S^d$ and it is contained within $\phi_{\alpha}(U_{\alpha})=\rrp$
Let $\xxx_d\triangleq 
\phi_{\alpha}^{-1}
\left(
S^d\left(
0,1
\right)
\right)
\subset 
U_{\alpha}
,
$ and note that $\xxx_d$ is must also be homeomorphic to $S^d$.  Note that, for every $1\leq d\leq p$, since $\xxx_d\subset 
U_{\alpha}
$ then for any $\zeta \in Z
$ and any continuous function $g:\rrflex{p}\rightarrow \rrflex{m}$ the map $
\rho_{\zeta}^{-1}
\circ g\circ 
\phi_{\alpha}
$ is well-defined.  
In particular, since the activation function $\sigma$ is continuous then, for every $k\in \nn$, every $g \in C(\rrp,\rrm)
$ is a continuous function from $\rrflex{p}$ to $\rrflex{m}$.  Hence, for every $k\in\nn_+$, every $g \in C(\rrp,\rrm)
$, and every $y \in \yyy$, the map 
\begin{equation}
    \hat{f}\triangleq 
    \rho_{\zeta}^{-1}
    \circ g\circ 
    \phi_{\alpha}
    \label{eq_definition_hat_f_for_proof}
    ,
\end{equation}
is well-defined on each $\xxx_d$ for $1\leq d\leq p$.

\textbf{Step 2 - Existence of a Non-Null Homotopic Function in $C(\xxx,\yyy)$:}
By the remark on \citep[page 26]{FuchsFomenkoHomotopicalTopology2016Edition2}, for any $1\leq d\leq p$, there is a continuous function in $C(\xxx_d,\yyy)$ which is not homotopic to a constant function if and only if there is a continuous function in $C(S^d,\yyy)$ which is not homotopic to a constant function.  Our next step is to demonstrate that there is a $1\leq d\leq p$ for which $C(S^d,\yyy)$ contains a continuous function which is not homotopic to a constant function.  Once we prove that $d$ exists, we set $\kkk\triangleq \xxx_d$. 

Since, $\yyy$ is closed and orientable then \citep[Theorem 3.30 (c)]{HatcherAlgebraicTopology} implies that the (singular) homology groups $H_s(\yyy)$ are all trivial for $s>p$.  
Suppose that no such $d$ existed.  Then, $C(S^d,\yyy)$ only consists of null-homotopic maps (maps which are homotopic to a constant function) and \citep[Proposition 7.2]{spanier1994algebraic} implies that for every $d< p$ the (singular) homology groups $H_d(\yyy;\zz)$ are all trivial $d<p$.  
Furthermore, \citep[Theorem 3.26]{HatcherAlgebraicTopology} implies that $H_d(\yyy)$ is isomorphic to $\zz$ and therefore it is non-trivial.  We may therefore apply \citep[Hurewicz's Isomorphism Theorem; Proposition 7.2]{spanier1994algebraic} to find that $\pi(S^{d},\yyy)$ is non-trivial since $H_{p}(\yyy)$ is non-trivial.  By the same result, there exists a surjective (group homomorphism) from $\pi(S^{p},\yyy)$ onto $H_{p}(\yyy)$ and therefore $\pi(S^{p},\yyy)$ must be non-trivial; a contradiction.  Hence, such a $d\in \{1,\dots,p\}$ must exist.  Hence, $C(\kkk,\yyy)=C(\xxx_d,\yyy)$ contains a function which is not null-homotopic.  
\textbf{Step 3 - Conclusion: }
The result now follows by Theorem~\ref{thrm_homotopic_necessary_condition} (ii).  
\end{proof}

\bibliography{References}
%
\end{document}